%% file: Main-rev.tex
\documentclass[10pt,journal,compsoc]{IEEEtran}
\ifCLASSOPTIONcompsoc

\usepackage{array}
\usepackage[caption=false]{subfig}

\usepackage{textcomp}
\usepackage{stfloats}
\usepackage[table]{xcolor}
\usepackage{url}
\usepackage{ragged2e}
\usepackage{verbatim}
\usepackage{graphicx}
\ifCLASSOPTIONcompsoc
  \usepackage[nocompress]{cite}
\else
  \usepackage{cite}
\fi
\usepackage{amsmath}
\usepackage{amssymb}
\usepackage{bm}
\usepackage{amsfonts}
\usepackage{amssymb}
\usepackage{amsthm}
\usepackage{graphicx}
\usepackage{booktabs}
\usepackage{algorithm}
\usepackage{algorithmic}
\usepackage{hyperref}
\usepackage{cancel}

\usepackage{pifont}
\usepackage{bbding}
\usepackage{xspace}
\newcommand{\cmark}{\textcolor{green!60!black}{\ding{51}}\xspace}
\newcommand{\xmark}{\textcolor{red!60!black}{\ding{55}}\xspace}

\usepackage{scrextend}
\deffootnote[0.1em]{0.1em}{1em}{\textsuperscript{\thefootnotemark}\,}

\usepackage{amssymb}
\usepackage{mathtools}
\usepackage{amsfonts}
\usepackage{amsmath}
\usepackage{amsthm}
\usepackage{booktabs}
\usepackage[]{microtype}
\usepackage{tikz}
\usepackage{adjustbox}
\usepackage{multirow}
\usepackage{mathtools}
\usepackage{soul}
\usepackage{xcolor}
\usepackage{orcidlink}

\newcommand{\revise}[1]{{\color{black}{#1}}}
\newtheorem{theorem}{Theorem}

\newtheorem{remark}{Remark}
\newtheorem{example}{Example}

\newtheorem{corollary}{Corollary}[theorem]

\newcommand{\Y}{\mathbf{Y}}
\newcommand{\X}{\mathbf{X}}
\newcommand{\x}{\mathbf{x}}
\newcommand{\y}{\mathbf{y}}
\newcommand{\ba}{\mathbf{h}}
\newcommand{\bA}{\mathbf{H}}
\newcommand{\bPhi}{\boldsymbol{\Phi}}
\newcommand{\I}{\mathbf{I}}
\newcommand{\w}{\mathbf{w}}
\newcommand{\W}{\mathbf{W}}
\newcommand{\bmu}{\boldsymbol{\mu}}
\newcommand{\bSigma}{\boldsymbol{\Sigma}}

\newcommand{\z}{\mathbf{z}}
\newcommand{\bv}{\mathbf{v}}
\newcommand{\btheta}{\boldsymbol{\theta}}
\newcommand{\K}{\mathbf{K}}

\newcommand{\vvarphi}{\boldsymbol{\varphi}}
\newcommand{\bTheta}{\boldsymbol{\Theta}}
\newcommand{\srflvm}{SRFLVMs}

\begin{document}

\title{Scalable Random Feature Latent Variable Models}


\author{
Ying Li, 
Zhidi Lin,
Yuhao Liu, 
Michael Minyi Zhang,
Pablo M. Olmos,
and Petar M. Djuri{\'c}\IEEEmembership{,~Fellow,~IEEE}

\IEEEcompsocitemizethanks{

\IEEEcompsocthanksitem
Y. Li and M. M. Zhang are with the Department of Statistics \& Actuarial Science, University of Hong Kong, Hong Kong, SAR, China. 
E-mail: \href{mailto:lynnli98@connect.hku.hk,u3010699@connect.hku.hk}{lynnli98@connect.hku.hk},  \href{mailto:mzhang18@hku.hk}{mzhang18@hku.hk}. 

\IEEEcompsocthanksitem Z. Lin is with the Department of Statistics \& Data Science, National University of Singapore, Singapore. 
E-mail: \href{mailto:zhidilin@nus.edu.sg}{zhidilin@nus.edu.sg}.

\IEEEcompsocthanksitem Y. Liu is with Capital One, McLean, Virginia, United States. Email: \href{mailto:yuhliu1994@gmail.com}{yuhaoliu1994@gmail.com}

\IEEEcompsocthanksitem P. M. Olmos is with the Department of Signal Theory \& Communications, University Carlos III de Madrid, Madrid, Spain. 
E-mail: \href{mailto:pamartin@ing.uc3m.es}{pamartin@ing.uc3m.es}.

\IEEEcompsocthanksitem P. M. Djuri{\'c} is with the Department of Electrical \& Computer Engineering, Stony Brook University, Stony Brook, NY 11794 USA. 
E-mail: \href{mailto:petar.djuric@stonybrook.edu}{petar.djuric@stonybrook.edu}.

}
}

\markboth{Journal of \LaTeX\ Class Files,~Vol.~14, No.~8, August~2015}%
{Shell \MakeLowercase{\textit{et al.}}: Bare Demo of IEEEtran.cls for Computer Society Journals}

\IEEEtitleabstractindextext{%
\begin{abstract}
\justifying
Random feature latent variable models (RFLVMs) represent the state-of-the-art in latent variable models, capable of handling non-Gaussian likelihoods and effectively uncovering patterns in high-dimensional data. However, their heavy reliance on Monte Carlo sampling results in scalability issues which makes it difficult to use these models for datasets with a massive number of observations. To scale up RFLVMs, we turn to the optimization-based variational Bayesian inference (VBI) algorithm which is known for its scalability compared to sampling-based methods. However, implementing VBI for RFLVMs poses challenges, such as the lack of explicit probability distribution functions (PDFs) for the Dirichlet process (DP) in the kernel learning component, and the incompatibility of existing VBI algorithms with RFLVMs. To address these issues, we introduce a stick-breaking construction for DP to obtain an explicit PDF and a novel VBI algorithm called ``block coordinate descent variational inference" (BCD-VI). This enables the development of a scalable version of RFLVMs, or in short, \srflvm. Our proposed method shows scalability, computational efficiency, superior performance in generating informative latent representations and the ability of imputing missing data across various real-world datasets, outperforming state-of-the-art competitors.
\end{abstract}

}

\maketitle

\IEEEdisplaynontitleabstractindextext

%
\IEEEpeerreviewmaketitle

\input{content/intro}

\input{content/gplvm_rflvm}

\input{content/vi}

\input{content/srflvm}

\input{content/simulation}

\input{content/conclusion}


%


\bibliographystyle{IEEEtran}
\bibliography{ref}{}


\begin{IEEEbiographynophoto}{Ying Li} received the B.Eng. degree from Shandong University in 2019 and M.Phil. degree from The Chinese University of Hong Kong, Shenzhen in 2022. She is currently pursuing a Ph.D. degree with the Department of Statistics and Actuarial Science at the University of Hong Kong, Hong Kong, SAR, China. Her research interests include Bayesian nonparametrics and variational auto-encoder. 
\end{IEEEbiographynophoto}

\begin{IEEEbiographynophoto}{Zhidi Lin} received his M.Sc. degree in Communication and Information Systems from Xiamen University, Xiamen, China, in 2019, and his Ph.D. degree in Computer and Information Engineering from the School of Science and Engineering, The Chinese University of Hong Kong, Shenzhen, China, in 2024. He is currently a Postdoctoral Research Fellow in the Department of Statistics and Data Science at the National University of Singapore. His research interests include statistical signal processing, Bayesian deep learning, approximate Bayesian inference, and data assimilation.
\end{IEEEbiographynophoto}

\begin{IEEEbiographynophoto}{Yuhao Liu} (Member, IEEE) received the BSc degree in mathematical sciences from Tsinghua University, Beijing, China, in 2016, the MSc degree in statistics from the University of Michigan, Ann Arbor, MI, USA, in 2018, and the PhD degree in applied mathematics and statistics from Stony Brook University, NY, USA, in 2023. He is currently working as a data scientist with Capital One, McLean, VA, USA. His research interests include Bayesian deep learning, ensemble machine learning, signal processing, causal inference, and their applications in finance and
neuroscience.
\end{IEEEbiographynophoto}

\begin{IEEEbiographynophoto}{Michael Minyi Zhang} is currently an Assistant Professor with the Department of Statistics and Actuarial Science, The University of Hong Kong. His research interests include statistical machine learning, scalable inference, and Bayesian non-parametrics. 
\end{IEEEbiographynophoto}

\begin{IEEEbiographynophoto}{Pablo M. Olmos} was born in Granada, Spain, in 1984. He received the
B.Sc., M.Sc., and Ph.D. degrees in telecommunication engineering from the University of Sevilla, in 2008 and 2011, respectively. He has held appointments as a Visiting Researcher with Princeton University, \'{E}cole Polytechnique F\'{e}d\'{e}rale de Lausanne (EPFL), Notre Dame University, \'{E}cole Nationale Sup\'{e}rieure de l'\'{E}lectronique et de ses Applications (ENSEA), and Nokia-Bell Labs. He is currently an Associate Professor with University Carlos III of Madrid. His research interests include approximate inference methods for Bayesian machine learning to information theory and digital communications. A detailed CV and list of publications can be accessed at \url{https://www.tsc.uc3m.es/~olmos/}. 
\end{IEEEbiographynophoto}

\begin{IEEEbiographynophoto}{Petar M. Djuri{\'c}} (Fellow, IEEE) received the BS
and MS degrees in electrical engineering from the
University of Belgrade, and the PhD degree in electrical engineering from the University of Rhode Island. Following the completion of his PhD, he joined Stony Brook University, where he currently holds the position of SUNY distinguished professor and serves as the Savitri Devi Bangaru professor in artificial intelligence. He also held the role of chair of the Department of Electrical and Computer Engineering from 2016 to 2023. His research has predominantly focused on machine learning and signal and information processing. In 2012, he received the EURASIP Technical Achievement Award whereas in 2008, he was appointed chair of excellence of Universidad Carlos III de Madrid-Banco de Santander. He has actively participated in various committees of the IEEE Signal Processing Society and served on committees for numerous professional conferences and workshops. He was the founding editor-in-chief of \textit{IEEE Transactions on Signal and Information Processing Over Networks}. In 2022, he was elected as a foreign member of the Serbian Academy of Engineering Sciences. Furthermore, he holds
the distinction of being a fellow of EURASIP. 
\end{IEEEbiographynophoto}

\clearpage
\appendices
\input{appendix/appendix}

\begin{figure*}[h]
    \centering
    \subfloat[MNIST reconstruction task with 0\% missing pixels.]{
    \includegraphics[width=.32\linewidth]{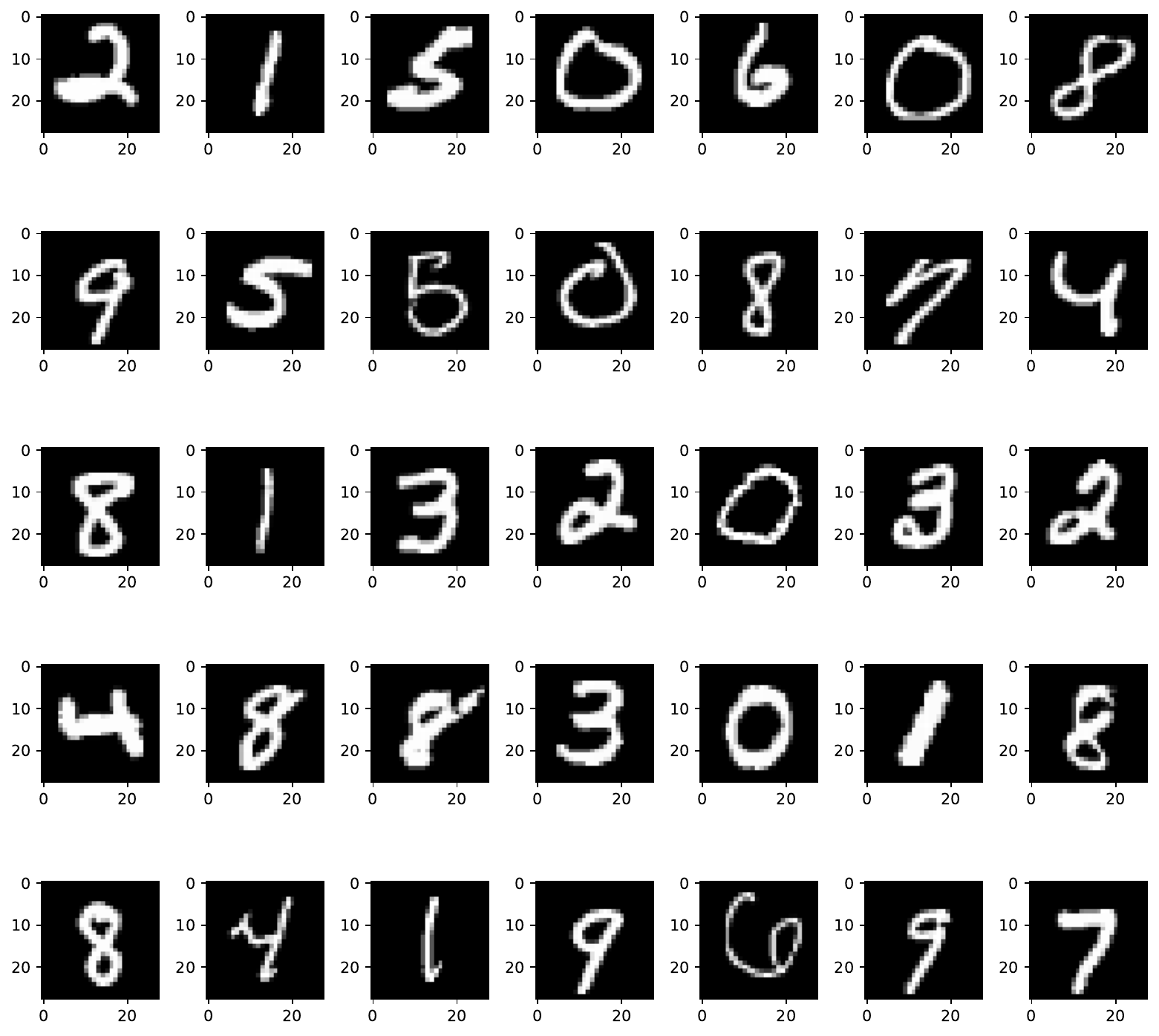} \ \ \
    \includegraphics[width=.32\linewidth]{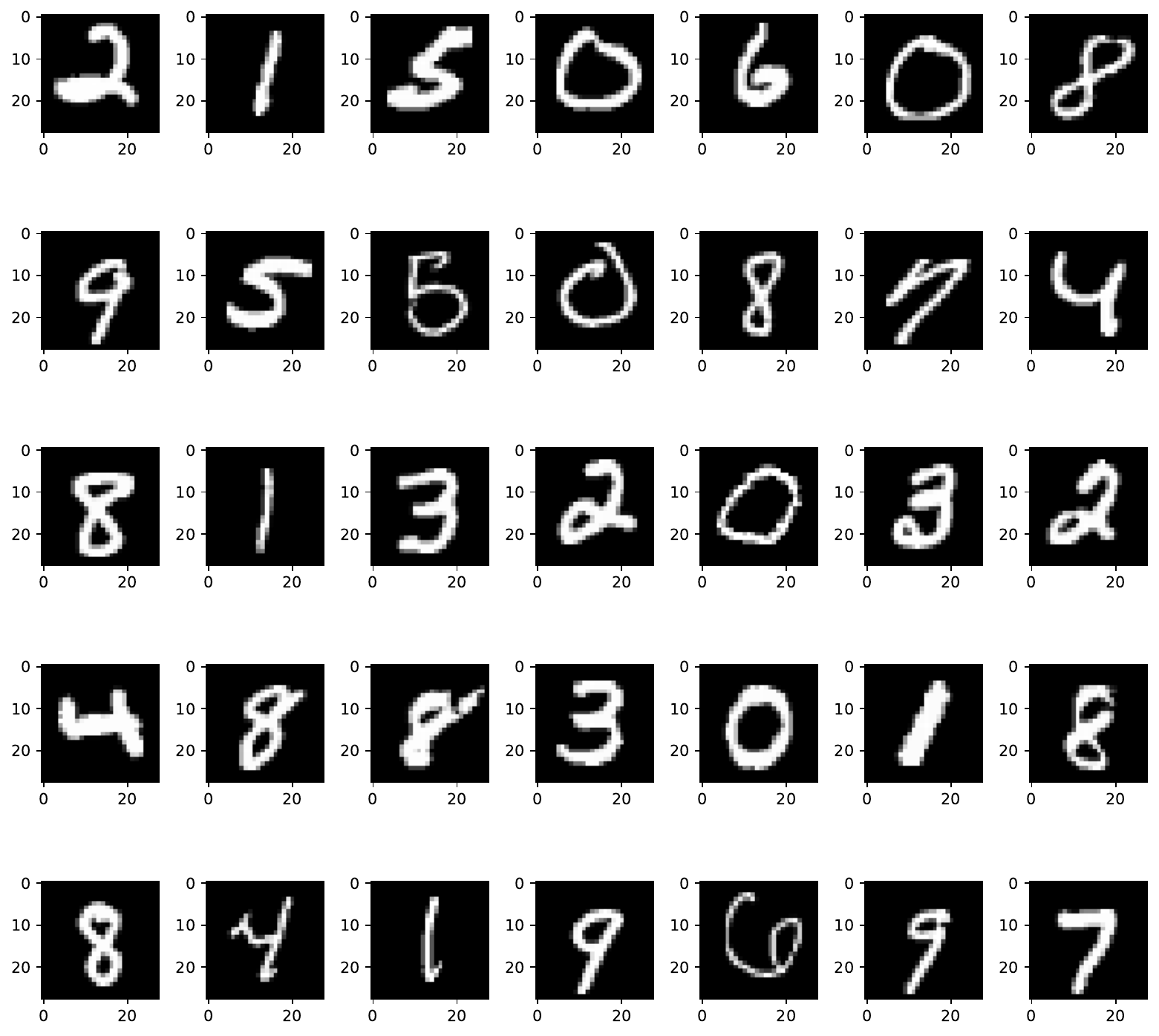} \ \ \
    \includegraphics[width=.32\linewidth]{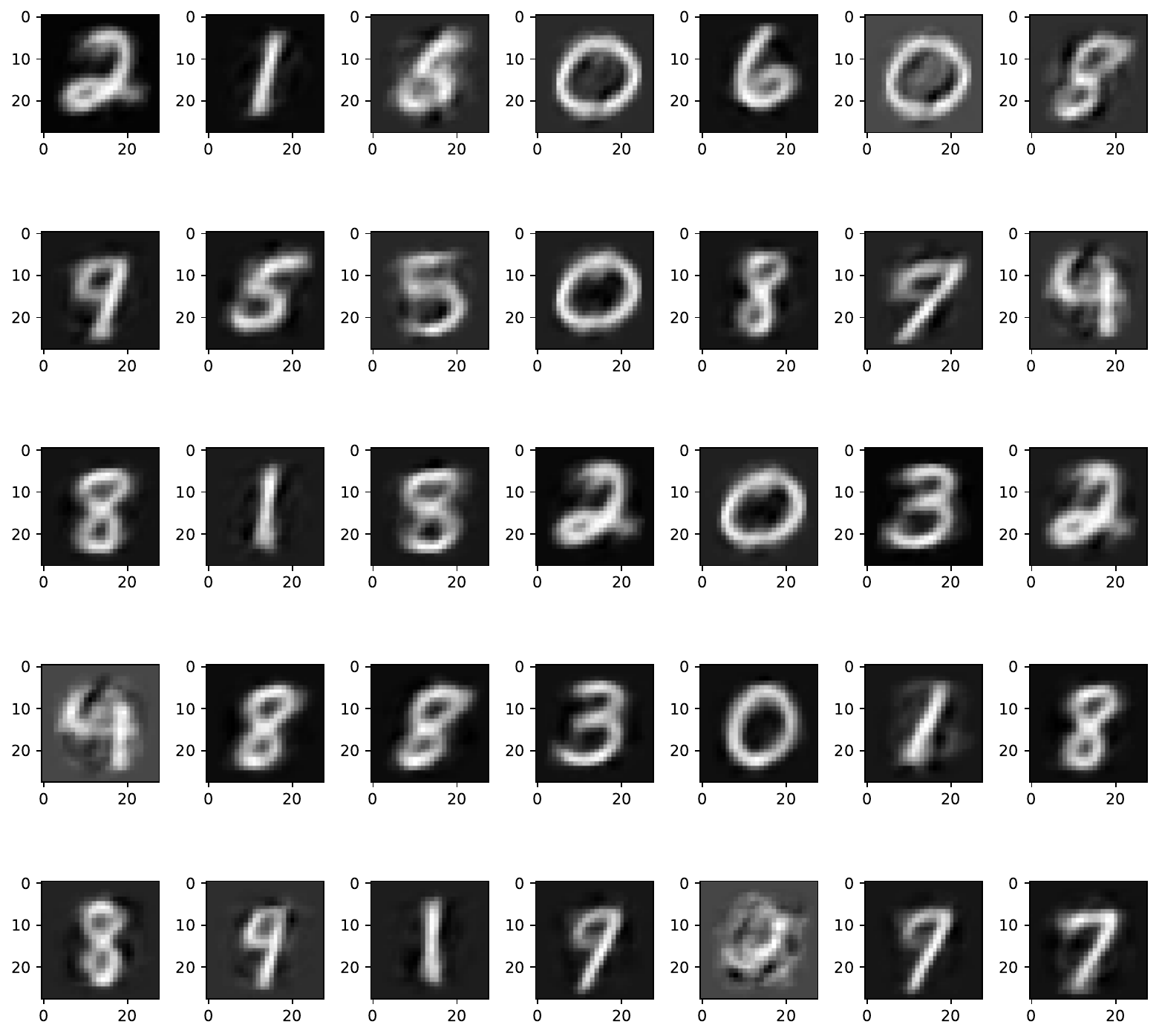}
    } 
    \vspace{-0.05in}
    
    \subfloat[MNIST reconstruction task with 10\% missing pixels.]{
    \includegraphics[width=.32\linewidth]{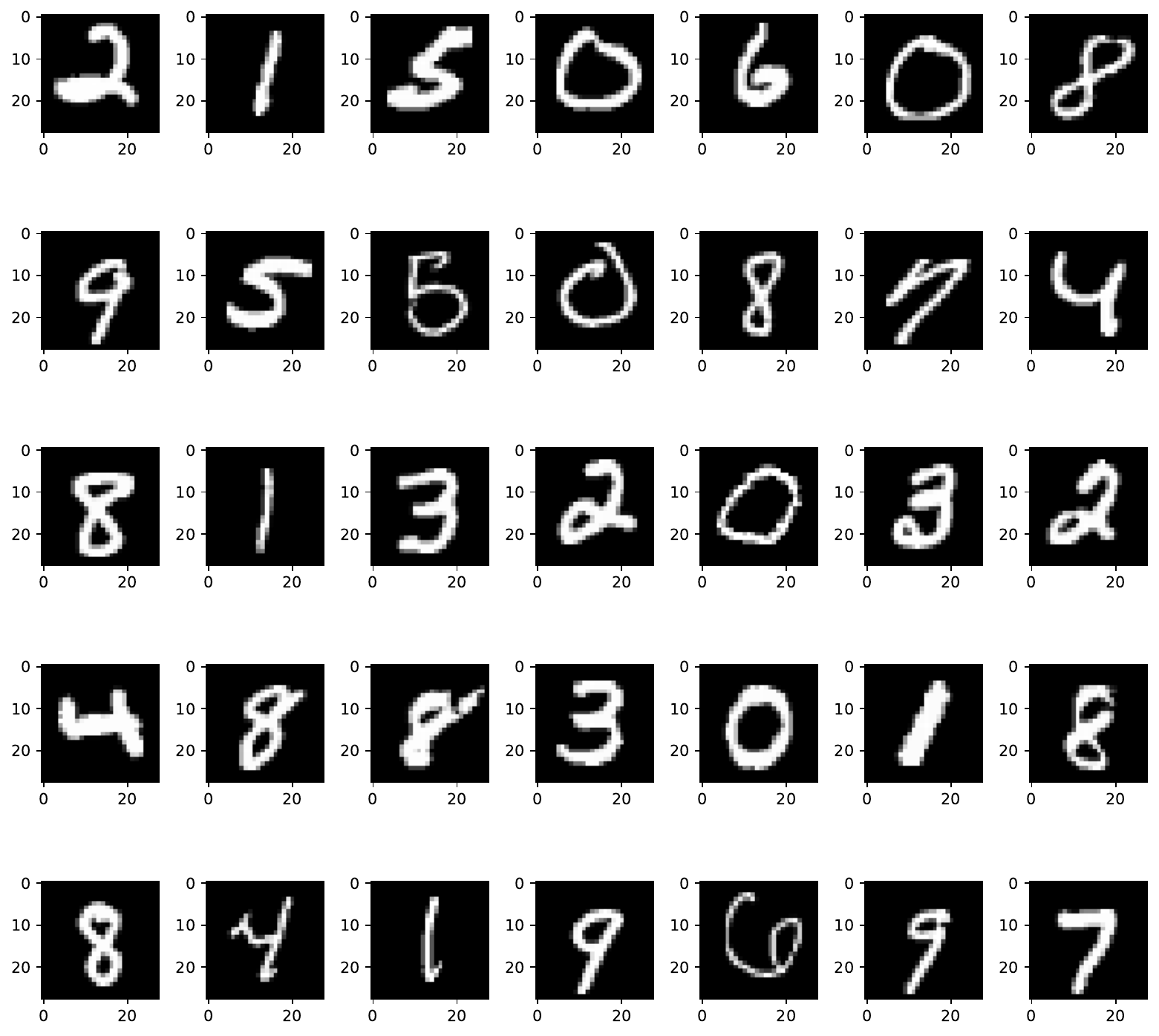} \ \ \
    \includegraphics[width=.32\linewidth]{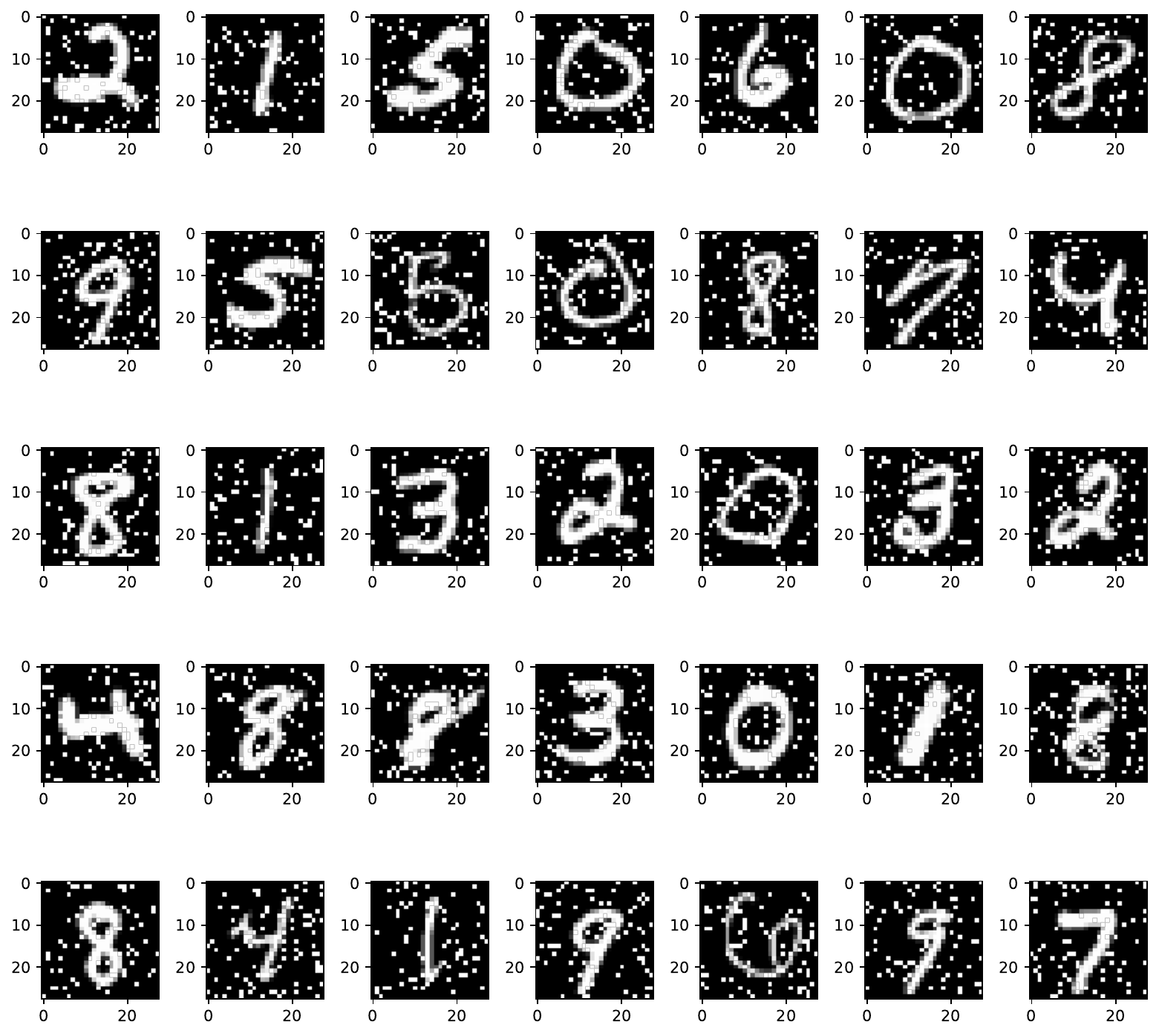} \ \ \
    \includegraphics[width=.32\linewidth]{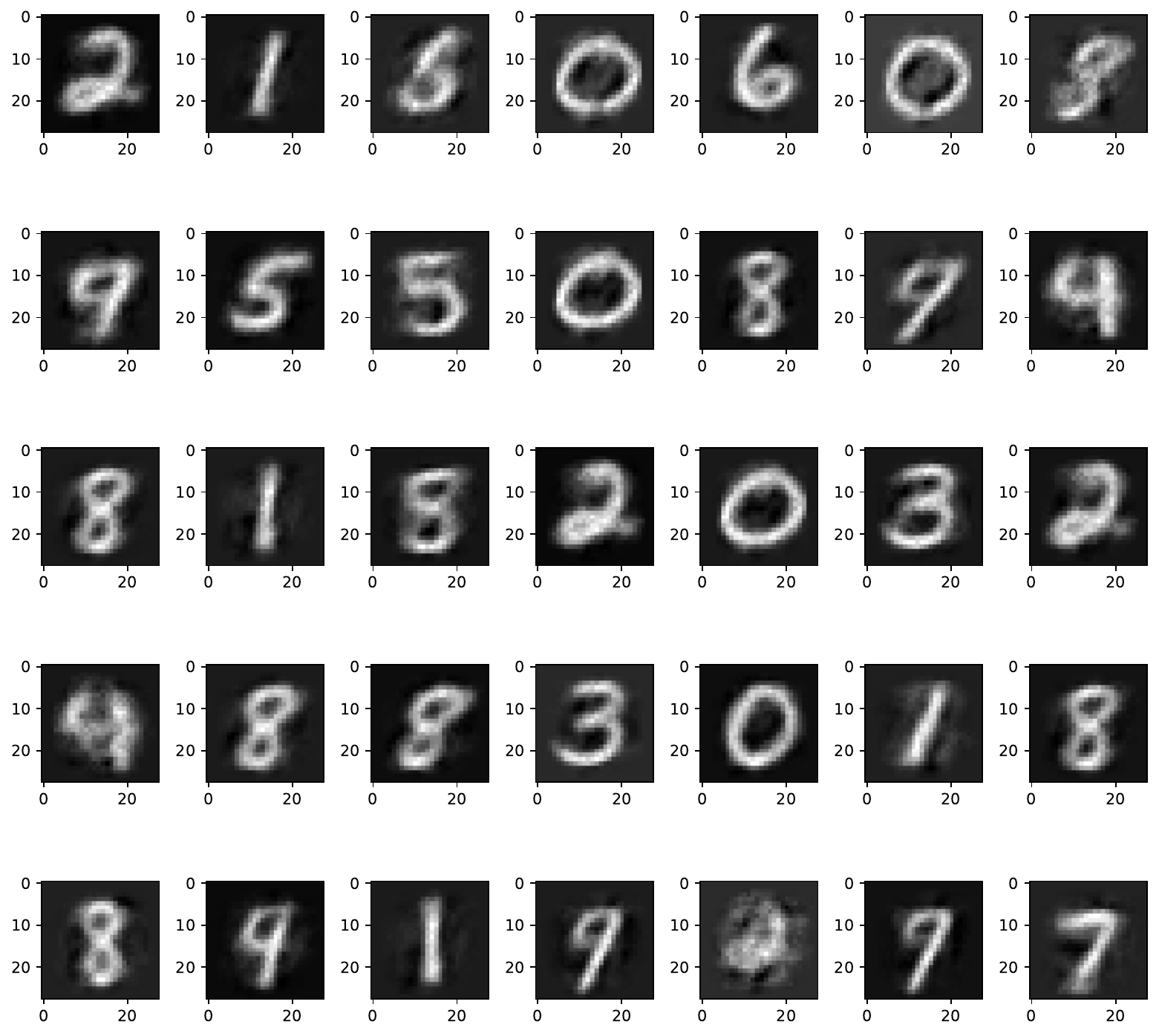}
    } 
    \vspace{-0.05in}

    \subfloat[MNIST reconstruction task with 30\% missing pixels.]{
    \includegraphics[width=.32\linewidth]{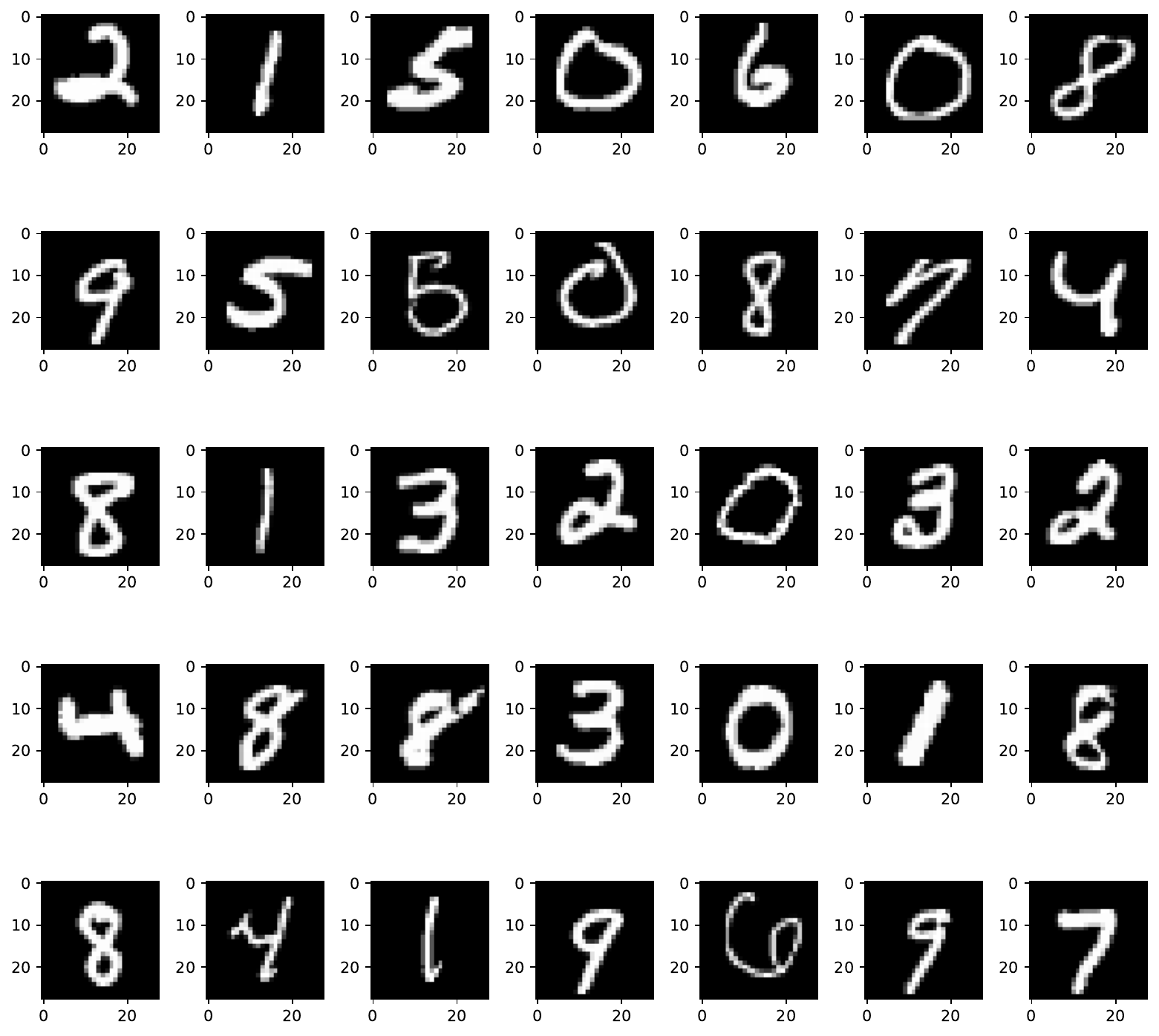} \ \ \
    \includegraphics[width=.32\linewidth]{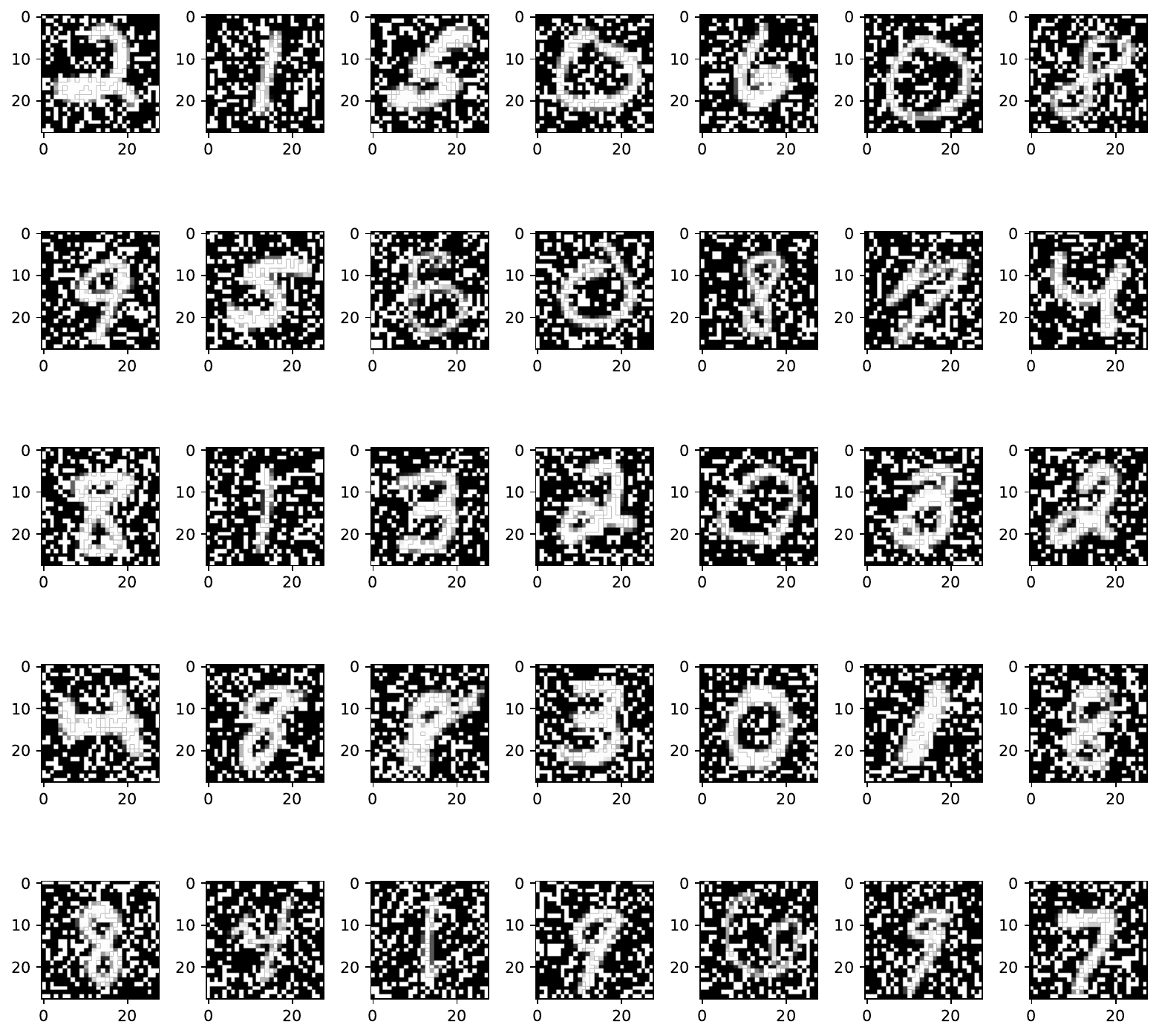} \ \ \
    \includegraphics[width=.32\linewidth]{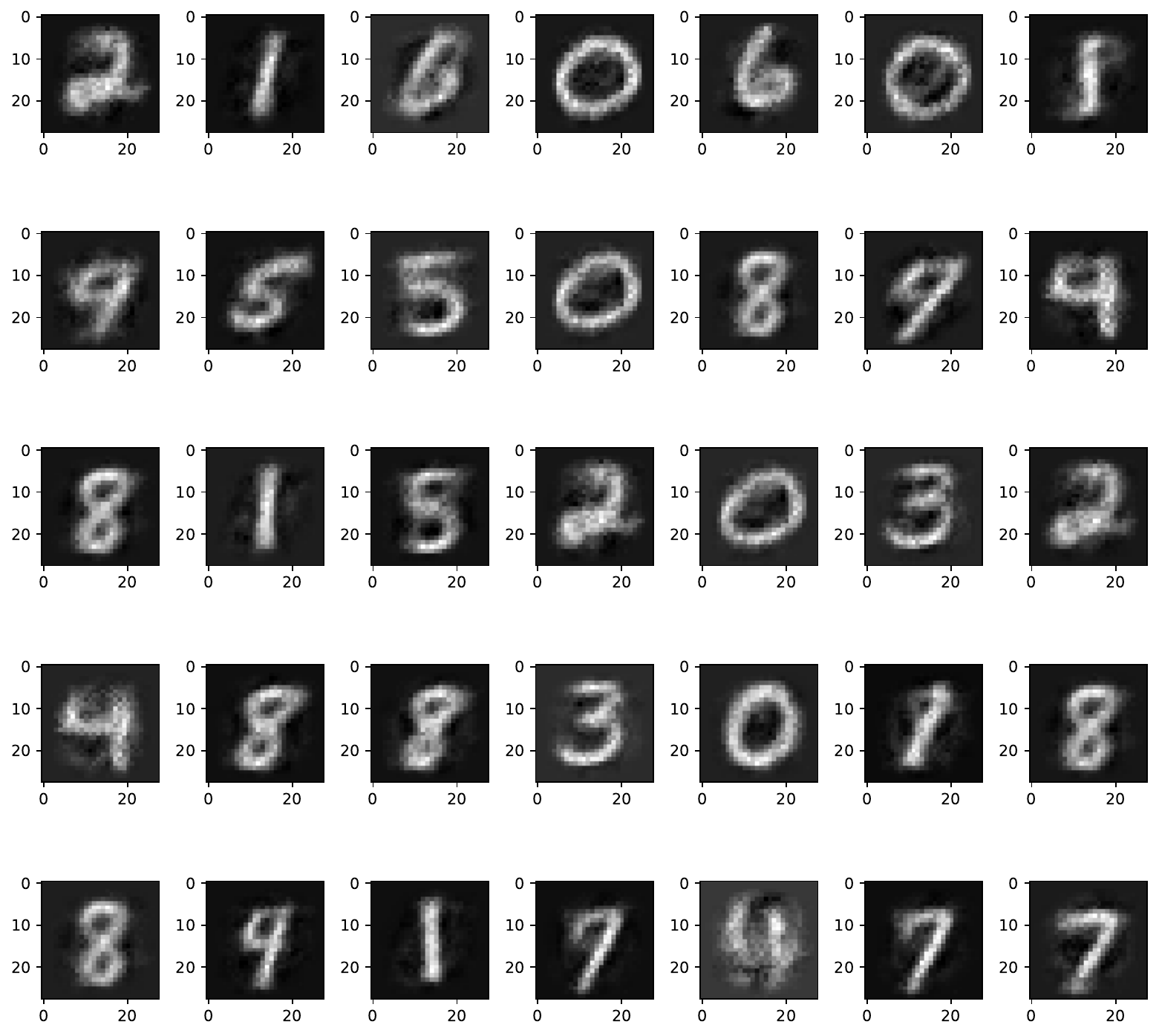}
    } 
    \vspace{-0.05in}

    \subfloat[MNIST reconstruction task with 60\% missing pixels.]{
    \includegraphics[width=.32\linewidth]{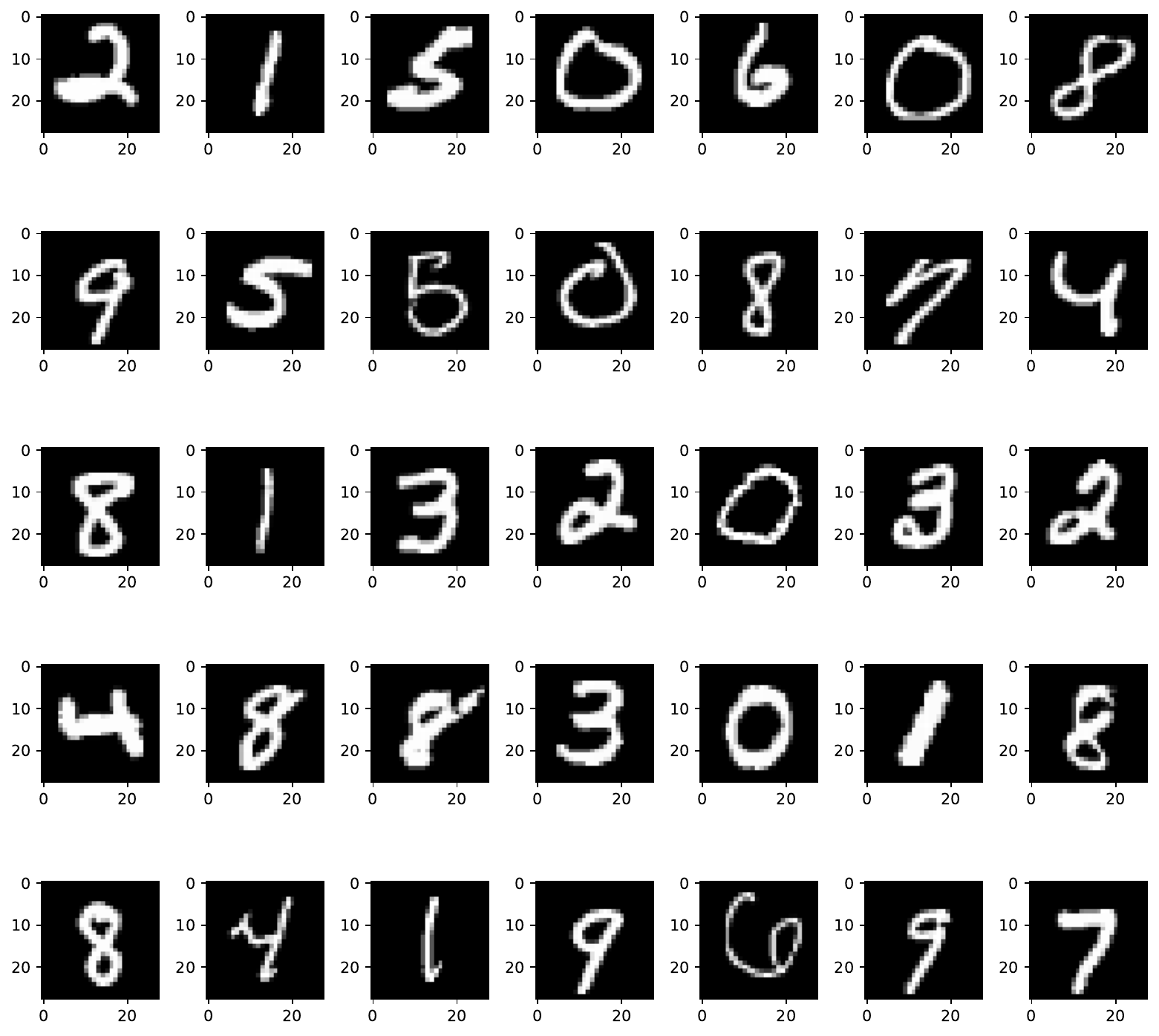} \ \ \
    \includegraphics[width=.32\linewidth]{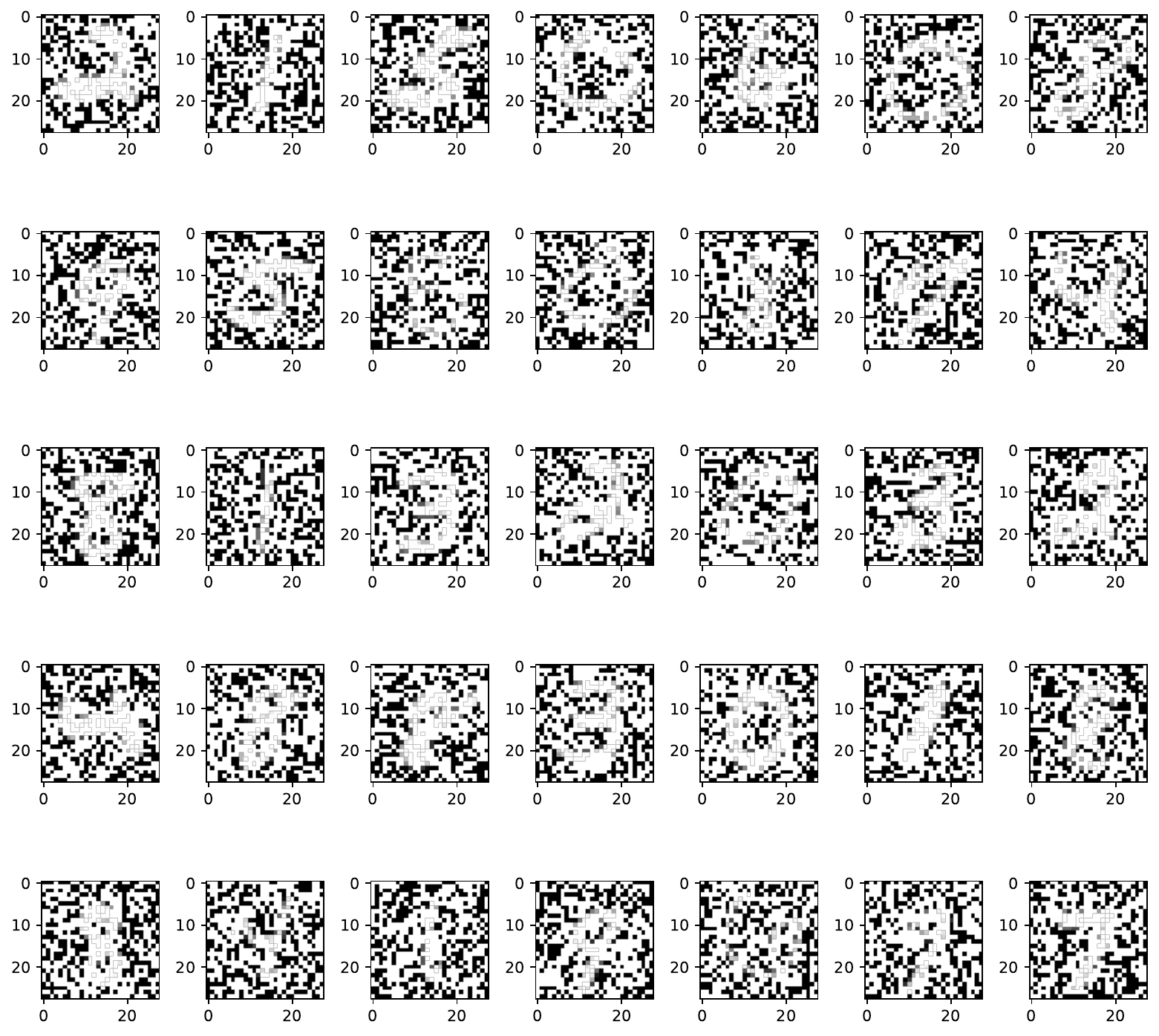} \ \ \
    \includegraphics[width=.32\linewidth]{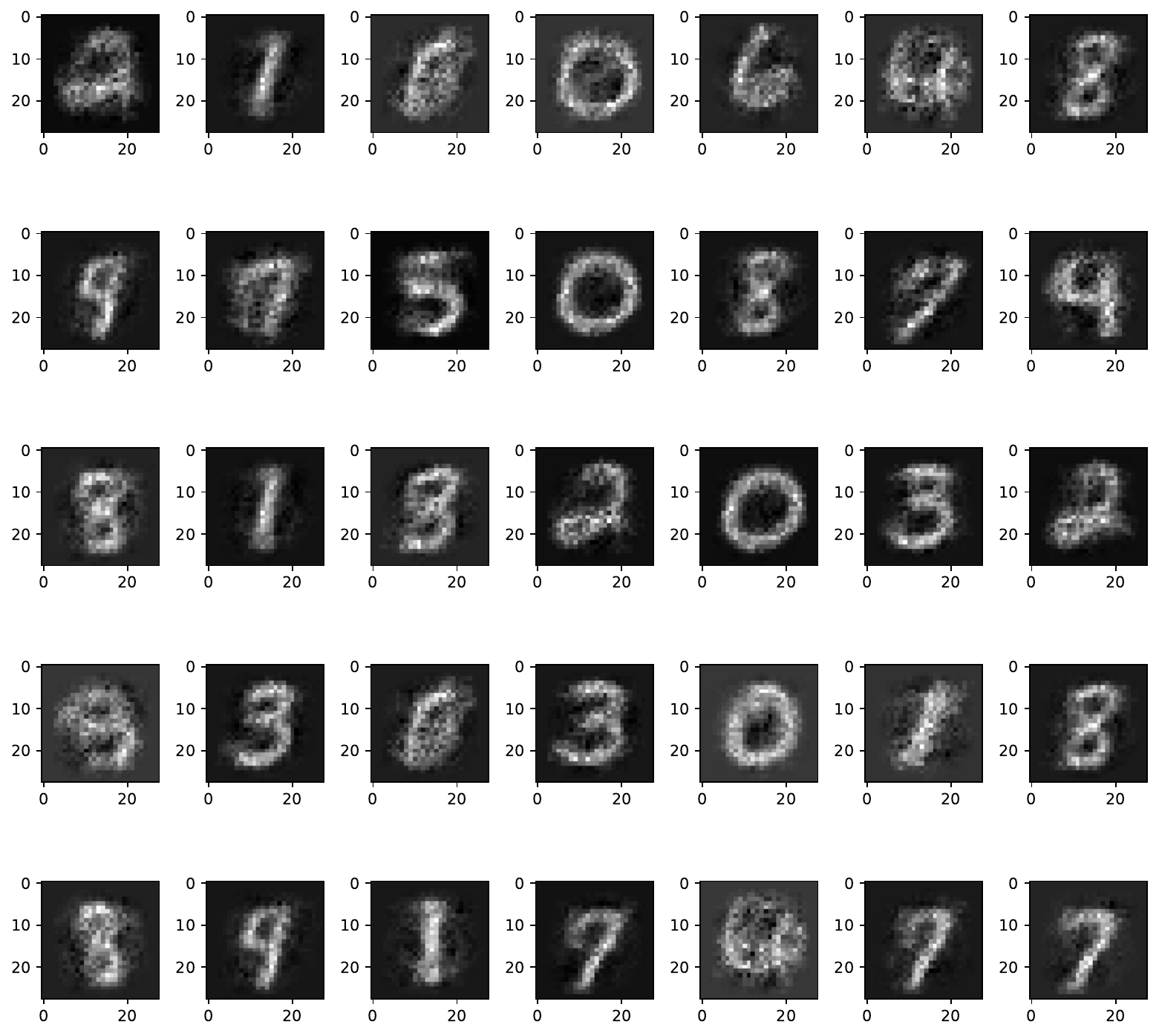}
    }  
    \caption{MNIST reconstruction task with missing pixels. From left to right: Ground truth, training images, reconstructions.}
    \vspace{-0.04in}
    \label{appx_fig:mnist_missing_illustration}
\end{figure*}

\begin{figure*}[ht!]
    \centering
    \subfloat[Brendan faces reconstruction task with 0\% missing pixels.]{
    \includegraphics[width=.32\linewidth]{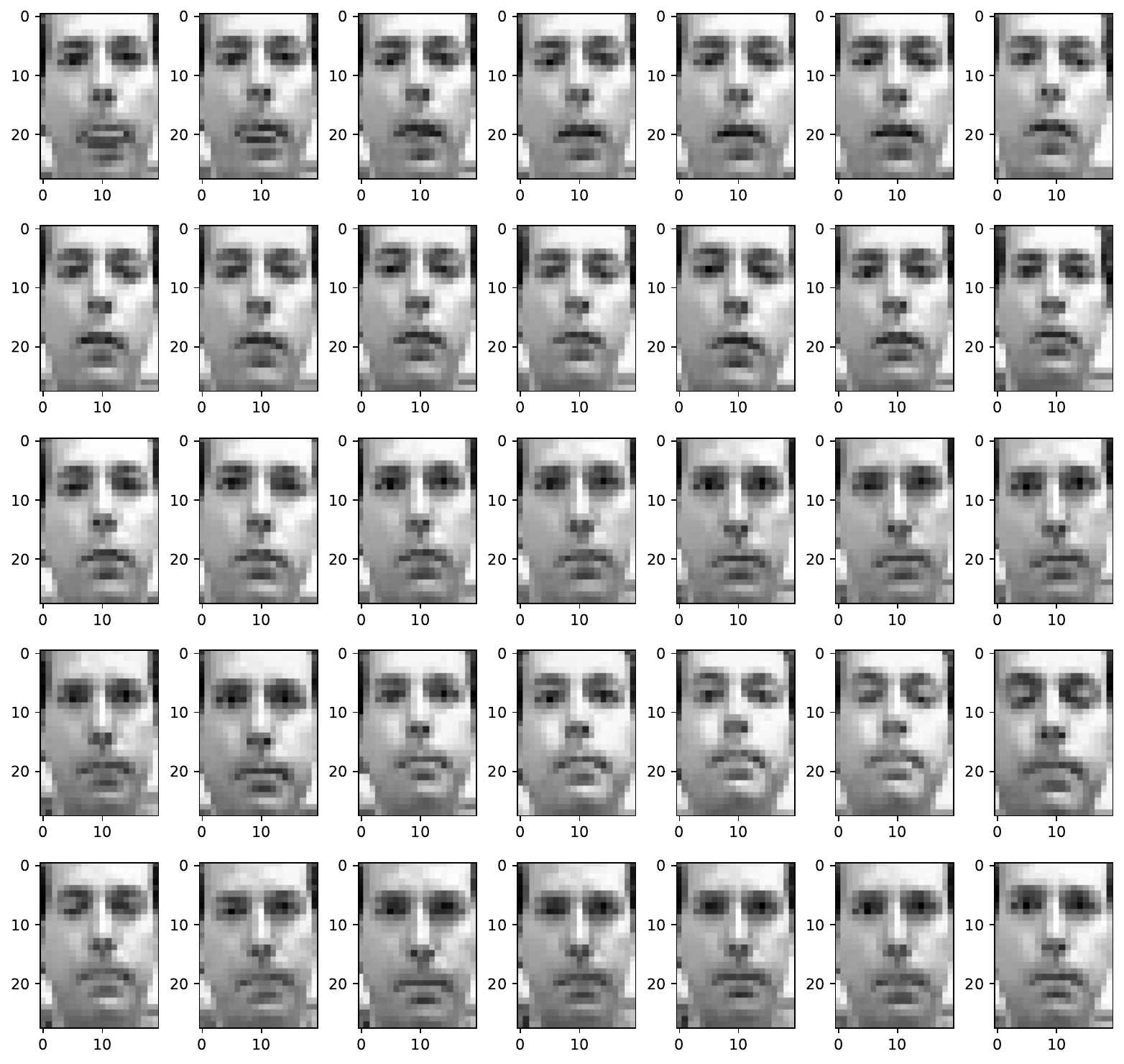} \ \ \
    \includegraphics[width=.32\linewidth]{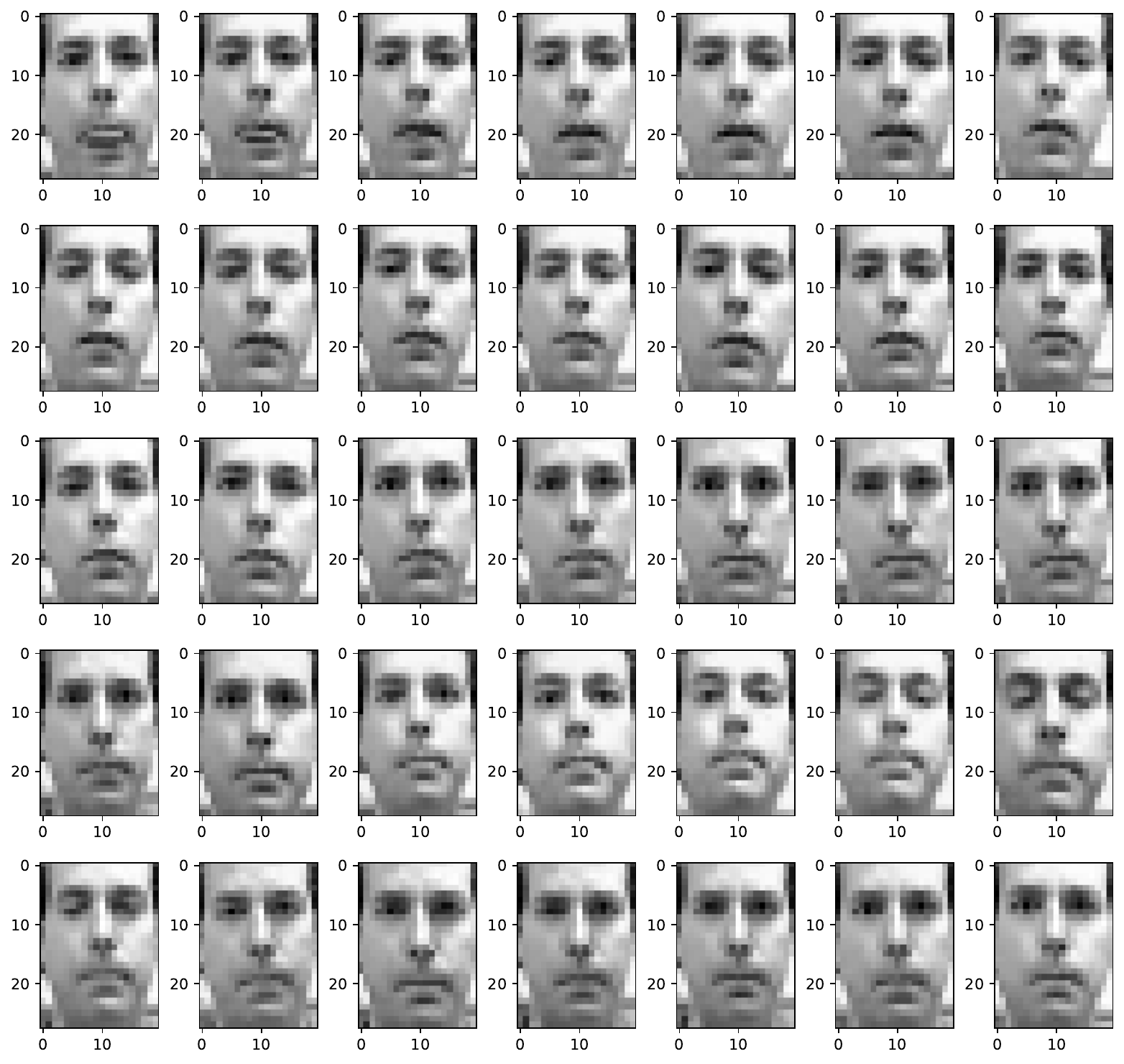} \ \ \
    \includegraphics[width=.32\linewidth]{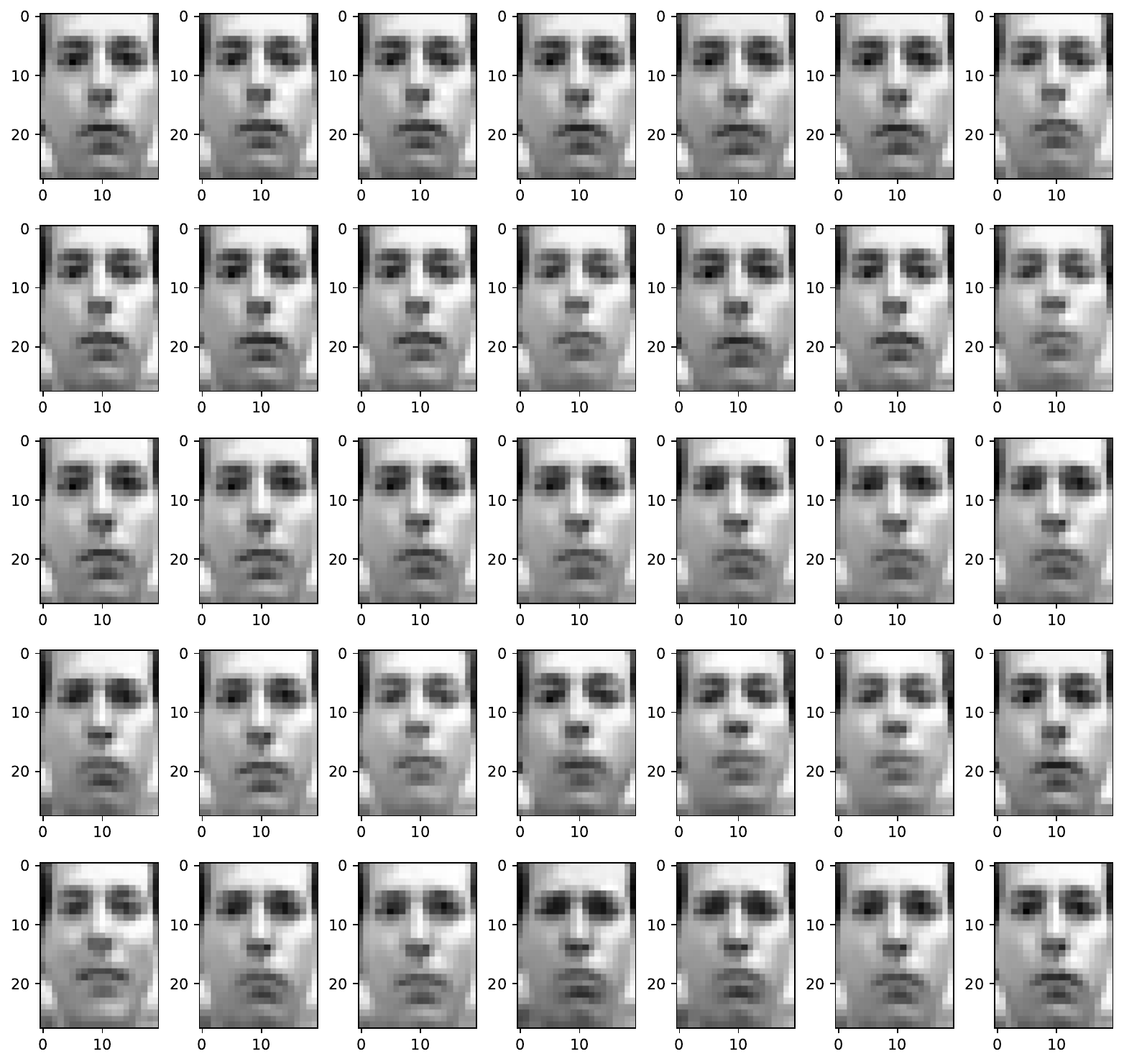}
    } \vspace{-0.05in}
    
    \subfloat[Brendan faces reconstruction task with 10\% missing pixels.]{
    \includegraphics[width=.32\linewidth]{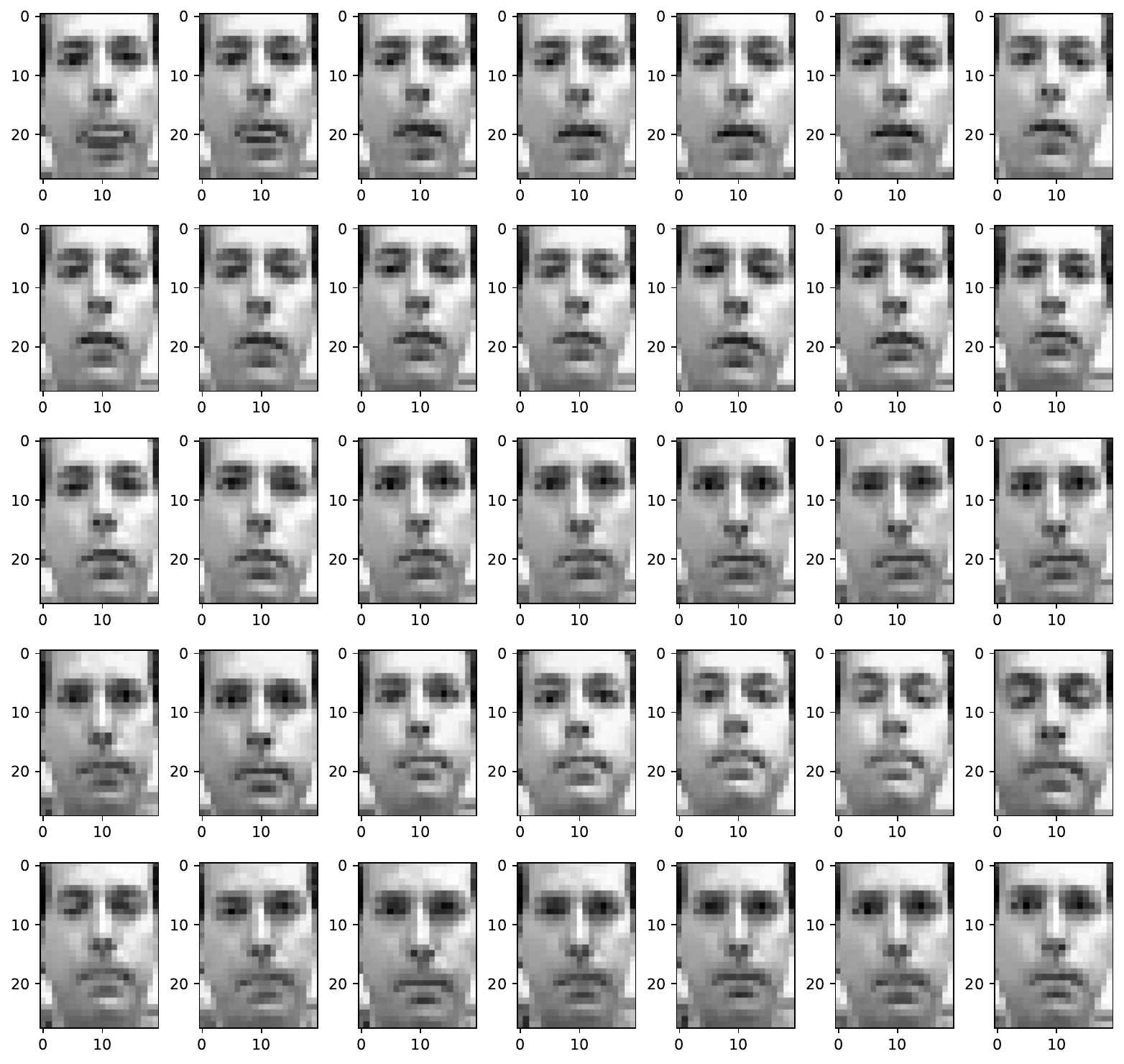} \ \ \
    \includegraphics[width=.32\linewidth]{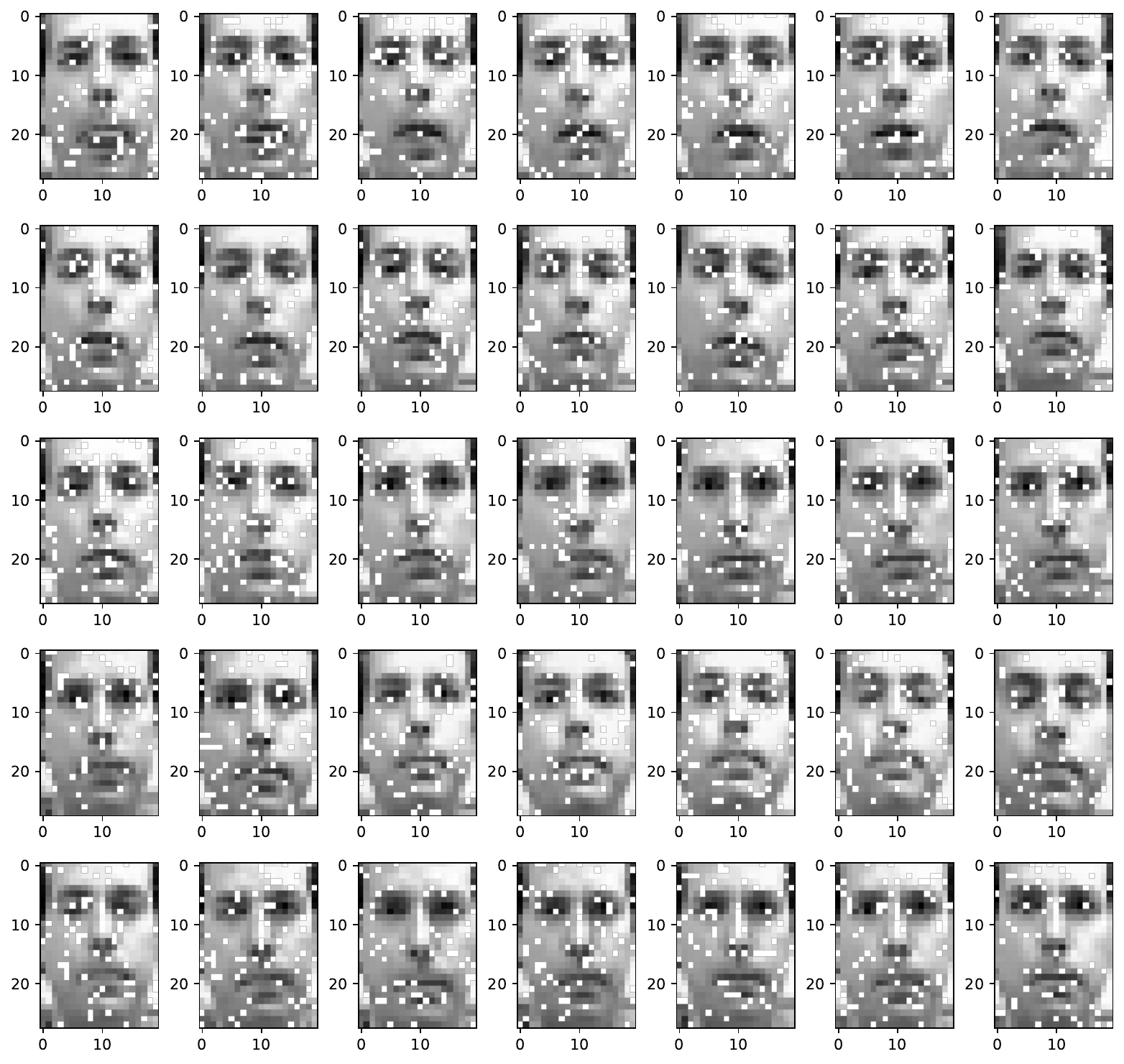} \ \ \
    \includegraphics[width=.32\linewidth]{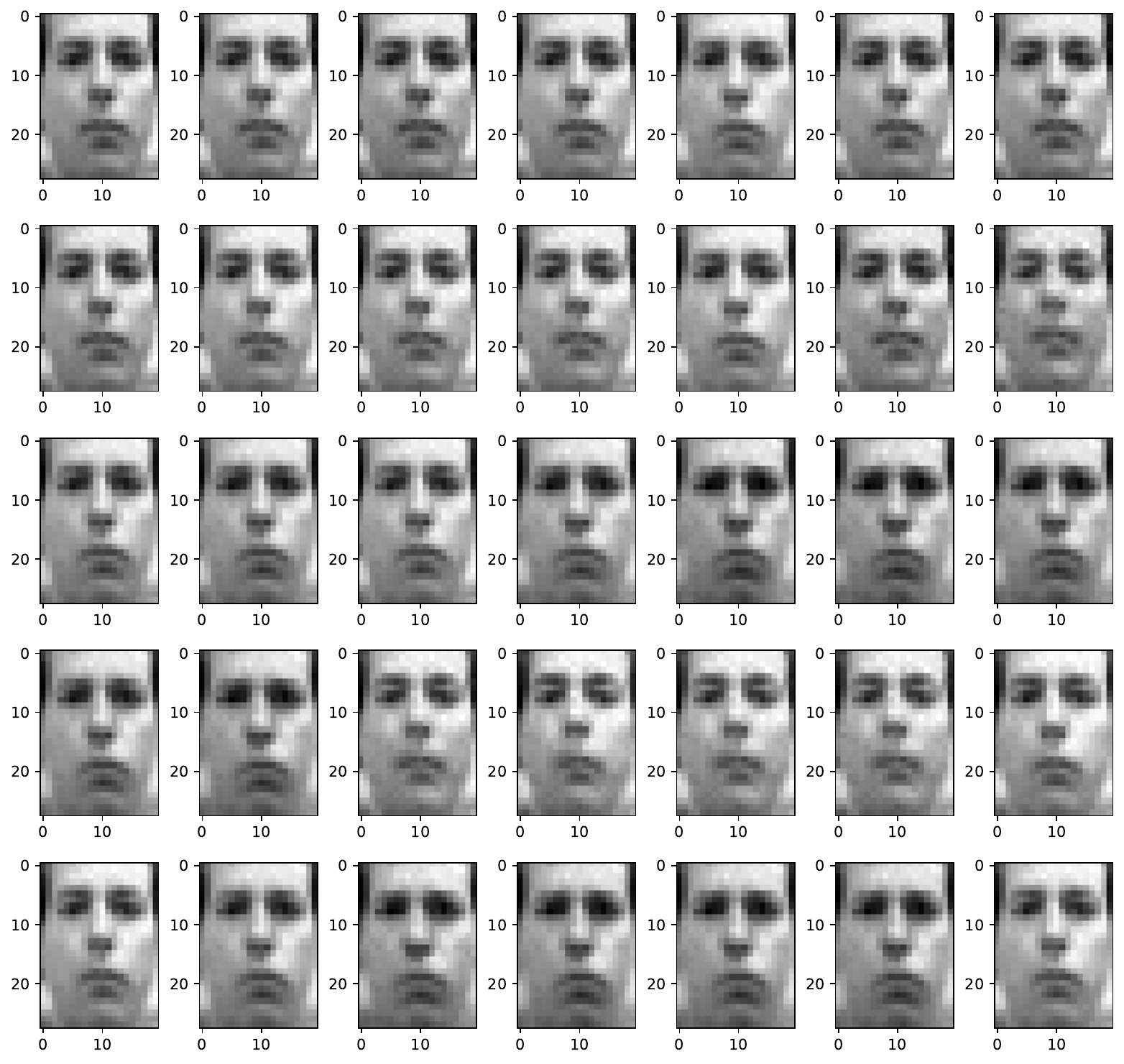}
    } \vspace{-0.05in}
    \subfloat[Brendan faces reconstruction task with 30\% missing pixels.]{
    \includegraphics[width=.32\linewidth]{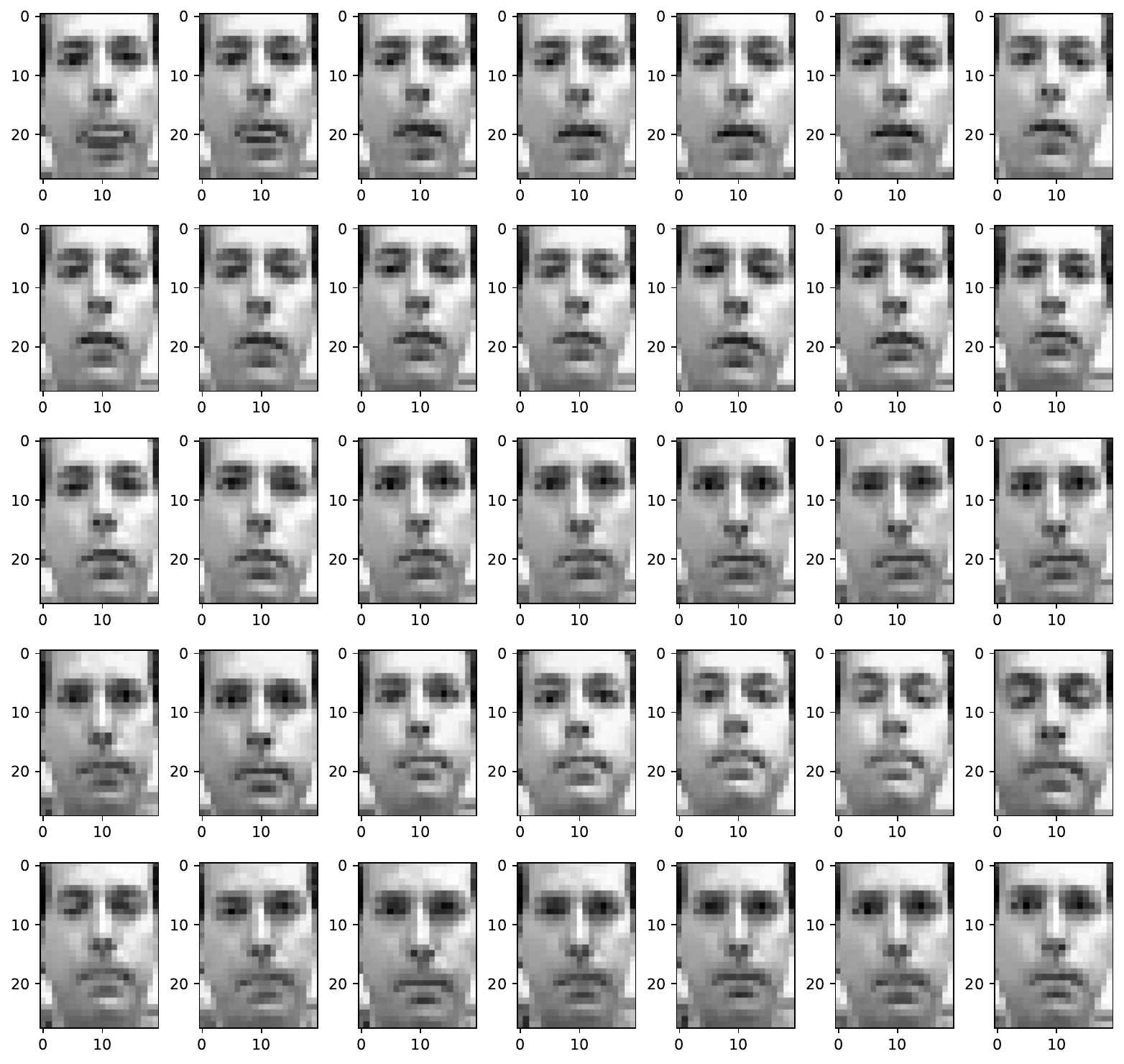} \ \ \
    \includegraphics[width=.32\linewidth]{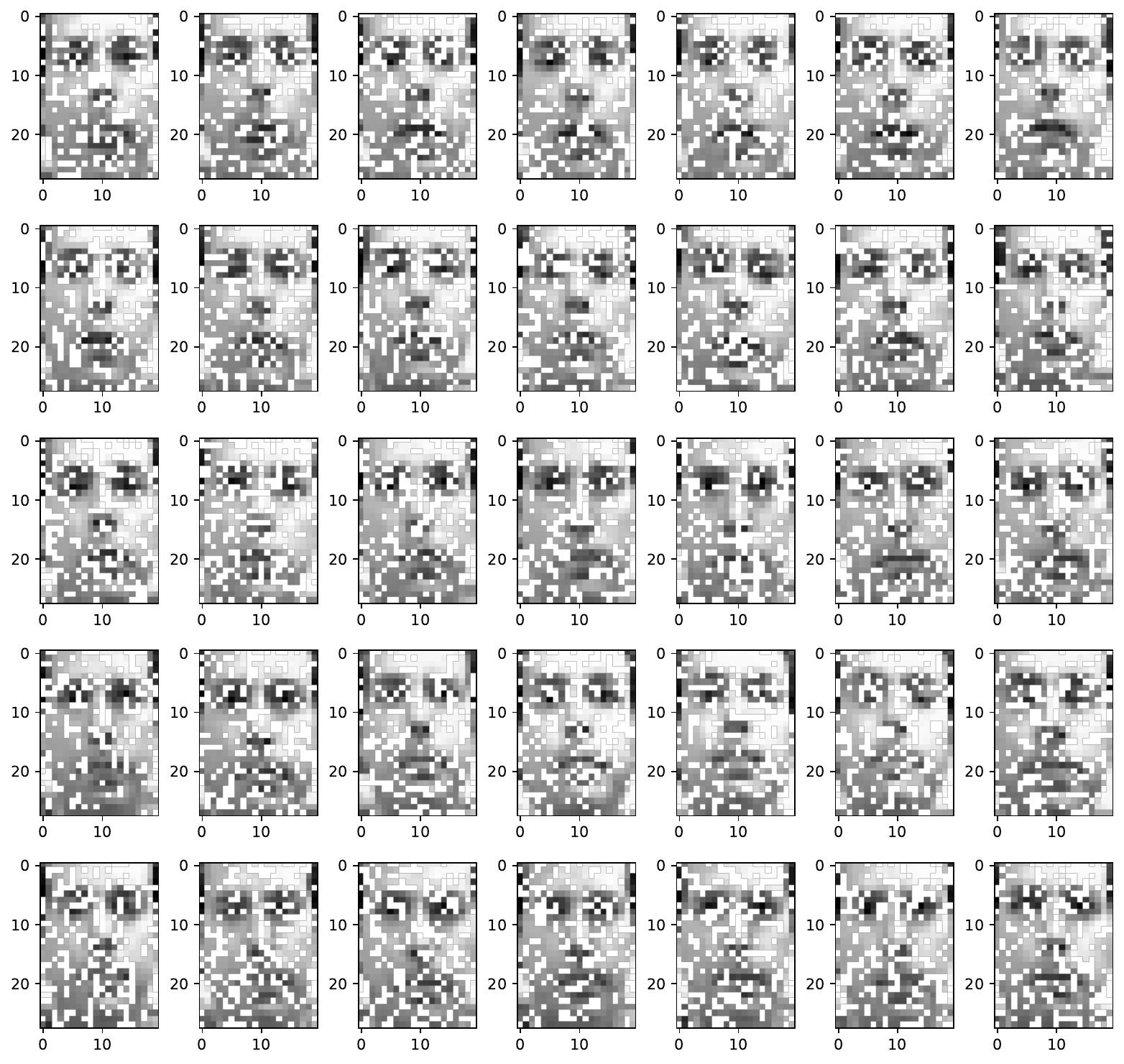} \ \ \
    \includegraphics[width=.32\linewidth]{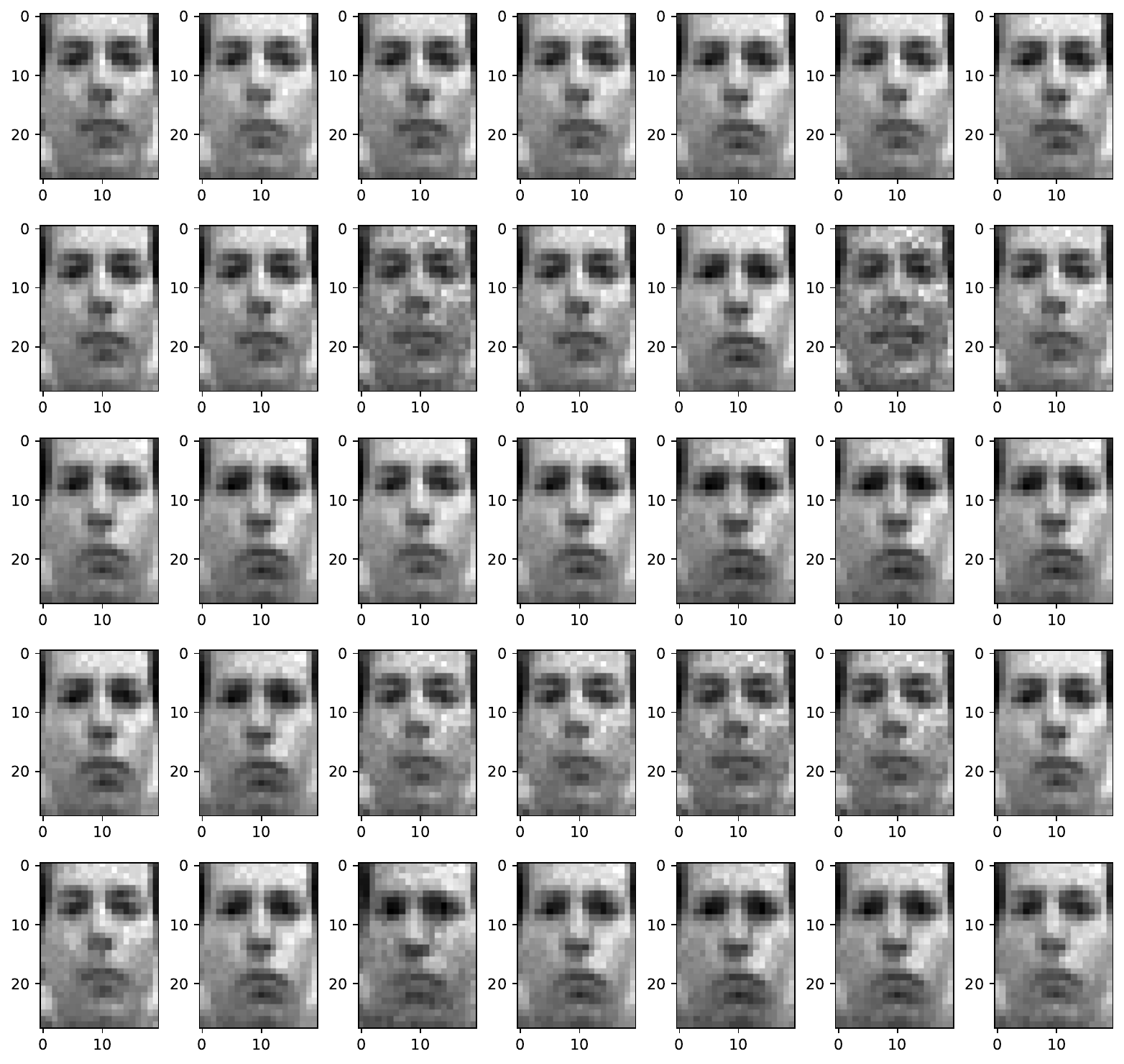}
    } \vspace{-0.1in}
    \subfloat[Brendan faces reconstruction task with 60\% missing pixels.]{
    \includegraphics[width=.32\linewidth]{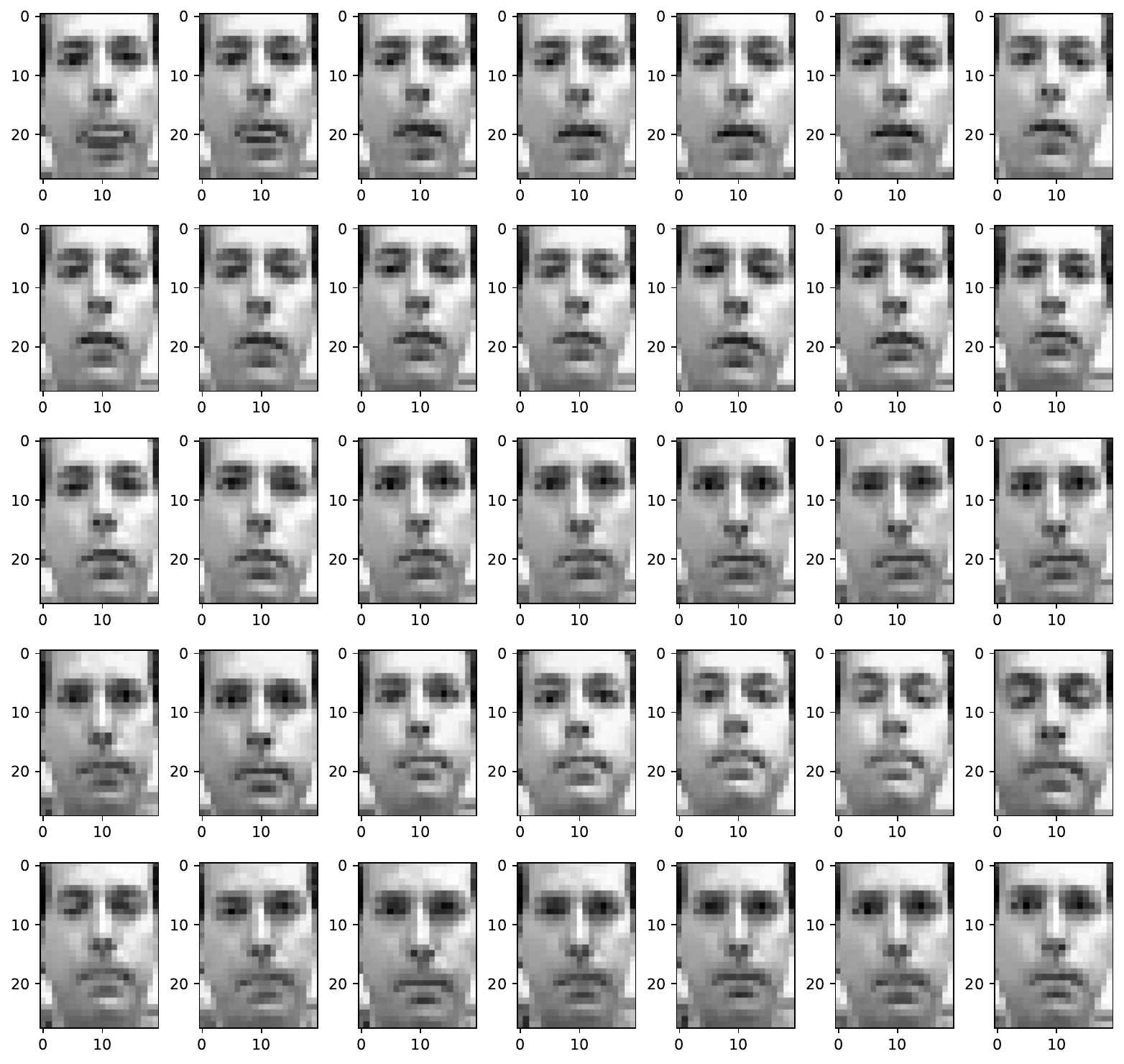} \ \ \
    \includegraphics[width=.32\linewidth]{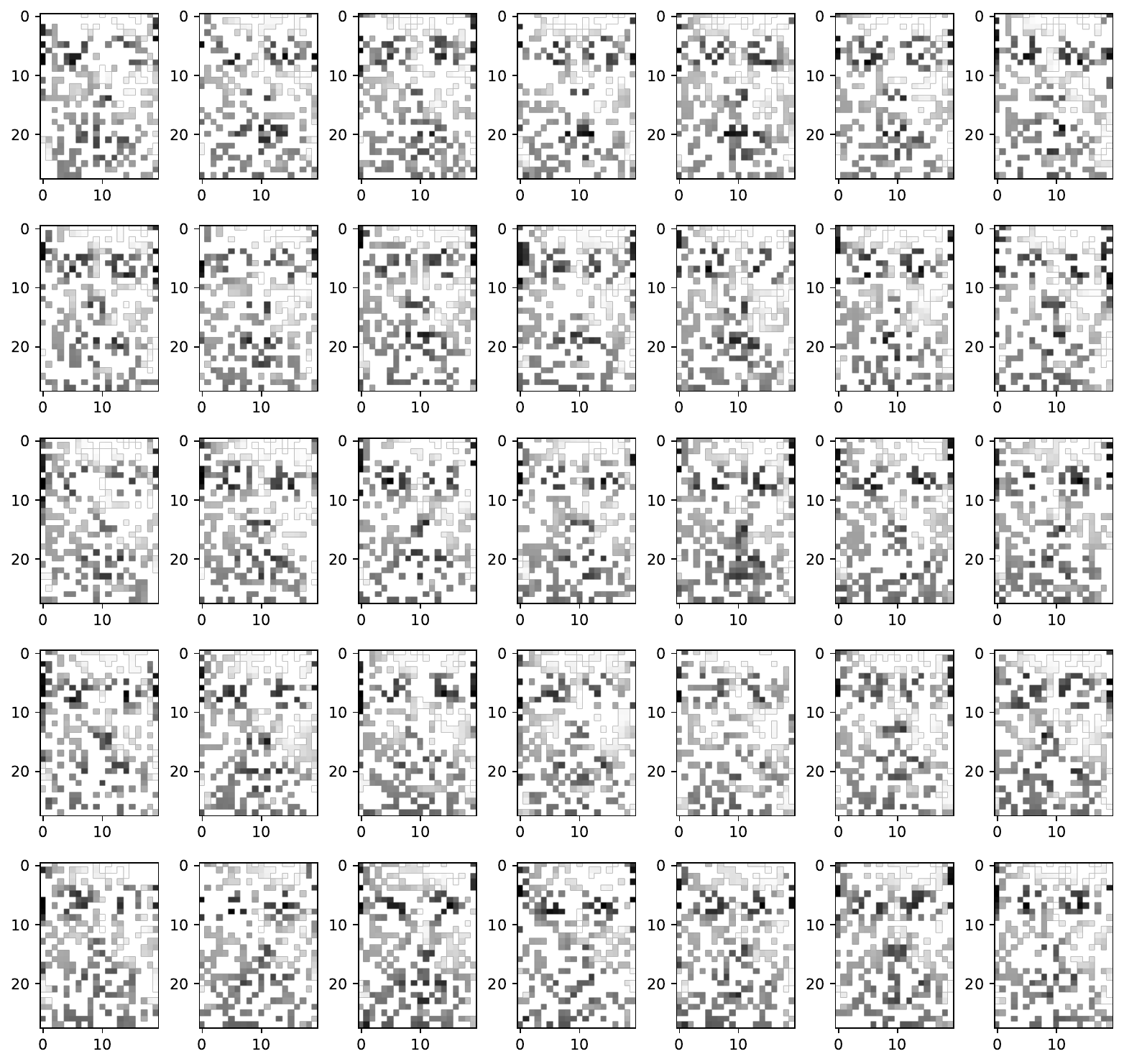} \ \ \
    \includegraphics[width=.32\linewidth]{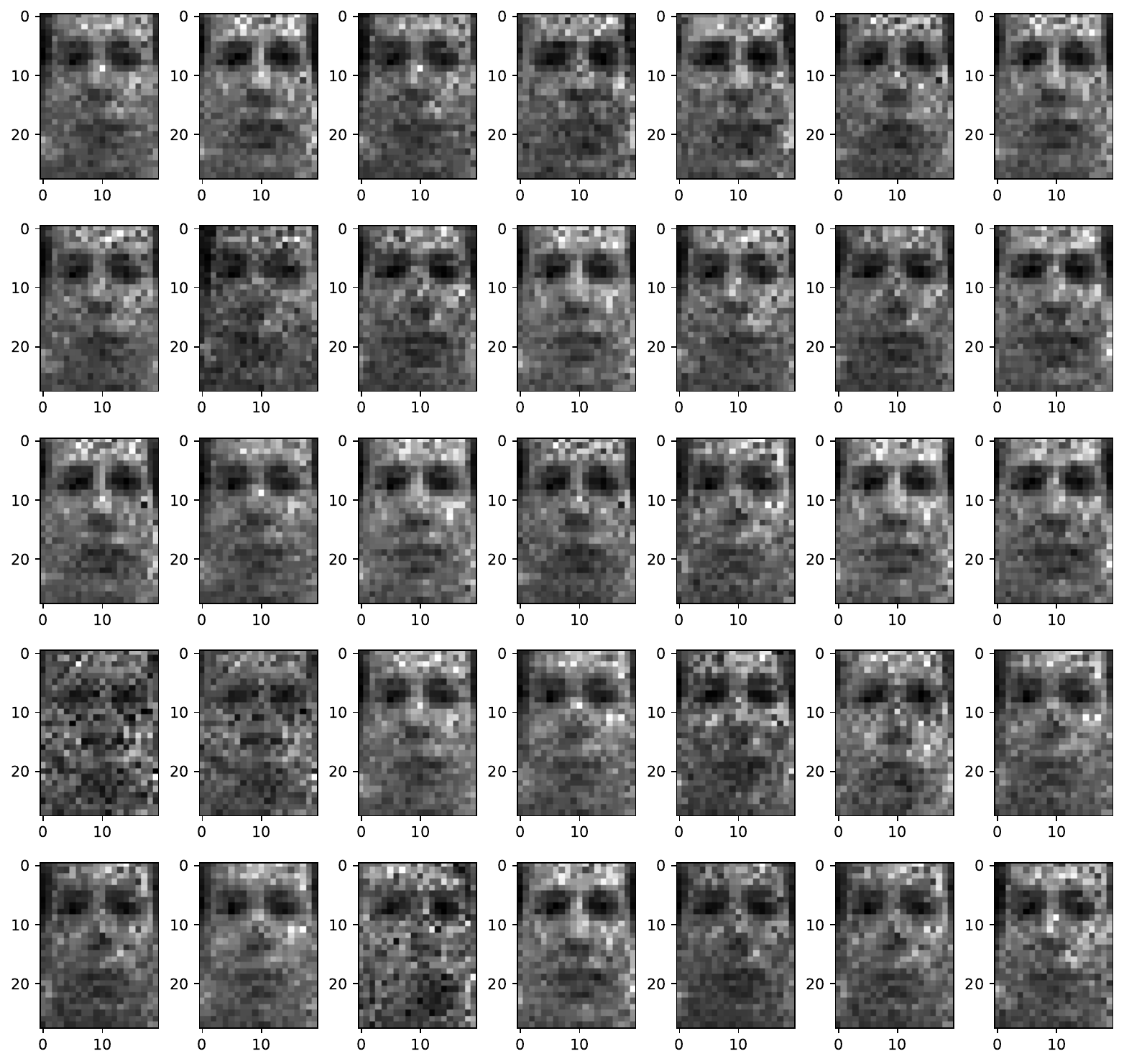}
    }  \vspace{-0.1in}
    \caption{Brendan faces reconstruction task with missing pixels. From left to right: Ground truth, training images, reconstructions}
    \label{appx_fig:bface_missing_illustration}
\end{figure*}

\end{document}





%% file: content/intro.tex
\IEEEraisesectionheading{
\section{Introduction}\label{sec:intro}}

\IEEEPARstart{O}{ne} classic problem in the field of machine learning is to learn useful latent representations of high-dimensional data in order to facilitate critical downstream tasks, e.g., classification 
\cite{tran2018representation, aguerri2019distributed}.  Latent variable models (LVMs) \cite{bishop1998hierarchical, song2019harmonized} are particularly useful due to their superior performance in unsupervised learning, especially in the task of dimensional reduction \cite{bishop1998hierarchical}. 

LVMs posit a relationship between the low-dimensional latent variables and the high-dimensional observations via a mapping function \cite{park2015bayesian, chalupka2013framework, gneiting2002compactly}. The latent variable, in essence, distills the information of the high dimensional data into a low dimensional representation \cite{quinonero2005unifying,  titsias2009variational, wilson2015kernel, gramacy2008bayesian}. The choice of the mapping is a key factor in LVMs. When a linear mapping is employed, linear LVMs such as principal component analysis (PCA) \cite{roweis1997algorithms}, canonical correlation analysis (CCA) \cite{8928538}, factor analysis (FA) \cite{lawley1962factor} among others, are incapable of capturing the intricate and complex patterns inherent in the observed data. To address this issue, neural networks and Gaussian processes (GPs) \cite{williams1995gaussian,marinai2005artificial}, which possess a universal function approximation capability, are a typical choice of non-linear mapping functions, leading to well-known LVMs such as variational autoencoders (VAEs) \cite{kingma2013auto} and Gaussian process latent variable models (GPLVMs) \cite{lawrence2005probabilistic}.

VAEs are a type of LVM where the mapping of the observed data to the latent space (and vice versa) are parameterized by a neural network \cite{kingma2013auto}. 
Despite the fact that VAEs exhibit extraordinary representational capabilities, they are still plagued by an issue known as ``posterior collapse." This issue amounts to the failure to capture the essential information contained within the observational data in certain scenarios \cite{qian2022learning, shao2021controlvae}. This problem is, in part, attributed to overfitting \cite{bowman2015generating, sonderby2016train}.

The overfitting issue could be naturally addressed by GPLVMs \cite{lawrence2005probabilistic}, which are Bayesian models and where the mapping function is assumed to be sampled from a GP prior \cite{williams1995gaussian}. Unfortunately, GPLVMs still exhibit inferior performance in learning an informative latent space, primarily due to limited approximation capacity of the typically used kernel function 
\cite{lawrence2005probabilistic, li2024preventing}. To address this issue, random Fourier feature latent variable models (RFLVMs) \cite{gundersen2021latent} employ a random Fourier feature approximation of the kernel function in order to facilitate a Markov chain Monte Carlo (MCMC) sampling algorithm for posterior inference under various likelihoods, in conjunction with a Dirichlet process (DP) kernel mixture for modeling flexibility \cite{rahimi2008random,oliva2016bayesian}.

Nevertheless, the application of the MCMC inference to RFLVMs 
presents several significant challenges. First, MCMC inference--\revise{characterized by its high computational costs}--limits the applicability of RFLVMs to large data sets \cite{cheng2022rethinking}, which renders RFLVMs prohibitive to use in the big data era. Second, the necessity for ongoing monitoring of Markov chain convergence introduces hyperparameters, e.g., burn-in steps, to the inference process. Optimizing and tuning these hyperparameters complicates model selection, consequently increasing resource and time consumption in computational efforts \cite{cai2015challenges}.

Due to the drawbacks of employing an MCMC sampling algorithm, we explore a variational Bayesian inference (VBI) approach \cite{blei2017variational}. The fundamental idea behind VBI is first to reformulate posterior inference as a deterministic optimization problem. This approach aims to approximate the true posterior with a surrogate distribution (or optimal variational distribution), which is identified from a predefined family (or variational distributions) indexed by variational parameters. VBI then proceeds by finding a set of variational parameters through maximizing the evidence lower bound (ELBO), which involves taking the expectation of {the log model joint distribution subtracted by the log variational distributions}, with respect to the variational distributions.

As an optimization-based approach, VBI has demonstrated scalable performance in contrast to sampling-based inference methods. Additionally, due to its optimization nature, the convergence of VBI is reached when the parameters stabilize, thereby circumventing the model selection complexity associated with ensuring convergence. Motivated by these properties of variational inference, we aim to devise a VBI approach tailored for RFLVMs.  However, several intricate issues arise when implementing a variational algorithm. 
First, the kernel learning portion of the model depends on a DP, a stochastic sampling procedure that \textbf{\textit{lacks explicit PDFs}}, thus rendering the ELBO calculation intractable \cite{blei2006variational}.  Second, \revise{\textbf{\textit{the numerous variational variables}}} within RFLVMs pose challenges to the optimization problem in VBI.

The most widely adopted approach for dealing with the problem of numerous variational variables is mean-field variational inference (MFVI) \cite{bishop2006pattern}, renowned for its ease of implementation. By leveraging the assumption of a \textbf{fully factorized} $q(\bTheta)$, MFVI illustrates the potential for obtaining optimal variational distributions through simplified integration. Nevertheless, the tractability of integration is built upon certain conditional conjugacy structures, thereby limiting the applicability of MFVI to non-conjugate models, e.g., RFLVMs.  

For handling non-conjugate models, reparameterization gradient variational inference (RGVI) uses gradient-based methods for variational parameter updates, where the low-variance gradient is estimated by Monte Carlo (MC) samples generated via the reparameterization trick \cite{zhang2018advances, kingma2013auto, kingma2014adam, figurnov2018implicit}.  Despite its notable success across a wide variety of models, RGVI remains restrictive by excluding several standard distributions, as in the case of non-parametric distributions like DPs, from utilizing reparameterization gradients. This limitation reduces its applicability. Furthermore, RGVI typically does not leverage conjugacy in gradient estimation, leading to several issues. Specifically, updates can be affected by the over-parameterization of certain variational distributions (despite being analytically calculable), and this may result in an excessive number of variational parameters and slower convergence rates.  

\subsection{Contributions}
While recent advancements in VBI have been dedicated to enhancing its applicability, it remains incapable of being applied to RFLVMs. To address this limitation, this paper proposes a novel VBI algorithm designed to facilitate its utilization in general models, including RFLVMs. The key contributions of our work are:  
\begin{itemize}
    \item \textbf{Stick-breaking construction of the DP}. \revise{We introduce the stick-breaking construction of DPs \cite{blackwell1973ferguson}, assuming a multinomial distribution for indicator variables. This approach enables the establishment of explicit PDFs for the indicator variables in the DP mixture used in RFLVMs, allowing us to explore the space of stationary kernels.}
    \item \textbf{Block coordinate descent variational inference}. We introduce a novel algorithm for VBI termed ``block coordinate descent variational inference" (BCD-VI). The fundamental idea behind BCD-VI is to partition the variational variables into distinct blocks, and then sequentially optimize the subproblems formulated by each block with a suitable solver (RGVI, or proposed MFVI-based) that is carefully selected.  More specifically, it uses RGVI for the non-conjugate terms, while retaining the computational efﬁciency of conjugate computations on the conjugate terms. Consequently, our proposed BCD-VI significantly simplifies problem-solving, improves  algorithmic efficiency, and thus, extends the applicability of VBI to a wider array of models. 
    \item \textbf{Scalable RFLVMs}. 
    We employ the BCD-VI algorithm on RFLVMs with the stick-breaking construction. We, facilitate the development of VBI for RFLVMs, ultimately leading to a scalable version of RFLVMs, namely \srflvm. Our proposed \srflvm~ method demonstrates scalability, computational efficiency, and superior performance in learning informative latent variables and imputing missing data, surpassing state-of-the-art (SOTA) LVM variants across a diverse set of datasets. 
\end{itemize}

\subsection{Paper Organization}
The structure of the remaining sections in this paper is summarized as follows. Section \ref{sec:pro_state} discusses prior work in LVMs and highlights challenges encountered in the application of VBI to RFLVMs. In Section \ref{sec:bcd-vi}, we introduce the stick-breaking construction and the BCD-VI algorithm to address those challenges, which leads to the development of \srflvm~ in Section \ref{sec:srflvm}. 
Based on the proposed approach, numerical results and discussions are presented in Section \ref{sec:results}. Finally, Section \ref{sec:conclusions} summarizes the conclusions derived from our work.   


%% file: content/gplvm_rflvm.tex
\section{Preliminaries and Problem Statement}
\label{sec:pro_state}

Section \ref{subsec:RFLVM-preliminary} presents an introduction to RFLVMs and outlines limitations associated with the existing sampling-based inference methods \cite{ gundersen2021latent}.  To address these limitations, Section \ref{subsec:VIandChallenge-preliminary}  discusses an alternative approach known as VBI. We detail the foundational concepts underlying VBI and examine the challenges it faces when integrated with RFLVMs. 

\subsection{Random Feature Latent Variable Models}
\label{subsec:RFLVM-preliminary} 

\subsubsection{Gaussian Process Latent Variable Models}
\label{subsubsec:GPLVMs}

A GPLVM \cite{lawrence2005probabilistic} establishes a relationship between an observed dataset $\mathbf{Y} \in \mathbb{R}^{N \times M}$ and latent variables $\mathbf{X} \in \mathbb{R}^{N \times Q}$ using a set of GP-distributed mappings, $\{ f_m(\cdot):\!\mathbb{R}^{Q}\!\mapsto\!\mathbb{R}, m \in 1, ..., M \}$. Formally, a GPLVM with Gaussian likelihood can be expressed as
\begin{subequations}
\label{eq:gplvm}
\begin{align}
    & \y_{:, m} \vert f_m({\X}) \sim \mathcal{N}_N( f_m(\X), \sigma^{2} \I_N), & m \in 1, \dots, M, \\ 
    & f_m \sim \mathcal{GP}(0, \kappa),
    & m \in 1, \dots, M, \\ 
    & \x_n \sim  \mathcal{N}_{Q}(\mathbf{0}, \mathbf{I}_{Q}), & n \in 1, \dots, N,  \label{eq:gplvm_prior_x}
\end{align}
\end{subequations}
where the subscript $m$ indexes the element of the output vector $\y$, $\mathcal{GP}(0, \kappa)$ stands for a GP with zero mean and kernel function $\kappa$, and 
the subscript of $\mathcal{N}$ denotes the dimension of the distribution. Here, $\sigma^{2}$ represents the noise variance, and $\mathbf{y}_{:, m} \in \mathbb{R}^N$ denotes the vector of observations for the $m$-th dimension. The function $f_m(\mathbf{X})$ is defined as $f_m(\mathbf{X}) = \left[f_m \left(\mathbf{x}_1\right), \ldots, f_m\left(\mathbf{x}_N\right)\right]^{\top} \in \mathbb{R}^{N}$. Additionally, we define $ \mathbf{K}_{\X} \in \mathbb{R}^{N \times N} $ as a covariance matrix evaluated on the finite input $\X$ with a positive definite kernel function $\kappa(\x, \x^{\prime})$, i.e., $\left[ \mathbf{K}_{\X} \right]_{i, j} =  \kappa(\x_i, \x_j)$. In the context of GPLVMs, the radial basis function (RBF) is a typical choice of kernel function \cite{williams2006gaussian}. Due to the conjugacy between the GP prior and the Gaussian likelihood, each mapping can be marginalized out, resulting in the marginal likelihood:
\begin{align}
    &\y_{:, m}\sim \mathcal{N}_N ({\bf 0}, \mathbf{K}_{\mathbf{X}} + \sigma^2 \mathbf{I}_N), & m \in 1, \dots, M. 
    \label{eq:gplvm_mle}
\end{align}
Built upon this likelihood and the prior from Eq.~\eqref{eq:gplvm_prior_x}, we can derive the \textit{maximum a posteriori} (MAP) estimate for the latent variable $\mathbf{X}$ using gradient-based optimization methods, with computational complexity that scales $\mathcal{O}(N^3)$ for $N$ observations \cite{lawrence2005probabilistic}. A critical limitation of the GPLVM is its reliance on Gaussian likelihood as it may be an unrealistic assumption for certain data types (like count data, for example). To overcome this, RFLVMs integrate RFFs with weight-space approximations of GPs, thereby generalizing GPLVM to support diverse data likelihoods \cite{rahimi2008random,gundersen2021latent}; see further discussions of related work on GPLVMs in Section~\ref{sec:related_work}. 

\subsubsection{Random Fourier Features}
\label{subsubsec:RFFs}

According to Bochner’s theorem \cite{bochner1959lectures}, a continuous stationary kernel function $\kappa(\x, \x^{\prime}) = \kappa(\x - \x^{\prime})$ is positive definite if and only if $\kappa(\cdot)$ is the Fourier transform of a finite positive measure $p(\mathbf{w})$. If the kernel is properly scaled,  $p(\mathbf{w})$ is guaranteed to be a density \cite{bochner1959lectures}, that is
\begin{align}
\label{eq:bachner_theorem}
   \kappa(\x - \x^{\prime})=\int p(\w) \exp \! \left({i \w^{\top} (\x - \x^{\prime})}\right) \mathrm{~d} \w. 
\end{align}
Built upon Eq.~\eqref{eq:bachner_theorem}, the RFFs methods try to obtain an unbiased estimator of the kernel function \cite{rahimi2008random}, which is summarized in the following theorem. 

\begin{theorem}
Let $\kappa(\mathbf{x}, \mathbf{x}^\prime)$ be a positive definite stationary kernel function, and let $\vvarphi(\mathbf{x})$ be an associated randomized feature map defined as follows:
\begin{equation}
    \sqrt{\frac{2}{L}}\!\!\begin{bmatrix}
    \sin(\w_1^{\top} \x), \cos(\w_1^{\top} \x),\!\cdots,\!\sin(\w_{L/2}^{\top} \x), \cos(\w_{L/2}^{\top} \x) 
    \end{bmatrix}^{\!\top}\!,
    \label{eq:rff_def}
\end{equation}
where $\mathbf{W} \triangleq \{\mathbf{w}_l\}_{l=1}^{L/2} \in \mathbb{R}^{\frac{L}{2} \times Q}$ are independent and identically distributed (i.i.d.) random vectors drawn from the spectral density  $p(\mathbf{w})$. Then, an unbiased estimator of the kernel $\kappa(\mathbf{x}, \mathbf{x}^\prime)$ using RFFs is given by:
\begin{equation}
\kappa(\mathbf{x}, \mathbf{x}^\prime) \approx \vvarphi(\mathbf{x})^\top \vvarphi(\mathbf{x}^\prime).  \label{eq:kernel_app}
\end{equation} 
\end{theorem}
\begin{proof}
See Appendix \ref{app:kernel_estimator} or \cite{rahimi2008random}. 
\end{proof}

\subsubsection{Weight-Space View of GPs}
\label{subsubsec:weight-space-approximations-GPs}

The representer theorem \cite{williams2006gaussian} states that the optimal solution for the mapping in kernel-based methods can be expressed as a linear combination of kernel evaluations between data points, specifically as $f_m^{\star}(\mathbf{x})\!=\!\sum_{n=1}^{N} \alpha_{nm} \kappa(\mathbf{x}_n, \mathbf{x})$. By integrating the unbiased estimator provided by RFFs (Eq.~\ref{eq:kernel_app}) with this optimal mapping $f_m^{\star}(\mathbf{x})$, a weight-space approximation of GP is obtained: 
\vspace{-0.05in}
\begin{equation}
\label{eq:RFF_approximation}
\begin{aligned}
     f_m^{\star}(\mathbf{x})\approx\sum_{n=1}^{N} \alpha_{nm} \vvarphi(\mathbf{x}_n)^{\top} \vvarphi(\mathbf{x})\triangleq\ba_m^{\top} \vvarphi(\mathbf{x}),
\end{aligned}
\end{equation}
where $\ba_m \!\sim\! \mathcal{N}(\mathbf{0}, \mathbf{I}_{L})$ are stochastic weights, and where we define $
\bA \triangleq \{ \ba_m \}_{m=1}^{M} \in \mathbb{R}^{M \times L}$. Substituting this approximation into the GPLVM (see Eq.~\eqref{eq:gplvm}) gives rise to the generalized GPLVMs as follows \cite{gundersen2021latent}:
\vspace{-0.05in}
\begin{subequations}
\label{eq:gene_gplvm}
\begin{align}
    & \y_{:, m} \vert \ba_m \sim p( g(\bPhi \ba_m), \boldsymbol{\zeta}), & m \in 1, \dots, M, \label{eq:mimc_marg}  \\ 
    & \ba_m \sim \mathcal{N}_{L}(\mathbf{0}, \mathbf{I}_{L}), & m \in 1, \dots, M, \label{eq:prior_h} \\
    & \x_n \sim  \mathcal{N}_{Q}(\mathbf{0}, \mathbf{I}_{Q}), & n \in 1, \dots, N,   \label{eq:prior_x}
\end{align}
\end{subequations}
where $p(\cdot,\cdot)$ stands for some probability distribution, $\bPhi\!\triangleq\!\left[ \begin{array}{ccc}\! \vvarphi(\x_1),\!\cdots,\!\vvarphi(\x_N)\!\end{array} \right]^{\top}\!\!\in\!\!\mathbb{R}^{N \times L}$, $g(\cdot)$ denotes an invertible link function that maps real vectors onto the support of {the parameters of $p(\cdot,\cdot)$}, and $\boldsymbol{\zeta}$ represents other parameters specific to $p(\cdot,\cdot)$. However, despite the versatility of generalized GPLVMs in handling diverse data, their reliance on a primary kernel function, such as the RBF kernel, often results in inaccurate learning of the underlying function map, thereby degrading the quality of latent variable inference \cite{li2024preventing}. 
\begin{figure}[t]
    \centering
    \includegraphics[width=0.70\linewidth]{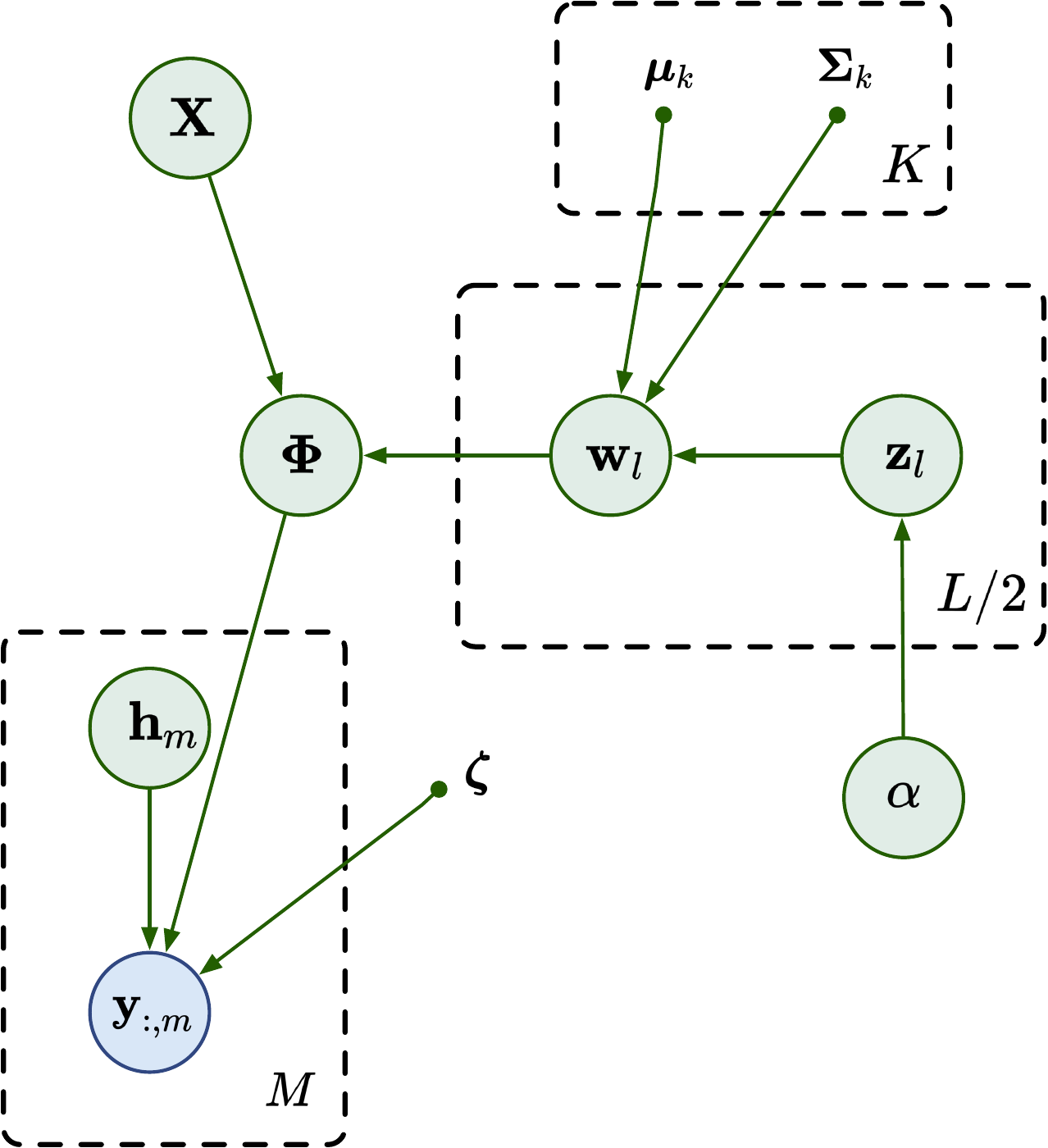}
    \caption{\textbf{Graphical Model of RFLVMs}. We use arrows to denote the dependency relations between variables. The blue and green circles denote the observed and variational variables, respectively. Smaller dots indicate the deterministic parameters of the model. The nodes surrounded by a box represent that there are $K$, $M$ or $L/2$ nodes of this kind.}
    \label{fig:graph_rflvm}
\end{figure}

\subsubsection{Random Feature Latent Variable Models}
\label{subsubsec:RFLVMs}

 To enhance kernel flexibility, RFLVMs integrate kernel learning within generalized GPLVMs \cite{gundersen2021latent}. Particularly, by leveraging Bochner's theorem (Eq.~\eqref{eq:bachner_theorem}), which establishes a duality between the spectral density and the kernel function, a category of kernel learning methods approximates the spectral density of the underlying stationary kernel. In RFLVMs, a DP mixture kernel is employed, with its spectral density $p(\mathbf{w})$ defined as follows: 
\vspace{-0.05in}
\begin{subequations}
    \begin{align}
        & \w_{l} \vert \z_l \sim \mathcal{N}_{Q} \left(\bmu_{\z_{l}}, \bSigma_{\z_{l}} \right), & l \in 1, \dots, L/2, \\ 
        & \z_l \sim \operatorname{CRP}(\alpha), & l \in 1, \dots, L/2, \\ 
        & \alpha \sim \operatorname{Ga}(\alpha_0, \beta_0) \label{eq:alpha}, 
    \end{align}
\end{subequations}
where each spectral frequency vector $\w_l$ is distributed with a mixture component which is decided by the associated integer variable (or indicator variable) $\z_l = k \in \{1, \ldots K\}$, and $K$ is the number of clusters. Each mixture component is characterized by the mean vector $\boldsymbol{\mu}_k \in \mathbb{R}^{Q}$ and the covariance matrix $\boldsymbol{\Sigma}_k \in \mathbb{R}^{Q \times Q}$, collectively denoted as  $\boldsymbol{\theta}_{\text{dpm}} \triangleq \{ \boldsymbol{\mu}_k, \boldsymbol{\Sigma}_k \}_{k=1}^{K}$. Each integer variable $\z_l$ in $\z \triangleq \{ \z_l \}_{l=1}^{L/2}$ follows the Chinese restaurant process (CRP) with a concentration parameter $\alpha$, which governs the degree of association for a set of indicator variables \cite{blei2006variational}. The parameter $\alpha$ is further drawn from a Gamma distribution with shape hyperparameter $\alpha_0 \in \mathbb{R}$ and rate hyperparameter $\beta_0 \in \mathbb{R}$. In summary, the graphical model of RFLVMs is shown in Fig.~\ref{fig:graph_rflvm}. Next, we provide two examples of different data {distributions} in RFLVMs. 

\begin{example}[Logistic {distribution}]
\label{example:Logistic_likelihood}
The functional expression of the data distributions such as Bernoulli and negative binomial is as follows:  
\begin{align}
    & \y_{:, m} \vert \ba_m \sim 
    p\left( \bPhi, \ba_m, 
    \boldsymbol{\zeta}\right),  & m \in 1, \dots, M,  \label{eq:logistic_likelihood}
\end{align}
where
\vspace{-0.05in}
\begin{align}
     p\left( \bPhi, \ba_m, 
     \boldsymbol{\zeta}\right) = \prod_{n=1}^N c_{nm} \frac{ \exp \left( \bPhi_{n, :}~\ba_m \right)^{a_{nm}}} {\left(1+\exp \left( \bPhi_{n, :}~ \ba_m \right)\right)^{ b_{nm} }}. 
     \label{eq:log_like}
\end{align}
In this case, $g_n(\cdot)\!=\!\exp(\bPhi_{n, :}\ba_m)$. By setting $a_{nm}\!=\!\y_{n m},$ $b_{nm}\!=\!\y_{n m}\!+\!r_m$, and $c_{nm}\!=\! \left(\begin{array}{c}\!\y_{n m}\!+\!r_m-\!1 \\ \y_{n m}\end{array}\right)$, we get the negative binomial RFLVM with $\boldsymbol{\zeta} \triangleq \{ r_m \}_{m=1}^{M}$ and $r_m$ is the feature-specific dispersion parameter. 
\end{example}
\begin{example}[Gaussian {Distribution}]
\label{example:Gaussian_likelihood}
In the case of Gaussian {distribution}, due to conjugacy, we can marginalize the GP and integrate the unbiased estimator of the kernel function using RFF. This approach leads to the RFLVM with Gaussian likelihood, as follows: \begin{align} \y_{:, m} \sim \mathcal{N}_N({\bf 0}, \bPhi \bPhi^{\top} + \sigma^{2} \I_N), ~ m \in 1, \dots, M. \label{eq:gaussian_like} \end{align} In this case, $\boldsymbol{\zeta}$ contains only $\sigma^2$. Note that this integration avoids the introduction of stochastic weights $\ba_m$.
\end{example}

Despite the fact that RFLVMs demonstrate SOTA performance in terms of informative latent variables by employing a DP mixture kernel, they still face several challenges associated with the use of MCMC inference \cite{gundersen2021latent}. First, the computational demands linked to MCMC inference limit the applicability of RFLVMs to large datasets \cite{cai2015challenges}. Second, the ongoing monitoring of Markov chain convergence introduces hyperparameters, such as burn-in steps, which increase the complexities in model selection \cite{cai2015challenges}. To address these challenges, the following subsection introduces VBI \cite{zhang2018advances}, which is well-known for its scalable performance and simplified convergence diagnostics, thereby avoiding the model selection complexities associated with ensuring convergence. For the sake of simplicity and to facilitate discussion, we will demonstrate the proposed scalable VBI method for RFLVMs with Gaussian likelihood unless otherwise specified.

\subsection{Variational Bayesian Inference and Challenges} \label{subsec:VIandChallenge-preliminary}
Let $\boldsymbol{\Theta}$ represent all variational variables, i.e., $\boldsymbol{\Theta} = \{\W, \X, \z, \alpha\}$. The variational method converts inference into an optimization problem by aiming to identify an optimal variational distribution, denoted by $q^{\star}(\boldsymbol{\Theta})$, from a predetermined distribution family $q(\boldsymbol{\Theta}; \boldsymbol \vartheta)$ indexed by variational parameters $\boldsymbol \vartheta$. This identified variational distribution should closely approximate the posterior distribution $p(\boldsymbol{\Theta} \vert \mathbf{Y})$ in terms of Kullback-Leibler (KL) divergence. Mathematically, the optimization task can be formulated as follows:
\vspace{-0.05in}
\begin{equation}
    \min_{\boldsymbol{\vartheta}} \ \operatorname{KL}\left(q(\bTheta; \boldsymbol \vartheta) ~\|~ p(\bTheta \mid \Y)\right),
    \label{eq:KL_VI}
\end{equation}
where $\operatorname{KL}(\cdot | \cdot)$ denotes the KL divergence between the two arguments. Because of the intractability of the posterior distribution, the equivalent optimization problem is to maximize the ELBO \cite{cheng2022rethinking}, $\mathcal{L}(q(\boldsymbol{\Theta};\boldsymbol \vartheta)$, i.e., 
\vspace{-0.05in}
\begin{align}
\begin{aligned}
& \max_{\boldsymbol{\vartheta}} \mathcal{L}(q(\bTheta; \boldsymbol \vartheta)) \triangleq \mathbb{E}_{q(\bTheta)}\left\{\log \frac{p(\bTheta, \Y)}{q(\bTheta; \boldsymbol \vartheta})\right\}. \label{eq:max_elbo}
\end{aligned}
\end{align}
\noindent Unfortunately, in RFLVMs, the optimization problem in Eq.~\eqref{eq:max_elbo} presents several key challenges: 1) \textbf{Lack of explicit PDFs:} 
    The indicator variables defined in DP are assumed to follow a CRP, a stochastic sampling procedure that lacks explicit PDFs \cite{blei2006variational}. As a result, the ELBO, defined in Eq.~\eqref{eq:max_elbo}, becomes incomputable. 
2) \textbf{Numerous variational variables:}   In VBI, one typically updates numerous variational variables jointly using a single optimization algorithm, unexpectedly overlooking the potential to exploit the unique optimization structures of the individual variables. For illustrating this point, we give an example below, which will be clearly demonstrated in the context of RFLVMs later in Section.~\ref{sec:srflvm}. 
    \begin{example}
        \label{eg:drawback_of_jointly_optimization}
        Let \(\bTheta_b\) represent a block of variational variables, which contains $J_b \neq 0$ variational variables, and $b$ denotes an integer.\footnote{We assume that there is no variable overlap between different \(\bTheta_b\) blocks.} Consider the following related terms of $\bTheta_1$ and $\bTheta_3$ within the ELBO:
        \begin{equation}
            \begin{aligned}
                & \mathbb{E}_{q(\bTheta_1; \boldsymbol \vartheta}) \left[ \log p(\Y \mid \bTheta_1) \right] - \mathrm{KL}(q(\bTheta_1 ;\boldsymbol \vartheta) \| p(\bTheta_1)) \\ 
                &~+ \mathbb{E}_{q(\bTheta_2, \bTheta_3 ;\boldsymbol \vartheta}) \left[ \log p(\bTheta_2 \mid \bTheta_3) \right] - \mathrm{KL}(q(\bTheta_3 ;\boldsymbol \vartheta) \| p(\bTheta_3)). \nonumber
            \end{aligned}
        \end{equation}
        Assume $q(\bTheta_3 ;\boldsymbol \vartheta)$ is a conjugate prior\footnote{A prior distribution is considered conjugate to the likelihood if the posterior distribution is in the same family of distributions as the prior.} while $q(\bTheta_1; \boldsymbol \vartheta)$ is non-conjugate. In such cases, jointly solving for \(\bTheta_1\) and \(\bTheta_3\) within the VBI framework becomes challenging. Specifically, MFVI is not suitable due to the non-conjugacy of $\bTheta_1$. For RGVI, the conditions are that each prior in the ELBO can either be sampled using the reparameterization trick or its related term can be analytically evaluated. However, even when RGVI is feasible, it still ignores the conjugacy structure within $q(\bTheta_3; \boldsymbol \vartheta)$, which may result in an excessive number of variational parameters (despite being analytically calculable) and slower convergence rates. 
    \end{example}
\noindent The subsequent section aims to address both of these issues.

%% file: content/vi.tex
\section{Methodology}
\label{sec:bcd-vi}

\begin{figure}[t]
    \centering
    \includegraphics[width=0.75\linewidth]{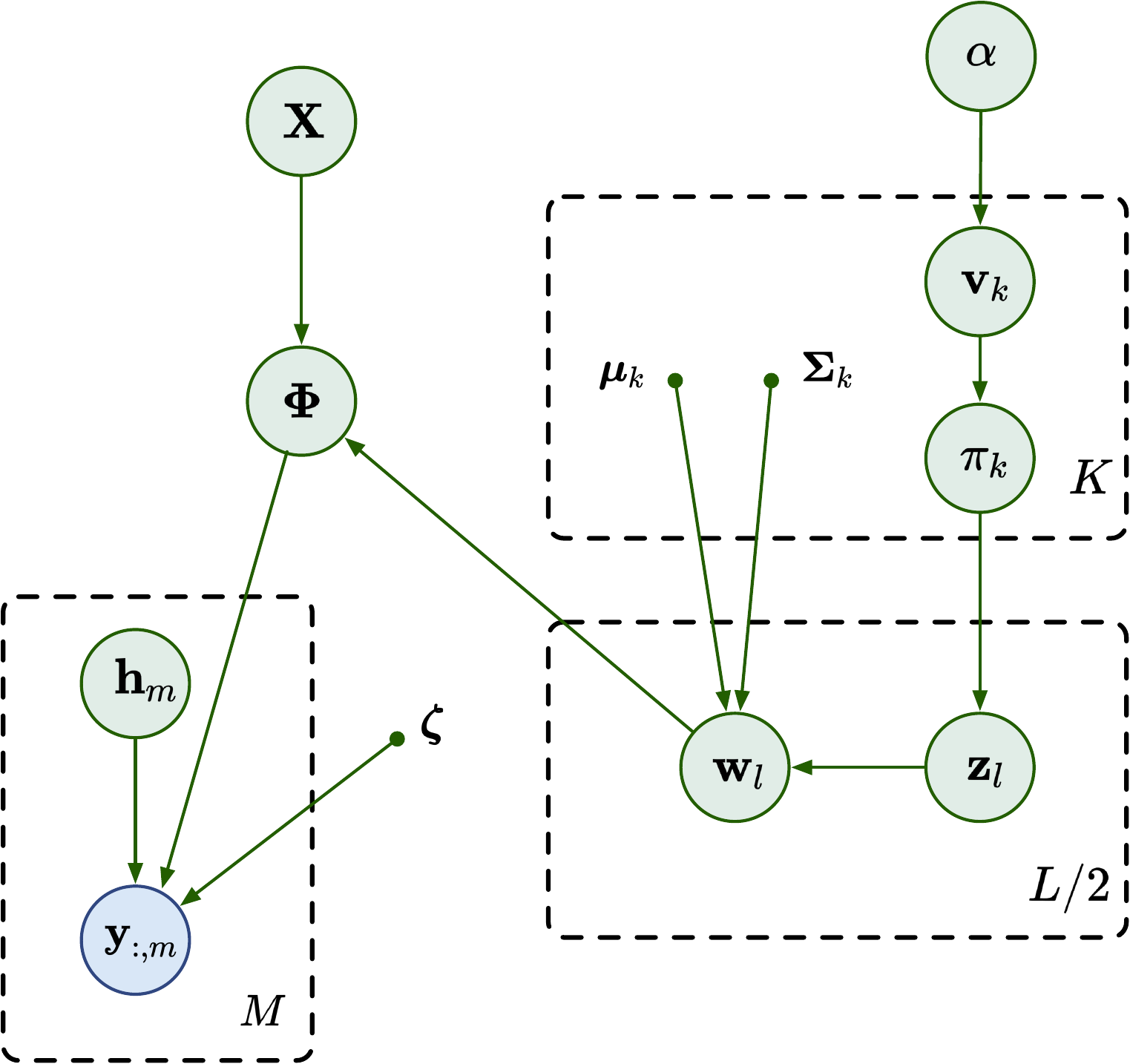}
    \caption{\textbf{Graphical Model of RFLVM with a stick-breaking construction}. We use arrows to denote the dependency relations between variables. The blue and green circles denote the observed and variational variables, respectively. Smaller dots indicate the deterministic parameters of the model. The nodes surrounded by a box represent that there are $K$, $M$ or $L/2$ nodes of this kind. 
    }
    \label{fig:graph_rflvm_stick}
\end{figure}

This section details two methods to address the challenges encountered when implementing VBI for RFLVMs. In Section \ref{subsec:Stick-breaking-DPM}, we introduce the utilization of the stick-breaking construction of DP to obtain explicit PDFs for indicator variables. Then in Section \ref{subsec:BCDVI}, we introduce the foundational principle of novel BCD-VI, which is motivated by RFLVMs.

\subsection{Stick-breaking Construction of DP}
\label{subsec:Stick-breaking-DPM}

As mentioned in Section~\ref{subsec:VIandChallenge-preliminary}, lacking explicit PDFs for $\z$ renders ELBO incomputable \cite{antoniak1974mixtures}.
This limitation does not affect MCMC sampling, as it only requires sampling from the posterior of $\z$, represented as a stochastic process that can be easily sampled (see Eq.~(10) in \cite{gundersen2021latent}). In contrast, in the context of VBI, an incomputable ELBO leads to an intractable problem (Eq.~\eqref{eq:max_elbo}). To tackle this issue, we utilize a well-known reformulation of the stick-breaking construction of a DP, also referred to as the Blackwell-MacQueen urn scheme \cite{blackwell1973ferguson}. Specifically, within the stick-breaking construction, each $\z_l$ is generated from a multinomial distribution with a probability mass function $\pi(\mathbf{v})$ that depends on {$\mathbf{v} {\in \mathbb{R}^{K}}$}, i.e., 
\vspace{-0.1in}
\begin{align}
    \begin{aligned}
        \z_l\!\sim\!\operatorname{Mult}(\pi(\mathbf{v}))
          = \operatorname{const} \times \left(\prod_{k=1}^{K}\left(1-\bv_{k}\right)^{\mathbf{1}\left[\z_l>k\right]}
        \bv_{k}^{\mathbf{1}\left[\z_l=k\right]}\right),  \nonumber
    \end{aligned}
\end{align}
where, the $k$-th element of the probability $\pi(\mathbf{v})$ is given by
\vspace{-0.1in}
\begin{subequations}
\begin{align}
      & \left[ \pi(\mathbf{v}) \right]_k  \triangleq \pi_{k}= \bv_{k} \prod_{j=1}^{k-1}\left(1-\bv_{j}\right), \\
      & \bv_{k} \sim \operatorname{Beta}(1, \alpha). 
\end{align}
\end{subequations}
Intuitively, each $\pi_k$ can be treated as sampling from a stick-breaking process. This process involves conceptualizing a stick of length $1$, where at each step $k$, a fraction $\bv_k$, sampled from a Beta distribution, is broken off. The length of the broken-off fraction is treated as the value of $\pi_k$, i.e., $\pi_k$$=\bv_k \prod_{j=1}^{k-1}\left(1-\bv_j\right)$. The remaining portion is retained for subsequent steps. Note that the distribution over $\pi_k$ is sometimes denoted as $\pi_k \!\sim\! \operatorname{GEM}(\alpha)$, with ``GEM" stands for \textbf{G}riffiths, \textbf{E}ngen, and \textbf{M}cCloskey \cite{blackwell1973ferguson}. Utilizing such stick-breaking construction within RFLVMs (see Eq.~\eqref{eq:gaussian_like} \& \eqref{eq:prior_x}), the graphical model for RFLVMs is depicted in Fig.~\ref{fig:graph_rflvm_stick}, and the joint distribution for the RFLVM, denoted as $p(\mathbf{Y}, \boldsymbol{\Theta})$, is expressed as:
\vspace{-0.1in}
\begin{align}
    \begin{aligned}
        p(\mathbf{Y}, \boldsymbol{\Theta}) = & \prod_{m=1}^{M} \prod_{n=1}^{N} p(\mathbf{y}_{n,m} \mid  \mathbf{x}_n, \mathbf{W}) p(\mathbf{x}_n)  \\ 
        & \prod_{l=1}^{L/2} p(\mathbf{w}_l \mid \mathbf{z}_l) p(\mathbf{z}_l \mid \mathbf{v}) p(\mathbf{v} \mid \alpha) p(\alpha), 
    \end{aligned} 
    \label{eq:joint_dis}
\end{align}
where $\boldsymbol{\Theta} \triangleq \{ \mathbf{W}, \mathbf{X}, \mathbf{z}, \mathbf{v}, \alpha \}$. The variational variables within $\boldsymbol{\Theta}$ each have explicit PDFs, thereby resulting in a computable ELBO. 

\begin{algorithm}[t]
\caption{BCD-VI}
\begin{algorithmic}[1]
\WHILE{Not Converged}
    \STATE Divide variational parameters into $B$ blocks, $\{ \bTheta_i \}_{i=1}^{J_1}, $ $ \dots,   \{ \bTheta_i \}_{i=1}^{J_B}$; 
    \FOR{$b \in 1, ..., B$}
        \IF{condition in Corollary \ref{coro: exp} \& $J_b=1$}
        \STATE Identifying $\boldsymbol{\vartheta}^{\star}_b$ from 
        \STATE $\log q^{\star}(\bTheta_b) = \langle\ln p(\Y, \bTheta \mid \text{rest})\rangle_{\bTheta \backslash \bTheta_b}+\text { const }$; 
        \ELSE
        \STATE Maximize the MC estimator of the $b$-th optimization objective (defined in Eq.~\eqref{eq:rgvi-sub}) and update $\boldsymbol{\vartheta}_b$ using the Adam \cite{kingma2014adam}. 
        \ENDIF
    \ENDFOR
\ENDWHILE
 \end{algorithmic}
\label{alg:bcd-vi}
\end{algorithm}

\subsection{Block Coordinate Descent Variational Inference}
\label{subsec:BCDVI}

As discussed in Section~\ref{subsec:VIandChallenge-preliminary}, each variational variable in the VBI framework is typically solved jointly using the same optimization algorithm (e.g., MFVI or RGVI). This approach overlooks the potential benefits of exploiting the specific optimization structures of different variational variables as illustrated in Example \ref{eg:drawback_of_jointly_optimization}. 
A potential solution is alternating optimization, which enables the use of the most appropriate solver for each variational variable, thus not only improving computational efficiency for conjugate terms but also ensuring tractable inference for non-conjugate terms. 

Inspired by this, we incorporate the concept of block coordinate descent (BCD) into VBI using a divide and conquer strategy. Specifically, we first partition the variational variables $\bTheta$ into $B$ blocks, denoted as $\{\bTheta_b\}_{b=1}^{B}$, with each block $\bTheta_b$ containing $J_b \neq 0$ variational variables. Such block partitioning then divides the optimization problem defined in Eq. \eqref{eq:max_elbo} into $B$ subproblems, that is 
\vspace{-0.05in}
\begin{align}
    & \max_{\boldsymbol{\vartheta}_b} \mathcal{L}\left(q(\bTheta_b; \vartheta)\right) , \quad b = 1, ..., B,  
    \label{eq:b-th-prob} 
\end{align}
where $\boldsymbol{\vartheta}_b$ denotes the variational parameters associated with the block $b$. The optimization objective, termed the $b$-th local ELBO, includes the terms within the ELBO given by Eq. \eqref{eq:max_elbo} that are related to block $b$. 
To solve these subproblems, we propose using the MFVI-based method for those with conjugate priors, and RGVI for those with non-conjugate priors, as illustrated in the remaining text of this section.

Let us first focus on MFVI-based methods, which aim to identify analytical optimal solutions for specific variable blocks. Particularly, MFVI-based methods rely on the following theorem:
\begin{theorem}
    Assuming that the variational variable within variable block $b$ is independent of other variational variables,\footnote{Note that we do not assume full factorization here. We only require that a variable block be independent from others, while the variables of each block can be interdependent.} i.e., 
    \begin{equation}
    q(\bTheta; \vartheta) = q(\bTheta_b;\vartheta)q(\bTheta_{\neq b}; \vartheta),
    \end{equation}
    where $q(\bTheta_{\neq b} ;\vartheta)$ represents the joint variational distribution for the remaining variational variables, we have the optimal solution $q^{\star}(\bTheta_b; \vartheta)$ that can be formulated as:
    \begin{align}
    \log q^{\star}(\bTheta_b;\vartheta) = \langle\ln p(\Y, \bTheta)\rangle_{\bTheta \backslash \bTheta_b}+\text { const },
    \label{eq:mf}
    \end{align}
    where $\langle \cdot \rangle_{\bTheta \backslash \bTheta_b}$ symbolizes the expectation of $\bTheta$ excluding $\bTheta_b$, and $\text{const}$ signifies a constant independent of the current $\bTheta_b$. Proof can be found in Appendix \ref{app:mean_field}. 
    \label{theo:mean-field}
\end{theorem}
\noindent As stated in Theorem \ref{theo:mean-field}, the independence assumption facilitates the acquisition of the optimal solution $q^{\star}(\bTheta_b)$. However, achieving this optimal solution still requires certain conditions. In this paper, we focus on the conjugacy structures, as detailed as follows.
\begin{corollary}
    We assume that $\bTheta_{b}$ follows a distribution within the exponential family. When this prior distribution is conjugate to the likelihood, the resulting posterior distribution also belongs to the exponential family distribution. Consequently, the optimal variational parameters $\boldsymbol \vartheta_b^{\star}$ for $q^{\star}(\bTheta_b)$ can be determined analytically. Proof can be found in Appendix \ref{app:close-form}. 
    \label{coro: exp}
\end{corollary}
\noindent Corollary \ref{coro: exp} outlines that 
in practice, when the conjugate structures are fulfilled, the optimal variational parameters $\boldsymbol \vartheta_{b}$ within block $b$ can be directly identified from Eq. \eqref{eq:mf} without further calculations (see the example in Section~\ref{subsec:dp_module}). However, the conjugacy structures specified in Corollary \ref{coro: exp} can be restrictive. For those non-conjugate cases, we use an RGVI-based method.

Specifically, in such case, the optimization objective can be approximated using an MC estimator $\hat{\mathcal{L}}$, obtained by sampling from the variational distributions using the reparameterization trick \cite{kingma2013auto}. This leads to the following optimization problem:
\begin{equation}
\setlength{\abovedisplayskip}{4.5pt}
\setlength{\belowdisplayskip}{4.5pt}
    \max_{\boldsymbol \vartheta_{b}}  \hat{\mathcal{L}}(\boldsymbol \vartheta_{b}).
     \label{eq:rgvi-sub}
\end{equation}
Since MC samples obtained with the reparameterization trick are differentiable with respect to $\boldsymbol{\vartheta}_b$, the maximization can be efficiently addressed using gradient descent-based methods, e.g., Adam \cite{kingma2014adam}.

Combining the RGVI-based approach with the previously discussed MFVI technique, we develop the BCD-VI algorithm (detailed in the next section and Algorithm \ref{alg:bcd-vi}). Our algorithm leverages the strengths of both methods, making it applicable to the considered RFLVM with various variational variables. 
Moreover, the proposed BCD-VI not only guarantees convergence, as summarized in the following remark, but also offers opportunities for further enhancements thanks to its ability to seamlessly integrate advanced VBI algorithms for subproblem solving.
\begin{remark}
    In the BCD-VI framework, the free variational parameters within each block are updated using closed-form solutions or iterative methods. Consequently, each update within BCD-VI can be regarded as a BCD \cite{ida2023fast, xu2015block} step, implying that the limit point generated by the BCD-VI algorithm is, at the very least, a stationary point of the KL divergence defined in Eq.~\eqref{eq:KL_VI}.
\end{remark}

%% file: content/srflvm.tex
\section{Scalable RFLVM (SRFLVM)}
\label{sec:srflvm}

In this section, we detail our proposed computationally efficient BCD-VI algorithm for RFLVMs, which we term SRFLVM (scalable RFLVM). 
Following the philosophy mentioned in Section~\ref{subsec:BCDVI}, we first preselect the following block-wise variational distribution, \( q(\bTheta) \), 
\begin{equation}
    q(\bTheta) = q(\mathbf{X}) p(\mathbf{W} \vert \mathbf{z}) q(\mathbf{z}) q(\mathbf{v}) q(\alpha), 
    \label{eq:gaussian_variational}
\end{equation}
 where 
$q(\mathbf{X})=\prod_{n=1}^N\mathcal{N}\left(\mathbf{x}_n\vert\boldsymbol{\mu}_n,\!\boldsymbol{S}_n\right)\!$ with $\boldsymbol{\mu}_n\!\!\in\!\mathbb{R}^{Q},$  $\!\mathbf{S}_n\!\!\in\!\mathbb{R}^{Q\!\times\!Q}$, and $\!q(\mathbf{v})\!\!=\!\!\prod_{k=1}^{K}\!\operatorname{Beta}(a_{\mathbf{v}_k},\!b_{\mathbf{v}_k})$ with $a_{\mathbf{v}_k},\!b_{\mathbf{v}_k}\!\!\in\!\mathbb{R}$. The distribution $q(\alpha)$ is a Gamma distribution with shape and rate parameters $a_{\alpha}$ and $b_{\alpha}$, respectively. Furthermore, we define the probability of the \( l \)-th spectral point \(\mathbf{w}_l\) being sampled from the \( k \)-th cluster as \(\phi_{lk}\), i.e., \( q(\mathbf{z}_{l} = k) = \phi_{lk} \). 
In addition, we set  $q(\W \vert \z)$ to its prior $p(\W \vert \z)$ to capture the dependence between $\W$ and $\z$, and achieve a more tractable ELBO (see Eq.~\eqref{eq:elbo_gaussian_first_term}).    

By combining the variational distributions $q(\bTheta)$ (Eq.~\eqref{eq:gaussian_variational}) with the model joint distribution (Eq.~\eqref{eq:joint_dis}), we immediately obtain the following concrete optimization problem for SRFLVMs as described in  Eq.~\eqref{eq:max_elbo}:
\begin{equation}
\begin{aligned}
    & \max_{\boldsymbol{\vartheta}, \boldsymbol{\theta}} \ \mathbb{E}_{q(\mathbf{X}) p(\mathbf{W} \mid \mathbf{z}) q(\mathbf{z})} \left[ \log \frac{ p(\Y \mid \X, \W) p(\X) p(\W \mid \mathbf{z})
        }{q(\X) q(\W \mid \mathbf{z})} \right] \\
    & \quad \quad \quad + \mathbb{E}_{q(\mathbf{z}) q(\mathbf{v}) q(\alpha)} \left[ \log \frac{p(\mathbf{z} \mid \mathbf{v}) p(\mathbf{v} \mid \alpha) p(\alpha)}{ q(\mathbf{z}) q(\mathbf{v}) q(\alpha)} \right],  
\end{aligned}
\label{eq:elbo_gaussian}  
\end{equation}
where {the} model hyperparameters, $\btheta$, encompass $\btheta_{\text{dpm}}$ from {the} kernel function learning and likelihood parameters $\boldsymbol{\zeta}$, and {the} variational parameters $\boldsymbol{\vartheta}$ encompass $\{ \{\boldsymbol{\mu}_n, \boldsymbol{S}_n\}_{n=1}^N, \{a_{\mathbf{v}_k}, b_{\mathbf{v}_k}\}_{k=1}^K, a_{\alpha}, b_{\alpha}, \{\phi_{lk}\}_{l=1, k = 1}^{L/2, K}\}$. Given the fact that we set $q(\W \mid \z) = p(\W \mid \z)$ for tractability and dependency, the first term (namely likelihood term) in this optimization problem can be expressed as 
\begin{equation}
\mathbb{E}_{q(\mathbf{X}) p(\mathbf{W} \vert \mathbf{z}) q(\mathbf{z})}\!\left[ \log \frac{ p(\mathbf{Y} \vert \mathbf{X}, \mathbf{W})  p(\mathbf{X}) }{q(\mathbf{X}) }\!\right]. 
\label{eq:elbo_gaussian_first_term}
\end{equation} 
We now turn to the implementation details of BCD-VI for solving the aforementioned optimization problem. To this ends, careful analysis of the optimization objective structure is essential for effective variable block optimization. Specifically, for evaluating the likelihood term, it becomes manageable once \( \W, \X \) and \( \z \) are divided into distinct blocks. As for \( \alpha \) and \( \bv \), which can have conjugate priors, they are each treated as separate blocks. 
In summary, we partition the variational variables into four blocks: 
\begin{equation}
    \color{black} \underbracket{ \color{black} \{ \W, \X \},}_{\text{likelihood block}} \color{black} \underbracket{ \color{black} \{ \z \}, \{ \bv \}, \{ \alpha \} }_{\text{DP kernel blocks}}, \color{black} 
\end{equation} 
where the set $\{ \W, \X \}$ is named the ``likelihood block'' simply because the data likelihood is dependent on $\W$ and $\X$. For the logistic likelihood, the main difference is that the likelihood block includes additional stochastic weights \( \bA \). Consequently, next, we will discuss the inference processes of Gaussian and logistic likelihood blocks in Section \ref{subsec:Gaussian} and \ref{subsec:Logistic}, respectively. Finally, Section \ref{subsec:dp_module} will introduce the inference process of the shared DP kernel blocks. 


\subsection{Gaussian Likelihood Blocks}
\label{subsec:Gaussian}

In the case of Gaussian likelihood, the related subproblem is, 
\begin{equation}
\max_{\boldsymbol{\vartheta}, \boldsymbol{\theta}} \ \mathbb{E}_{q(\mathbf{X}) p(\mathbf{W} \vert \mathbf{z}) q(\mathbf{z})}\!\left[ \log \frac{ p(\mathbf{Y} \vert \mathbf{X}, \mathbf{W}) p(\mathbf{X}) }{q(\mathbf{X})}\!\right], 
\label{eq:gaussian-elbo-z}
\end{equation}
where \(\boldsymbol{\vartheta} \!=\! \{ \boldsymbol{\mu}_n, \mathbf{S}_n \}_{n=1}^{N}\). Through some algebraic calculations, the subproblem (see Eq.~\eqref{eq:gaussian-elbo-z}) can be reformulated as follows:
\begin{equation}
    \underbracket{\mathbb{E}_{q(\W) q(\X)} \left[ \log p(\Y \vert \W, \X) \right]}_{\text{term (a)}} - \underbracket{\operatorname{KL}(q(\X) \| p(\X) )}_{\text{term (b)}} \label{eq:gaussian-elbo-without-z}, 
\end{equation}
where 
\begin{equation}
   \!\! q(\mathbf{W}) \!=\! \int \! p(\mathbf{W} \vert \z)q(\z) d \z \!=\! \prod_{l=1}^{L/2} \mathcal{N}( \sum_{k=1}^{K} \phi_{lk} \bmu_{k},  \sum_{k=1}^{K} \phi_{lk} \bSigma_{k}).  
    \label{eq:variational_w}
\end{equation}
The interpretability of Eq.~\eqref{eq:gaussian-elbo-without-z} is as follows: Term (a) represents the data reconstruction error, encouraging accurate reconstruction of the observed data using any $\mathbf{X}$ and $\mathbf{W}$ samples drawn from their variational distributions. Term (b) serves as a regularizer, penalizing substantial deviations of $q(\mathbf{X})$ from the prior distribution $p(\mathbf{X})$.

Due to the highly nonlinear nature of the GP prior, the prior for $\X$ and $\W$ and their likelihoods are non-conjugate. Thus,  we adopt the RGVI-based approach to solve the subproblem defined in Eq.~\eqref{eq:gaussian-elbo-without-z}.
We numerically evaluate the objective function Eq.~\eqref{eq:gaussian-elbo-without-z}, which is 
\vspace{-0.05in}
\begin{equation}
    \sum_{m=1}^M \frac{1}{I} \sum_{i=1}^I \log \mathcal{N}\left(\mathbf{y}_{:, m} \vert \mathbf{0}, (\bPhi \bPhi^{\top})^{(i)} \!+\! \sigma^2 \mathbf{I}_N\right)    \!-\! \operatorname{KL}(q(\mathbf{X}) \| p(\mathbf{X})), 
    \label{eq:gaussian_elbo_estimate}
\end{equation}
where $I$ denotes the number of MC samples drawn from $q(\mathbf{X})$ and $p(\mathbf{W})$. The second term, a KL divergence between two Gaussian distributions, can be computed analytically. Further computational details are provided in Appendix \ref{app:Gaussian_elbo}. Thanks to the reparameterization trick, we can apply gradient-based optimization methods, e.g., Adam \cite{kingma2014adam}, to update the model hyperparameters $\boldsymbol{\theta}$ and the variational parameters $\boldsymbol{\vartheta}$.

\begin{algorithm}[t]
\caption{Scalable RFLVMs (SRFLVMs)}
\begin{algorithmic}[1]
\REQUIRE $\boldsymbol{\theta}$
\WHILE{Not Converged}
    \WHILE{Not Converged}
        \STATE Sampling $q(\X)$ from $\prod_{n=1}^N\mathcal{N}\left(\mathbf{x}_n\!\mid\!\boldsymbol{\mu}_n,\boldsymbol{S}_n\right)$ using reparametrization trick; 
        \STATE Sampling $q(\W)$ from Eq.~\eqref{eq:variational_w} using reparametrization trick;
        \IF{Gaussian Likelihood}
            \STATE Evaluate the local ELBO with \eqref{eq:gaussian_elbo_estimate}, using the sampled $\X, \W$; 
        \ELSE
            \STATE Sample $\bA$, $\boldsymbol{\Omega}$ from posteriors in Theorem~\ref{theo:opt_a}. 
            \STATE  Evaluate the local ELBO with \eqref{eq:logistic_estima} using the sampled $\X, \W, \bA, \boldsymbol{\Omega}$;
        \ENDIF
        \STATE Update variational parameters $\boldsymbol{\vartheta}$ and model hyperparameters $\boldsymbol{\theta}$ using the Adam optimizer \cite{kingma2014adam}; 
    \ENDWHILE
    \WHILE{Not Converged}
        \STATE Evaluate the local ELBO (\eqref{eq:gaussian_elbo_estimate} for Gaussian and \eqref{eq:logistic_estima} for logistic) with sampled variational variables. 
        \STATE Update variational parameters $\boldsymbol{\vartheta}$ using the Adam optimizer \cite{kingma2014adam};
    \ENDWHILE
    \FOR{$k \in 1, ..., K$} 
            \STATE $q(v_k) \sim \operatorname{Beta}(1 + \sum_{l=1}^{L/2} \phi_{l, k}, \alpha +  \sum_{l=1}^{L/2} \sum_{j=k+1}^{K}  \phi_{l, j}) $; 
    \ENDFOR
    \STATE $q(\alpha) = \operatorname{Ga}(\alpha_0, \beta_0 - \sum_{k=1}^{K}  \Psi\left(b_{v_k}\right)-\Psi\left(a_{v_k}+b_{v_k}\right))$. 
\ENDWHILE
\ENSURE $\hat{\X}$ 
\end{algorithmic}
\label{alg:1}
\end{algorithm}

\subsection{Logistic Likelihood Block}
\label{subsec:Logistic}

The distinction in logistic likelihood from the Gaussian case lies in the introduction of stochastic weights \( \bA \) (see Eq.~\eqref{eq:RFF_approximation}), resulting in a likelihood term formed as \(p(\Y \vert \X, \W, \bA )\) and a manageable KL divergence term involving $\bA$. Assuming \( q(\ba_m) \) for all \( m \in \{1, ..., M\} \) as independent multivariate Gaussian variables and incorporating $\bA$ into the likelihood block, which now includes $\{ \W, \X, \bA \}$, the corresponding subproblem can be solved using the RGVI-based method.

Unfortunately, the empirical results reveal that using approximate inference on the stochastic weights $\bA$ significantly degrades the quality of learned kernel matrices, regardless of likelihood types (see Appendix.~\ref{app:motivation_for_PG}). To address this issue, we integrate Pólya-gamma (PG) augmentation \cite{polson2013bayesian} into SRFLVMs. By augmenting a set of PG distributed variables into the model, we can  transform the logistic likelihood into a quadratic form with respect to \( \bA \). This transformation ultimately yields a closed-form solution for the posterior of $\bA$, thereby improving both tractability and performance. 

\subsubsection{Pólya-gamma Distribution}

If $\omega \in \mathbb{R}$ follows a PG distribution (PGD) with parameters $b > 0$ and $c \in \mathbb{R}$, denoted as $\omega \sim \mathrm{PG}(b, c)$, then it is equivalent in distribution to an infinite weighted sum of Gamma:
\begin{equation}
    \omega \stackrel{d}{=} \frac{1}{2 \pi^2} \sum_{k=1}^{\infty} \frac{g_k}{(k-1 / 2)^2+c^2 /\left(4 \pi^2\right)},
    \label{eq:pg_cdf}
\end{equation}
where $\stackrel{d}{=}$ denotes equality in the cumulative distribution function (CDF), and $g_k \sim \operatorname{Ga}(b, 1), k \in 1, ..., \infty$ are independent Gamma random variables. Building on the equality established in Eq. \eqref{eq:pg_cdf}, several useful properties are proposed in \cite{polson2013bayesian}. 
First,
\begin{equation}
    \frac{\left(e^\psi\right)^a}{\left(1+e^\psi\right)^b}=2^{-b} e^{\kappa \psi} \int_0^{\infty} e^{-\omega \psi^2 / 2} p(\omega) \mathrm{d} \omega,
    \label{eq:pg}
\end{equation}
where $\kappa=a-b / 2$ and $p(\omega)=\operatorname{PG}(\omega \mid b, 0)$. And second,
\begin{equation}
    p(\omega \mid \psi) \sim \operatorname{PG}(b, \psi).
    \label{eq:pos_omega}
\end{equation} 
Given \(\omega\), the left-hand side  of Eq.~\eqref{eq:pg}, which resembles the logistic likelihood expression (Eq.~\eqref{eq:logistic_likelihood}), can be represented in a quadratic (or Gaussian) form with respect to \(\psi\). The corresponding posterior distribution of \(\omega\) is provided by Eq.~\eqref{eq:pos_omega}.

\subsubsection{PG Augmentation for Logistic Likelihood}
It is obvious that the established equivalence between the logistic likelihood and a Gaussian form in PGD potentially enhances tractability in the logistic case. To leverage this property, we first introduce a set of random variables \(\bm{\Omega} \triangleq \{\boldsymbol{\omega}_{n,m}\}\), where \(n \in \{1, \ldots, N\}\) and \(m \in \{1, \ldots, M\}\). Here, each variable follows the distribution \(\operatorname{PG}(b_{nm}, c_{nm})\). Then, we define \(\psi_{nm} = \bPhi_{n,:} \ba_{m}\), which will naturally distribute with Gaussian given \(\bm \Omega\) as previously mentioned, allowing us to analytically derive the true posterior of \(\bA\), as stated in the following theorem. 
\begin{theorem} \label{theo:opt_a}
 By leveraging the property of the PGD (given by Eq. \eqref{eq:pg}), in conjunction with the logistic likelihood term from Eq.~\eqref{eq:logistic_likelihood}, we can derive the optimal posterior distribution for \( \bA \) as:
  \begin{equation}
        p(\ba_m \mid \bPhi, \boldsymbol{\Omega}) = \mathcal{N} \left( \mathbf{m}_{m}, \mathbf{V}_{m}  \right), \quad  m \in 1, ..., M, 
  \end{equation}
    where 
    \begin{align*}
            & \boldsymbol{\Omega}_m\! =\!\operatorname{diag}\left(\left[\boldsymbol{\omega}_{1 m}, \ldots, \boldsymbol{\omega}_{N m}\right]\right)   & & \kappa_{nm} = a_{nm} - b_{nm} / 2    \\
            &\mathbf{V}_{m}=\left(\bPhi^{\top} \boldsymbol{\Omega}_m \bPhi+\mathbf{B}_0^{-1}\right)^{-1} & &\boldsymbol{\kappa}_m  = [ \kappa_{1m}, ..., \kappa_{nm} ]^{\top}  \\ 
            & \mathbf{m}_{m}=\mathbf{V}_{m}\left(\bPhi^{\top} \boldsymbol{\kappa}_m+\mathbf{S}_0^{-1} \mathbf{m}_0\right) 
    \end{align*}
    and the optimal posterior distribution for $\boldsymbol \omega_{nm}$ is given by 
    \begin{equation}
    p( \boldsymbol \omega_{nm} \mid \ba_m) = \operatorname{PG}\left(b_{nm}, \psi_{nm} \right).  
    \end{equation}
\end{theorem} 
\begin{proof}
    The proof can be found in Appendix \ref{app:theo_3}. 
\end{proof}
\noindent Based on the result given by Theorem \ref{theo:opt_a}, we set the predefined family for $\bA, \boldsymbol{\Omega}$ to their true posteriors, that is, 
\begin{equation}
    q(\bA, \boldsymbol \Omega) = p(\bA \mid \bPhi, \boldsymbol{\Omega}) p(\boldsymbol{\Omega} \mid \bA),  
\end{equation}
in order to enhance the inference accuracy of the VBI. Moreover, during the iteration of optimization, these two posterior distributions will also evolve with the iterations, because they are dependent on other variational variables. For the remaining variational variables, the settings of predefined family are same as to those used in the Gaussian case (refer to Eq.~\eqref{eq:gaussian_variational}). 

In the logistic likelihood block, we need fine-tuning the free variational parameters \(\boldsymbol{\vartheta} = \{ \boldsymbol{\mu}_n, \mathbf{S}_n \}_{n=1}^{N}\) and model hyperparameters \(\boldsymbol{\theta}\), which is achieved by optimizing the following subproblem,  
\begin{equation}
\begin{aligned}
    \max_{\boldsymbol{\vartheta}, \boldsymbol{\theta}} \ & \underbracket{ \mathbb{E}_{q(\X, \bA, \mathbf{W} )} \log p(\Y \vert \X, \bA, \mathbf{W})}_{\text{term (a)}} \!-\! \underbracket{ \operatorname{KL}(q(\X) \| p(\X))}_{\text{term (b)}}  \\ 
    & - \underbracket{ \mathbb{E}_{q(\mathbf{W}, \mathbf{X}, \boldsymbol{\Omega})}\left[ \operatorname{KL}(p(\bA \vert \bm{\Phi}, \boldsymbol{\Omega}) \| p(\bA))  \right]}_{\text{term (c)}}  \\
    &-\underbracket{\mathbb{E}_{q(\bA)} \left[ \operatorname{KL} \left( p(\boldsymbol{\Omega} \mid \bA) \|  q(\boldsymbol{\Omega}) \right) \right]}_{\text{term (d)}}.  
\end{aligned} 
\label{eq:log_ELBO}
\end{equation}
The term (a) represents the data reconstruction error, which encourages accurate reconstruction of the observed data $\Y$
using any sampled variational variables $\X, \W$, and $\bA$ from their respective variational distributions. The subsequent terms serve as regularization for the variational distributions, penalizing significant divergence from their corresponding prior distributions.

Similar to the Gaussian case, instead of using the intractable MFVI-based approach, we evaluate the local ELBO, or the objective function (see Eq.~\eqref{eq:log_ELBO}), by following the principle of the RGVI-based method. Specifically, we analytically evaluate term (b) given the Gaussian nature of the underlying distributions. For term (d), we omit the intractable KL divergence, which appears to have minimal impact on the latent variable inference, as evidenced by the comparable performance with RFLVMs demonstrated in Section~\ref{subsec:text_image}. The terms (a) and (c) are numerically evaluated using MC estimation, as follows:
\begin{align}
    \text{(a)}&\approx 
    \sum_{m=1}^{M} \sum_{n=1}^N \Bigg\{ \log(c\left(y_{n m}\right))  
    + \sum_{i=1}^{I} \Big[    a(y_{nm}) \bPhi_{n, :}^{(i)}~\ba_m^{(i)} ] \nonumber \\ & - b(y_{nm}) \log \left( 1 + \exp \left( \bPhi_{n, :}^{(i)}~ \ba_m^{(i)} \right) \right) \Big] \Bigg\} \label{eq:logistic_estima}  \\ 
    \text{(c)}&\approx \frac{1}{2} \sum_{m=1}^M \sum_{i=1}^{I} \left[ -\log |\mathbf{V}_{m}^{(i)}| + \operatorname{Tr}(\mathbf{V}_{m}^{(i)}) + (\mathbf{m}_{m}^{(i)})^{\top} \mathbf{m}_{m}^{(i)} \right]. \nonumber
\end{align}
where $I$ represents the number of MC samples drawn from the variational distribution of $\W, \X, \bA $, and $\boldsymbol{\Omega}$. Thanks to the reparameterization trick, we can apply gradient-based optimization methods to update the model hyperparameters $\boldsymbol{\theta}$ and the variational parameters $\boldsymbol{\vartheta}$. 


\subsection{DP Mixture Kernel Blocks}
\label{subsec:dp_module}

\subsubsection{\textbf{Estimation of \texorpdfstring{$\z$}{z}}}

For the block $\{ \z \}$ under the Gaussian likelihood, we need to fine-tune the free variational parameters $\bm \vartheta = \{ \phi_{l, k} \}_{k=1, l=1}^{K, L/2}$ by solving the following subproblem, 
\begin{equation}
\max_{\boldsymbol{\vartheta}} \ \mathbb{E}_{q(\X, \W, \mathbf{z}, \bv) }\!\left[  \log p(\Y \vert \X,\!\W)\right]\!-\!\operatorname{KL}\!\left[ q(\z) \| p(\z \vert \bv) \right]. 
\label{eq:local_elbo_z}
\end{equation}
To address this optimization problem, we first identified that block $\{ \mathbf{z} \}$ does not satisfy the fully independent assumption specified in Theorem \ref{theo:mean-field}, leading to an intractable optimal solution. Therefore, we evaluate the local ELBO by following the principle of the RGVI-based approach. Specifically, the first term is estimated by a MC estimator as demonstrated in Eq.~\eqref{eq:gaussian_elbo_estimate}, while the second term is analytically computed as follows:
\begin{equation}
\label{Eq:opt_z}
\begin{aligned}
    &\sum_{l=1}^{L/2} \Big\{ \sum_{k=1}^{K} \phi_{l, k}^2 \big\{ \mathbb{E}_{q(\bv_k)} \left[ \log \bv_k \right] + \sum_{j=1}^{k} \mathbb{E}_{q(\bv_j)} \left[ \log( 1 - \bv_j ) \right] \big\} \\ 
    &\quad \quad - \phi_{l, k} \log \phi_{l, k} \Big\}. 
\end{aligned}
\end{equation}
Here, the expectations are given by:
\begin{align}
    \mathbb{E}_{q(\bv_j)}\left[  \log(1-\bv_j) \right] &=  \Psi\left(b_{\bv_j}\right)-\Psi\left(a_{\bv_j}+b_{\bv_j}\right),
    \\ 
    \mathbb{E}_{q(\bv_k)} \left[ \log \bv_k \right] &= \Psi\left(a_{\bv_k}\right)-\Psi\left(a_{\bv_k}+b_{\bv_k}\right), 
\end{align}
where $\Psi$ denotes the digamma function. Due to the differentiability of samples using the reparameterization trick w.r.t. variational parameters, we can update the variational parameters $\boldsymbol{\vartheta}$ using gradient descent-based methods \cite{kingma2014adam}. For the logistic likelihood, the primary distinction lies in the first term of Eq.~\eqref{eq:local_elbo_z} (or the likelihood term), which is evaluated in Eq.~\eqref{eq:logistic_estima}. The following subsections will examine the inference processes for the blocks $\{ \bv \}$ and $\{ \alpha \}$, which not only satisfy the fully independent assumption but also exhibit a conjugate structure between the prior and the conditional likelihood \cite{bishop2006pattern}, leading to a feasible optimal solution.

\subsubsection{\textbf{Estimation of \texorpdfstring{$\bv$}{v}}}
According to Theorem \ref{theo:mean-field}, the optimal solution for $\log q(\mathbf{v})$ is expressed as:
\begin{equation}
    \begin{aligned}
         & \mathbb{E}_{q(\mathbf{z})} \Big[ \log \Big( \prod_{k=1}^{K} \big(1-\bv_{k}\big)^{\mathbf{1}\left[\z_l > k\right]} \bv_{k}^{\mathbf{1}\left[\z_l = k\right]} \Big) \Big] + \text{const} \label{eq:v} \\
        & \quad + \mathbb{E}_{q(\alpha)} \Big[ \log \prod_{k=1}^{K} \frac{1}{B(1, \alpha)} (1 - \bv_k)^{\alpha-1} \Big],
    \end{aligned}
\end{equation}
where $B(1, \alpha)$ represents the Beta function. Upon examining Eq.~\eqref{eq:v}, each $q(\bv_k)$ follows a Beta distribution, $\operatorname{Beta}(a_{\bv_k}, b_{\bv_k})$, with parameters defined as:
\begin{equation}
     a_{\bv_k} = 1 + \sum_{l=1}^{L/2} \phi_{l, k}, \quad
     b_{\bv_k} = \alpha +  \sum_{l=1}^{L/2} \sum_{j=k+1}^{K}  \phi_{l, j}
        \label{update-v}, 
\end{equation}
where $\alpha$ is experimentally set to be the expectation of $\alpha$, w.r.t. its variational distribution.

\subsubsection{\textbf{Estimation of \texorpdfstring{$\alpha$}{alpha}}}

According to Theorem \ref{theo:mean-field}, the optimal solution for $\log q(\alpha)$ is given by 
\begin{align}
        \Big( \sum_{k=1}^{K} \mathbb{E}_{q(\bv_k)} \left[  \log (1-\bv_k) \right] - \beta_0 \Big) \alpha
       + ( \alpha_0 - 1 ) \log(\alpha) + \text{const.}  \nonumber
\end{align}
It can be observed that $q(\alpha)$ follows a Gamma distribution, $\operatorname{Ga}(a_{\alpha}, b_{\alpha})$, with parameters given by:
\begin{align}
       &  a_{\alpha} = \alpha_0, 
       & b_{\alpha} = \beta_0 - \sum_{k=1}^{K} \mathbb{E}_{\bv_k} \left[ \log(1-\bv_k) \right], 
       \label{upate-alpha}
\end{align}
where $\mathbb{E}_{\bv_k}\left[  \log(1-\bv_k) \right] =  \Psi\left(b_{\bv_k}\right)-\Psi\left(a_{\bv_k}+b_{\bv_k}\right)$.

%% file: content/simulation.tex
\input{table/method_summary}

\begin{figure*}[t!]
    \centering
    \subfloat[Manifold and Kernel Matrix Learning]{\includegraphics[width=0.49\linewidth]{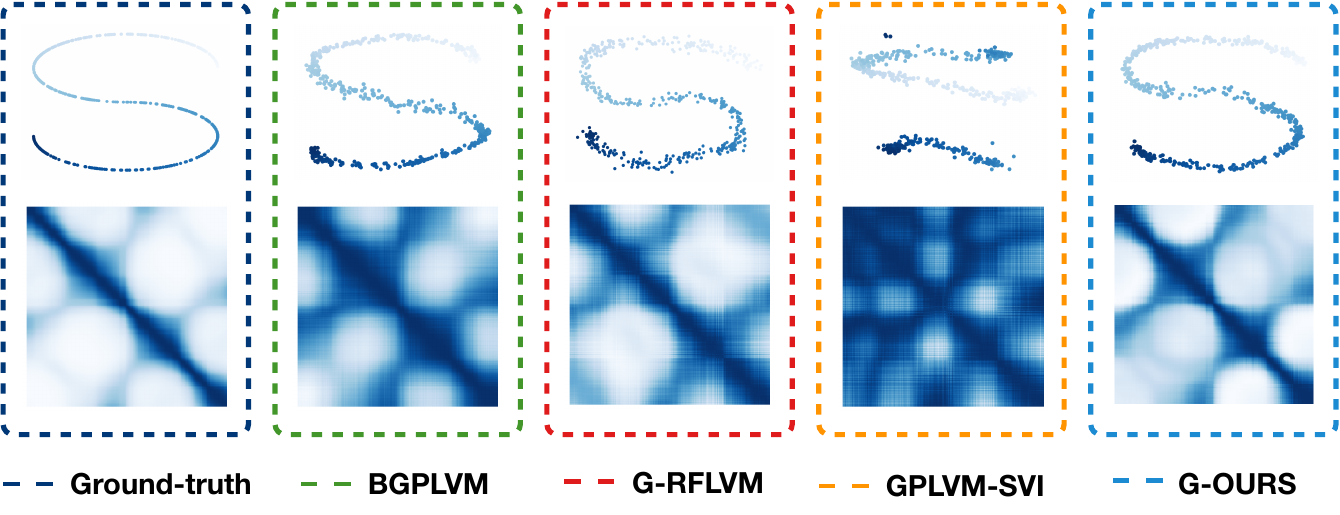}
    \label{subfig:toy_example_left}
    }
    \hspace{0.2in}
    \subfloat[Wall-Time]{\includegraphics[width=0.40\linewidth]{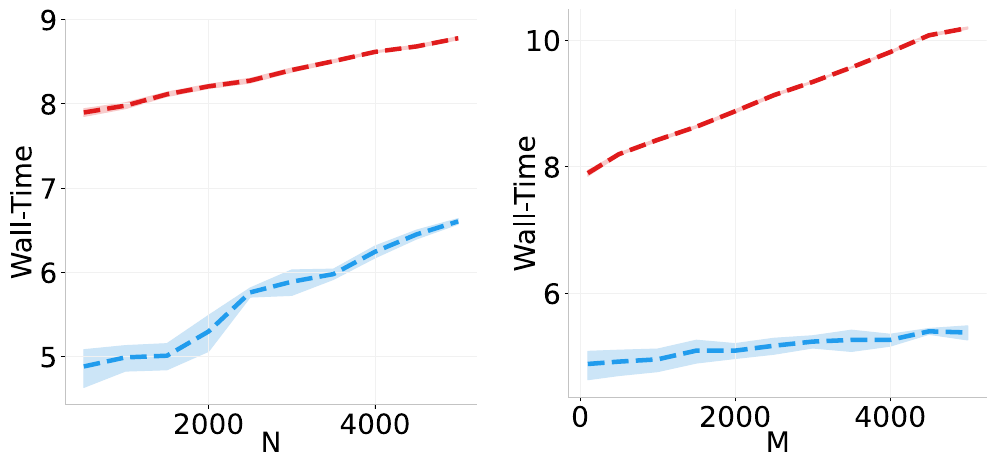}
    \label{subfig:toy_example_right}
    }~
    \vspace{-0.12in}
    \caption{\textbf{(a)} Comparison of true latent variable $\mathbf{X}$ or true kernel matrix $\K$ with inferred latent variables $\hat{\mathbf{X}}$ or inferred kernel matrices $\hat{\K}$ obtained from various methods. \textbf{(b)} Logarithm of wall-time for model fitting plotted against $N$ or $M$.}
    \label{fig:toy_example}
\end{figure*}

\section{Simulation Results}
\label{sec:results}

In this section, we show the efficacy of the proposed SRFLVMs on synthetic datasets (Section~\ref{subsec:toy_example}), real-world datasets (Section~\ref{subsec:text_image}) and missing data imputation (Section~\ref{subsec:missing_data}). Additionally, to illustrate the scalability of SRFLVMs, we conduct comprehensive comparisons on larger datasets in Section~\ref{subsec:large_scale}. In this section, model prefixes such as 'G', 'N', 'B' denote Gaussian, negative binomial, and Bernoulli likelihoods, respectively. For further experimental details, please refer to Appendix~\ref{app:experiment}.

\input{table/knn_acc}
\input{table/wall-time}

\vspace{-0.12in}
\subsection{2-D $S$-shape Manifold Estimation }
\label{subsec:toy_example}

In this subsection, we demonstrate the superior performance and computational efficiency of the proposed SRFLVMs across various synthetic datasets, utilizing varied settings for $N$ and $M$ (with default values of $N = 500$ observations and $M = 100$  dimensions). These datasets are synthesized from a GPLVM featuring a two-dimensional ($2$-D) latent $S$-shaped manifold, employing a combination of an RBF kernel and a periodic kernel \cite{williams2006gaussian}. Detailed kernel configurations are provided in Appendix~\ref{app:ds_preprocessing}.  

The results for $S$-shaped manifold learning and kernel learning across various GPLVM variants are presented at Fig.~\ref{fig:toy_example}. We compare RFLVM-based methods with two inducing point-based methods—Bayesian GPLVM (BGPLVM) \cite{titsias2010bayesian} and GPLVM with stochastic variational inference (GPLVM-SVI) \cite{ramchandran2021latent}; additional details can be found in Table~\ref{table:comparison} and Appendix \ref{sec:related_work}. As depicted in the left-hand side of Fig. \ref{subfig:toy_example_left}, the latent representations and kernel matrices learned with RFLVM-based methods more closely align with the ground-truth latent manifold and the kernel matrix compared to those obtained with BGPLVM and GPLVM-SVI, indicating the superior performance of RFLVM-based methodologies. This performance improvement can be attributed to the DP mixture kernel of RFLVM-based methods, which offer enhanced kernel flexibility compared to the RBF kernel used in BGPLVM and GPLVM-SVI. Additionally, the comparable performance of SRFLVMs and RFLVMs suggests that employing the approximate inference method, i.e., VBI, does not significantly affect manifold learning when compared to the exact inference, i.e., MCMC. 

\vspace{-0.02in}
However, there is a notable discrepancy in scalability between the two inference methods. To elucidate this point, we illustrate the variation in wall-time for model learning with respect to $N$ and $M$ on the right side of Fig.~\ref{subfig:toy_example_right}. The results show that the wall-times for RFLVMs are consistently and significantly higher than those of our proposed methods across all values of $N$ and $M$, with a faster increase as $M$ grows. In summary, our proposed \srflvm~not only  demonstrate superiority in manifold learning but also exhibits enhanced computational efficiency compared to RFLVMs.

\vspace{-0.12in}
\subsection{Real-World Data Evaluation}
\label{subsec:text_image}

This subsection further showcases the capability and computational efficiency of our proposed method in capturing latent space across multiple real-world datasets. Due to the high computational demands of RFLVMs, running them on the full \textsc{mnist} and \textsc{cifar} datasets within a reasonable time frame is infeasible. Therefore, for a fair comparison with RFLVMs, we have reduced the sizes of these datasets, as detailed in the dataset description in Appendix~\ref{app:ds_preprocessing}. For each dataset, we retain the labels and utilize them to evaluate the estimated latent space using a \textsc{knn} classifier with five-fold cross-validation. In addition to the GPLVM variants utilized in Section~\ref{subsec:toy_example}, we also consider various recent VAEs \cite{kingma2013auto, zhao2020variational, zheng2023online} and classic dimensionality reduction methods as competing methods and the implementation detail for them could be found in Appendix~\ref{app:benchmark}. The \textsc{knn} classification accuracy results for all competing methods, along with their corresponding wall-times for model fitting, are presented in Table~\ref{table:1} and Table~\ref{table:2}, respectively. 

\vspace{-0.02in}
As shown in Table \ref{table:1}, our G-OURS and G-RFLVM achieve the highest KNN accuracy across most datasets when compared with other methods. Notably, on the \textsc{montreal} dataset, models employing negative binomial likelihood perform significantly better than those using Gaussian likelihood, with N-OURS and N-RFLVM demonstrating optimal performance. These results reveal that different datasets vary in their suitability for different likelihoods, underscoring the importance of model adaptability to multiple likelihoods. They also indicate that the latent variables estimated by RFLVM-based methods are more informative compared to alternative approaches. 

\vspace{-0.02in}
Clearly, among the five classic methods—PCA \cite{roweis1997algorithms}, NMF \cite{lee1999learning}, LDA \cite{blei2003latent}, HPF \cite{gopalan2015scalable}, and Isomap \cite{balasubramanian2002isomap}—showing lower \textsc{KNN} accuracy, primarily attributable to their constrained model flexibility.  For VAE-based models, despite their remarkable representational capabilities, they are often affected by a phenomenon known as "posterior collapse", wherein they struggle to capture essential information from observational data \cite{qian2022learning, shao2021controlvae}. Overfitting is a significant factor contributing to posterior collapse, as it leads deterministic neural networks to misleading latent spaces \cite{bowman2015generating, sonderby2016train}. Consequently, this phenomenon is more pronounced in the smaller datasets that we used to adapt to the high computational demands of RFLVMs. Therefore, extra comparisons on larger datasets will be presented in Section~\ref{subsec:large_scale} for a comprehensive evaluation. 

In contrast, the regularization imposed by the GP prior helps prevent overfitting \cite{wilson2020bayesian}, thereby enhancing the generalization capability of latent space learning and resulting in higher \textsc{knn} accuracy even with limited dataset sizes. Within the GPLVM variants, SRFLVMs and RFLVMs exhibit comparable performance, both surpassing BGPLVM and GPLVM-SVI, attributed to their enhanced kernel flexibility and reduced optimization burden. However, as shown in Table~\ref{table:2}, the average wall-times for RFLVMs are significantly higher—often by an order of magnitude—compared to our proposed methods across all datasets. This outcome underscores the substantial improvement in computational efficiency achieved by our proposed approach, making it viable for application in the big data era. 

\input{table/large_datasets}
\input{table/missing_data}

\vspace{-0.10in}
\subsection{Missing Data Imputation}
\vspace{-0.02in}
\label{subsec:missing_data}
This subsection evaluates the efficacy of SRFLVMs in the context of imputing missing data from two image datasets: \textsc{cifar} and \textsc{brendan} \cite{roweis2000nonlinear}. To evaluate performance, we randomly withhold varying proportions (0\%, 10\%, 30\%, and 60\%) of elements from the observed data matrix $\mathbf{Y}$. Subsequently, we estimate the latent variables $\mathbf{X}$ using the incomplete dataset denoted as $\mathbf{Y}_{obs}$, and impute the missing values $\mathbf{Y}_{miss}$ using their posterior mean $\hat{\mathbf{Y}}_{miss} = \mathbb{E}[ \mathbf{Y}_{miss} \mid  \mathbf{X}, \mathbf{Y}_{obs}, -]$. 
Imputation performance for different methods is presented in Table~\ref{tab:missing_data}, measured by the mean square error (MSE) between $\hat{\mathbf{Y}}_{miss}$ and the ground-truth $\mathbf{Y}_{miss}$. The results show that our methods outperform other competing methods, demonstrating its superior ability to reconstruct observations and recover latent representations across all levels of missing data. Particularly, our performance improvement over RFLVMs can be attributed to calibrated uncertainty in latent variables, which avoids inaccurate reconstruction through bypassing poorly point estimates caused by high uncertainty in missing data. Additional insights into the reconstruction performance are detailed in App.~\ref{app:missing_data}, highlighting exceptional capability of SRFLVMs in restoring missing pixels. 

\vspace{-0.1in}
\subsection{Large Scale Experiment}
\vspace{-0.02in}
\label{subsec:large_scale}

As mentioned earlier, the overfitting issue in VAEs is particularly severe on small datasets. Therefore, to conduct a fair comparison of VAE-based models and comprehensively showcase the performance of our proposed algorithm, we conducted extensive comparisons on full \textsc{mnist} and \textsc{cifar} datasets in this section. The \textsc{knn} classification accuracy results for several competing methods are summarized in Table \ref{tab:large_datasets},  where \textsc{f-cifar} and \textsc{f-mnist} represent the full \textsc{cifar} and \textsc{mnist} datasets, respectively, and \textsc{fd-cifar} denotes the full \textsc{cifar} dataset with each image downsampled to 400 pixels. Empirical results demonstrate significant performance improvements for both VAE-based models and our method when applied to large datasets in terms of generating informative latent variables, outperforming other benchmarks across datasets of varying sizes.

%% file: table/method_summary.tex
\begin{table*}[t!]
    \caption{An overview of relevant LVMs, with $N$ representing the total observations and $M$ indicating the observation dimensions. Additionally, $U$ refers to the count of inducing points, and $L$ denotes the dimension of random features.
        } \label{table:comparison}
    \centering
    \vspace{-2ex}
    \resizebox{0.94\textwidth}{!}{%
        \begin{tabular}{lccccccc}
        \toprule
            Model 
            & \begin{tabular}[c]{@{}c@{}} Scalable \\ model  \end{tabular} 
            & \begin{tabular}[c]{@{}c@{}} Advanced \\ kernel \end{tabular} 
            & \begin{tabular}[c]{@{}c@{}} Probabilistic \\ mapping  \end{tabular} 
            & \begin{tabular}[c]{@{}c@{}} Bayesian  inference \\ of latent variables   \end{tabular}  
            & \begin{tabular}[c]{@{}c@{}} Computational \\ complexity \end{tabular} 
            & \begin{tabular}[c]{@{}c@{}} Data Likelihood  \end{tabular}  
            & \begin{tabular}[c]{@{}c@{}} Reference  \end{tabular} 
           \\
            \midrule
            \textsc{vae}    
            & \cmark                       
            & -                  
            &  \xmark       
            & \cmark         
            & -                 
            & Gaussian                
            & \cite{kingma2013auto}  \\
           \textsc{nbvae}  
            & \cmark                  
            & -                  
            & \xmark & \cmark   
            & -              
            & Negative Binomial             
            & \cite{zhao2020variational}  \\
             \textsc{cvq}-\textsc{vae}       
            & \cmark                   
            & -                 
            & \xmark  & \cmark         
            & -                     
            & Gaussian                  
            & \cite{zheng2023online}  \\
            \textsc{gplvm}   
            & \xmark  
            & \xmark    
            & \cmark     
            & \xmark     
            & $\mathcal{O}(N^3)$                         
            & Gaussian    
            & \cite{lawrence2005probabilistic}   \\
             \textsc{bgplvm}                      
            & \cmark       
            & \xmark                    
            & \cmark                    
            & \cmark                          
            & $\mathcal{O}(NU^2)$                        
            & Gaussian                                                  
            & \cite{titsias2010bayesian}     \\
              \textsc{gplvm}-\textsc{svi}               
            & \cmark      
            & \xmark             
            & \cmark                    
            & \cmark   
            & $\mathcal{O}(MU^3)$   
            & Any    
            & \cite{lalchand2022generalised}   \\
             \textsc{rflvm}      
            & \xmark                 
            & \cmark                  
            & \cmark      
            & \xmark            
            & $\mathcal{O}(NL^2)$    
            & Any   
            & \cite{gundersen2021latent}   \\
            \midrule
             \textsc{srflvm}        
            & \cmark                                 
            & \cmark & \cmark                
            & \cmark      
            & $\mathcal{O}(N L^2)$    
            & Any              
            & \textbf{This work} \\
            \bottomrule
        \end{tabular}}
    \vspace{-.1in}
\end{table*}

%% file: table/knn_acc.tex
\begin{table*}[t] 
\centering
\caption{Classification accuracy evaluated by fitting a KNN classifier ($K=1$) with five-fold cross validation. Mean accuracy and standard deviation were computed by running with each experiments five times. For each dataset, we bold the value of the best KNN accuracy and underline the second-best. }
\vspace{-0.1in}
\scalebox{0.94}{
\begin{tabular}{c cccccccc}
\toprule
        \textsc{dataset}   & PCA               & LDA               & Isomap            & NMF   & HPF                      & G-RFLVM   &  N-RFLVM  &  B-RFLVM          \\ \midrule \midrule
\rowcolor[HTML]{f2f2f2}  \textsc{bridges}    & 0.841 $\pm$ 0.007 & 0.668 $\pm$ 0.053 & 0.797 $\pm$ 0.025 & \textbf{{0.864 $\pm$ 0.015}} & 0.544 $\pm$ 0.109  & 0.846 $\pm$ 0.039 & 0.850 $\pm$ 0.011 & 0.839 $\pm$ 0.014 
\\
\textsc{cifar-10}   & 0.267 $\pm$ 0.002 & 0.227 $\pm$ 0.006 & 0.272 $\pm$ 0.006 & 0.246 $\pm$ 0.006 & 0.208 $\pm$ 0.006  & \textbf{0.284 $\pm$ 0.103} & 0.260 $\pm$ 0.004 & 0.268 $\pm$ 0.004 
\\
\rowcolor[HTML]{f2f2f2}  \textsc{mnist}      & 0.365 $\pm$ 0.012 & 0.233 $\pm$ 0.026 & 0.444 $\pm$ 0.021 & 0.269 $\pm$ 0.018 & 0.314 $\pm$ 0.040 & {\underline{0.602 $\pm$ 0.055}} & 0.441 $\pm$ 0.011 & 0.361 $\pm$ 0.019 
\\
\textsc{montreal}   & 0.678 $\pm$ 0.013 & 0.602 $\pm$ 0.028 & 0.709 $\pm$ 0.005 & 0.678 $\pm$ 0.014 & 0.618 $\pm$ 0.001  &  0.769 $\pm$ 0.010 & \textbf{0.826 $\pm$ 0.015} & 0.744 $\pm$ 0.004
\\
\rowcolor[HTML]{f2f2f2}  \textsc{newsgroups} & 0.392 $\pm$ 0.005 & 0.391 $\pm$ 0.018 & 0.397 $\pm$ 0.010 & 0.393 $\pm$ 0.004 & 0.334 $\pm$ 0.019 & \textbf{0.413 $\pm$ 0.009} & 0.405 $\pm$ 0.006 & 0.386 $\pm$ 0.007 
\\
\textsc{yale}       & 0.543 $\pm$ 0.008 & 0.338 $\pm$ 0.023 & 0.588 $\pm$ 0.017 & 0.479 $\pm$ 0.022 & 0.511 $\pm$ 0.019  & \textbf{0.653 $\pm$ 0.067} & 0.529 $\pm$ 0.026 & 0.527 $\pm$ 0.010 
\\ \midrule 
       \textsc{dataset}    & VAE               &  CVQ-VAE               & NBVAE             & BGPLVM    & GPLVM-SVI       & G-OURS &     N-OURS  &  B-OURS           \\ \midrule \midrule
\rowcolor[HTML]{f2f2f2}  
\textsc{bridges}    & 0.751 $\pm$ 0.016 & 0.688 $\pm$ 0.013 & 0.758 $\pm$ 0.038 & 0.818 $\pm$ 0.037 & 0.796 $\pm$ 0.019  & {\underline{0.851 $\pm$ 0.017}} &   0.846 $\pm$ 0.019    & 0.833 $\pm$ 0.022             
\\
\textsc{cifar-10}   & 0.266 $\pm$ 0.002 &  0.224 $\pm$ 0.012 & 0.259 $\pm$ 0.005 & 0.271 $\pm$ 0.014 &  0.251 $\pm$ 0.012  & {\underline{0.279 $\pm$ 0.010}} &  0.258 $\pm$ 0.008  & 0.262 $\pm$ 0.006               
\\
\rowcolor[HTML]{f2f2f2}  \textsc{mnist}      & \textbf{0.643 $\pm$ 0.021} &  0.128 $\pm$ 0.005 & 0.281 $\pm$ 0.012 & 0.342 $\pm$ 0.074 &  0.344 $\pm$ 0.054 & {{0.584 $\pm$ 0.042}} &    0.408 $\pm$ 0.008  & 0.369 $\pm$ 0.008            
\\
\textsc{montreal}   & 0.668 $\pm$ 0.012 & 0.646 $\pm$ 0.003 & 0.716 $\pm$ 0.009 & 0.725 $\pm$ 0.012 & 0.676 $\pm$ 0.010 & {0.744 $\pm$ 0.026} & {\underline{0.781 $\pm$ 0.010}}   & 0.722 $\pm$ 0.010               
\\
\rowcolor[HTML]{f2f2f2}  \textsc{newsgroups} & 0.385 $\pm$ 0.002 & 0.356 $\pm$ 0.019 & 0.398 $\pm$ 0.010 & 0.385 $\pm$ 0.010 &  0.378 $\pm$ 0.018  & {\underline{0.411 $\pm$ 0.008}} &  0.391 $\pm$ 0.005   & 0.398 $\pm$ 0.009 \\
\textsc{yale}       & 0.611 $\pm$ 0.020 & 0.338  $\pm$ 0.002  & 0.456 $\pm$ 0.046 &  0.553 $\pm$ 0.036  & 0.521 $\pm$ 0.015 & {\underline{0.635 $\pm$ 0.024}} &  0.549 $\pm$ 0.014 & 0.487 $\pm$ 0.011  \\\bottomrule          
\end{tabular}
\label{table:1}
} \vspace{-.1in}
\end{table*}

%% file: table/wall-time.tex
\begin{table*}[ht] 
\centering
\caption{Wall-time in seconds for model fitting. Mean and standard error were computed by running each experiment five times.}
\vspace{-0.1in}
\scalebox{0.81}{
\begin{tabular}{ccccccccc}
\toprule
           & PCA                 & LDA                & Isomap             & NMF  & HPF                                & G-RFLVM    & N-RFLVM & B-RFLVM         \\ \midrule \midrule
\rowcolor[HTML]{f2f2f2} \textsc{bridges}    & 0.008 $\pm$ 0.002   & 1.686 $\pm$ 0.011  & 0.019 $\pm$ 0.007  & 0.014 $\pm$ 0.004  & 0.031 $\pm$ 0.004  & 679.073 $\pm$ 2.094  & 1578.207 $\pm$ 59.981 & 1397.628 $\pm$ 54.404
\\
\textsc{cifar-10}   & 0.070 $\pm$ 0.004   & 21.675 $\pm$ 0.050  & 1.367 $\pm$ 0.034  & 0.363 $\pm$ 0.005 & 140.168 $\pm$ 0.898   & 9011.504 $\pm$ 3.445 & 26159.991 $\pm$ 63.597 & 6140.220 $\pm$ 207.333 
\\
\rowcolor[HTML]{f2f2f2} \textsc{mnist}      & 0.056 $\pm$ 0.010   & 6.708 $\pm$ 0.317  & 0.372 $\pm$ 0.017  & 0.161 $\pm$ 0.019 & 54.295 $\pm$ 0.462  & 9151.876 $\pm$ 7.758  & 13165.206 $\pm$ 130.369 & 6149.416 $\pm$ 69.432
\\
\textsc{montreal}   & 0.010 $\pm$ 0.003   & 2.285 $\pm$ 0.212  & 0.025 $\pm$ 0.004  & 0.019 $\pm$ 0.003 &  0.105 $\pm$ 0.002  & 485.678 $\pm$ 0.615  & 1844.815 $\pm$ 47.995 &  1493.150 $\pm$ 49.798
\\
\rowcolor[HTML]{f2f2f2} \textsc{newsgroups} & 0.056 $\pm$ 0.003   & 3.355 $\pm$ 0.186  & 1.758 $\pm$ 0.037  & 0.063 $\pm$ 0.003 & 0.257 $\pm$ 0.006 & 1763.735 $\pm$ 2.923  & 2697.096 $\pm$ 68.233 & 2236.637 $\pm$ 29.174
\\
\textsc{yale}       & 0.017 $\pm$ 0.004   & 3.123 $\pm$ 0.019  & 0.017 $\pm$ 0.006  & 0.102 $\pm$ 0.005 & 5.112 $\pm$ 0.027  & 4432.849 $\pm$ 2.301  & 11085.157 $\pm$ 32.951 & 3848.943 $\pm$ 127.506
\\
 \midrule
           & VAE                 & CVQ-VAE                & NBVAE              & BGPLVM   & GPLVM-SVI             & G-OURS   &  N-OURS  &  B-OURS                 \\ \midrule \midrule 
\rowcolor[HTML]{f2f2f2} \textsc{bridges}    & 4.802 $\pm$ 0.131  & 8.858 $\pm$ 2.238  & 2.787 $\pm$ 0.065  & 2.473 $\pm$ 0.850  & 24.083 $\pm$ 0.274 & 103.550 $\pm$ 0.296  &  78.630 $\pm$ 5.176 & 143.768 $\pm$ 11.651                   
\\
\textsc{cifar-10}   & 130.226 $\pm$ 1.072 & 12.342 $\pm$ 0.155 & 49.932 $\pm$ 1.247 &  18.984 $\pm$ 0.147 & 72134.541 $\pm$ 13.921  & 130.710 $\pm$ 1.433 &  575.494 $\pm$ 24.201 & 574.890 $\pm$ 15.192            
\\
\rowcolor[HTML]{f2f2f2} \textsc{mnist}      & 93.737 $\pm$ 0.754  & 7.854 $\pm$ 1.691  & 43.766 $\pm$ 1.628 & 8.715 $\pm$ 5.662 & 35776.171 $\pm$ 43.995 &  111.560 $\pm$ 0.694 &  719.091 $\pm$ 62.666 & 719.651 $\pm$ 3.250                   
\\
\textsc{montreal}   & 7.106 $\pm$ 0.033   & 8.993 $\pm$ 0.414  & 3.250 $\pm$ 0.043  & 3.180 $\pm$ 0.027 & 65.347 $\pm$ 4.550  & 44.631 $\pm$  0.375  &  173.163 $\pm$ 4.380  & 133.205 $\pm$ 1.812                  
\\
\rowcolor[HTML]{f2f2f2} \textsc{newsgroups} & 57.949 $\pm$ 0.478  & 10.817 $\pm$ 1.244 & 18.986 $\pm$ 0.100 & 4.483 $\pm$ 3.331 & 7135.885 $\pm$ 48.387 & 81.730 $\pm$ 0.385 &  283.602 $\pm$ 6.276  & 131.833 $\pm$ 7.838               
\\
\textsc{yale}       & 33.548 $\pm$ 0.334  & 13.464 $\pm$ 3.356 & 17.461 $\pm$ 0.739 & 5.402 $\pm$ 2.297 & 2343.803 $\pm$ 6.988 &  185.611 $\pm$ 0.714 &  848.532 $\pm$ 50.485 & 857.537 $\pm$ 50.530      \\ \bottomrule                        
\end{tabular}
\label{table:2}
}\vspace{-.1in}
\end{table*}

%% file: table/large_datasets.tex
\begin{table*}[t!]
\centering
\caption{KNN classification accuracy with different $k$ values on larger datasets. We ran this classification using 5-fold cross validation. 
} \label{tab:large_datasets}
\vspace{-0.1in}
\scalebox{.87}{
\begin{tabular}{@{}l || lllllllll | lllllllll@{}}
\toprule
\multicolumn{1}{c||}{\textsc{methods}}                   & \multicolumn{9}{c|}{VAE}                                 & \multicolumn{9}{c}{G-OURS}                 \\ 
\multicolumn{1}{c||}{$k$-\textsc{value}}   & \multicolumn{1}{c}{2} & \multicolumn{1}{c}{3} & \multicolumn{1}{c}{4} & \multicolumn{1}{c}{5} & \multicolumn{1}{c}{6} & \multicolumn{1}{c}{7} & \multicolumn{1}{c}{8} & \multicolumn{1}{c}{9} & \multicolumn{1}{c|}{10} & \multicolumn{1}{c}{2} & \multicolumn{1}{c}{3} & \multicolumn{1}{c}{4} & \multicolumn{1}{c}{5} & \multicolumn{1}{c}{6} & \multicolumn{1}{c}{7} & \multicolumn{1}{c}{8} & \multicolumn{1}{c}{9} & \multicolumn{1}{c}{10} 
\\ \midrule\midrule 
 \rowcolor[HTML]{f2f2f2} 
 \multicolumn{1}{c||}{\textsc{f-Cifar}}    &	0.151 &	0.157 &	0.166	& 0.174 &	0.178 &	0.180 &	0.187 &	0.188 &	0.190  & \textbf{0.157} & \textbf{0.166} & \textbf{0.179}  & \textbf{0.188} & \textbf{0.191} &  \textbf{0.193} &  \textbf{0.195} & \textbf{0.198} & \textbf{0.199}   
  \\ 
 \multicolumn{1}{c||}{\textsc{fd-Cifar}}   	& 0.266 &	0.279 &	0.285 & 0.293	& 0.297	& 0.302	& 0.304	& 0.309 & 0.312 & \textbf{0.275} &  \textbf{0.283} &  \textbf{0.291} &  \textbf{0.304} &  \textbf{0.303} & \textbf{0.307} & \textbf{0.314} &  \textbf{0.316} &
 \textbf{0.316}      
  \\ 
 \rowcolor[HTML]{f2f2f2}  \multicolumn{1}{c||}{\textsc{f-Mnist}}    & \textbf{0.728}  & \textbf{0.756}  & \textbf{0.766}  & \textbf{0.774}  & \textbf{0.775}  & \textbf{0.778}  & \textbf{0.782}  & \textbf{0.782}  &  \textbf{0.783}  & 0.672  &   0.705  &   0.718 &  0.726 & 0.731 & 0.734 & 0.736 & 0.739 &  0.740            
\\ \midrule
\multicolumn{1}{c||}{\textsc{methods}}    & \multicolumn{9}{c|}{BGPLVM}                              & \multicolumn{9}{c}{Isomap}                               
\\  
\multicolumn{1}{c||}{$k$-\textsc{value}}   & \multicolumn{1}{c}{2} & \multicolumn{1}{c}{3} & \multicolumn{1}{c}{4} & \multicolumn{1}{c}{5} & \multicolumn{1}{c}{6} & \multicolumn{1}{c}{7} & \multicolumn{1}{c}{8} & \multicolumn{1}{c}{9} & \multicolumn{1}{c|}{10} & \multicolumn{1}{c}{2} & \multicolumn{1}{c}{3} & \multicolumn{1}{c}{4} & \multicolumn{1}{c}{5} & \multicolumn{1}{c}{6} & \multicolumn{1}{c}{7} & \multicolumn{1}{c}{8} & \multicolumn{1}{c}{9} & \multicolumn{1}{c}{10}          
\\ \midrule\midrule
\rowcolor[HTML]{f2f2f2} \multicolumn{1}{c||}{\textsc{f-cifar}}     
 & 0.132 & 0.140 &	0.145 &	0.154 &	0.156 &	0.159 &	0.161 &	0.162	 &  0.144 &	0.142 &	0.147 &	0.157 	& 0.159 &	0.163 & 	0.165 &	0.168	& 0.170 &	0.173  
 \\ 
 \multicolumn{1}{c||}{\textsc{fd-Cifar}}     
& 0.260	& 0.258 & 	0.264 &	0.279 &	0.277 &	0.281	& 0.285	& 0.287 &	0.287  &	0.270 &	0.279 &	0.287 &	0.291 &	0.292 &	0.296	& 0.301 &	0.305	& 0.305 
  \\ 
 \rowcolor[HTML]{f2f2f2}  \multicolumn{1}{c||}{\textsc{f-Mnist}}      & 0.420	& 0.433 &	0.449 &	0.455 &	0.464	& 0.466 &	0.470	& 0.473 &	0.474  &	0.468 &	0.493	& 0.504 &	0.514 &	0.524 &	0.529	& 0.534	& 0.535 &	0.540
\\ \midrule 
\multicolumn{1}{c||}{\textsc{methods}}    & \multicolumn{9}{c|}{NBVAE}                              & \multicolumn{9}{c}{CVQ-VAE}                                
\\  
\multicolumn{1}{c||}{$k$-\textsc{value}}    & \multicolumn{1}{c}{2} & \multicolumn{1}{c}{3} & \multicolumn{1}{c}{4} & \multicolumn{1}{c}{5} & \multicolumn{1}{c}{6} & \multicolumn{1}{c}{7} & \multicolumn{1}{c}{8} & \multicolumn{1}{c}{9} & \multicolumn{1}{c|}{10} & \multicolumn{1}{c}{2} & \multicolumn{1}{c}{3} & \multicolumn{1}{c}{4} & \multicolumn{1}{c}{5} & \multicolumn{1}{c}{6} & \multicolumn{1}{c}{7} & \multicolumn{1}{c}{8} & \multicolumn{1}{c}{9} & \multicolumn{1}{c}{10}                     
\\ \midrule\midrule
\rowcolor[HTML]{f2f2f2} \multicolumn{1}{c||}{\textsc{f-cifar}}                         & 0.137 & 0.140 & 0.147 & 0.152 & 0.157 & 0.156  & 0.155 & 0.161 & 0.162    &   0.098   &   0.099   &   0.101   &     0.102    &   0.102   &   0.101   &   0.100   &   0.098   &   0.096   
\\  
\multicolumn{1}{c||}{\textsc{fd-Cifar}}    & 0.248 & 0.255 & 0.264 & 0.273 & 0.277 & 0.282  & 0.287 & 0.291 & 0.292    &   0.201   &   0.200   &   0.199   &     0.201    &   0.200   &   0.201   &   0.199   &   0.197   &   0.200                
  \\ 
 \rowcolor[HTML]{f2f2f2}  \multicolumn{1}{c||}{\textsc{f-Mnist}}          & 0.502 & 0.533 & 0.548 & 0.557 & 0.566 & 0.571  & 0.577 & 0.579 & 0.582  &   0.107   &   0.107   &   0.104   &     0.105    &   0.106   &   0.102   &   0.103   &   0.103   &   0.103 
 \\
\bottomrule
\end{tabular}
}
\vspace{-.1in}
\end{table*}

%% file: table/missing_data.tex
\begin{table*}[t]
\centering
\caption{Missing data imputation on the \textsc{cifar} and \textsc{brendan} datasets. 
\label{tab:missing_data}}
\vspace{-2ex}
\small
\setlength{\tabcolsep}{1.05mm}{
    \scalebox{.97}{
        \begin{tabular}{cr|| cccc| cccc| cccc| cccc}
        \toprule
        \multirow{2}{*}{\textsc{dataset}} 
        & \multirow{2}{*}{\textsc{metric}}   
        & \multicolumn{4}{c|}{\textsc{vae}}  
        & \multicolumn{4}{c|}{\textsc{bgplvm}}   
        & \multicolumn{4}{c|}{\textsc{rflvm}}  
        & \multicolumn{4}{c}{\textsc{srflvm}}   
        \\  
        \cmidrule(lr){3-6}
        \cmidrule(lr){7-10}
        \cmidrule(lr){11-14}
        \cmidrule(lr){15-18}
        & &0\% &10\% &30\% &60\%   &0\% &10\% &30\% &60\% &0\% &10\% &30\% &60\%     &0\% &10\% &30\% &60\% 
        \\ \midrule  \midrule
        \multirow{1}{*}{\textsc{cifar}}  
        & \textsc{test mse} ($\downarrow$) &0.023& 0.044 &   0.071 &  0.072 & 0.025 & 0.028 & 0.048 & 0.119  & 0.030 &0.050 & 0.055 & 0.127 &\textbf{0.022} & \textbf{0.026} &\textbf{0.039}& \textbf{0.048}
        \\ \midrule \midrule
        \multirow{1}{*}{\textsc{brendan}}  
        & \textsc{test mse} ($\downarrow$) &0.005 & \textbf{0.009} &  0.043 & \textbf{0.150} & 0.006 &0.041 &0.087 &0.197 & 0.010 & 0.015 & 0.049 & 0.153 & \textbf{0.004} & \textbf{0.009} & \textbf{0.042} & {0.152} 
        \\ \bottomrule
        \end{tabular}
    }
}
\end{table*}

%% file: content/conclusion.tex
\vspace{-0.1in}
\section{Conclusions}
\label{sec:conclusions}
\vspace{-0.02in}

This paper advances the implementation of VBI in RFLVM with various data likelihoods by overcoming two key challenges, leading to the development of a scalable variant of RFLVMs, termed SRFLVMs. Specifically, we introduced a stick-breaking construction for the DP in the kernel learning component of RFLVMs, effectively managing the intractable ELBO that results from the absence of an explicit PDF for the DP. Additionally, we proposed a generic BCD-VI algorithm framework, which resolves the existing incompatibility of VBI with RFLVMs, enhances algorithmic efficiency, and ensures convergence. Empirical results robustly validate that our SRFLVMs not only exhibit superior computational efficiency and scalability but also outperform SOTA methods in generating informative latent representations and imputing missing data across various real-world datasets. 


%% file: appendix/appendix.tex
\appendices

\input{appendix/kernel}

\input{appendix/mf}

\input{appendix/srflvm}

\input{appendix/related_work}

\input{appendix/experiments}

%% file: appendix/kernel.tex
\vspace{-0.1in}
\section{Random Fourier Features-based Estimator}
\label{app:kernel_estimator}
\begin{proof}
   According to the Bochner's theorem \cite{rahimi2008random}, we have
   \vspace{-0.1in}
   \begin{align}
       \kappa \left(\mathbf{x}-\mathbf{x}^{\prime}\right)
            & =\int p(\mathbf{w}) \exp \left(i \mathbf{w}^{\top}\left(\mathbf{x}-\mathbf{x}^{\prime}\right)\right) \mathrm{d} \mathbf{w}, \nonumber \\ 
            &=  \int p(\mathbf{w})    \cos(\mathbf{w}^{\top}\left(\mathbf{x}-\mathbf{x}^{\prime}\right))  \mathrm{d} \mathbf{w}. \label{eq:rff_boch}
   \end{align}
   Then, by sampling from the spectral density $p(\w)$, we can obtain the RFFs-based estimator for Eq.~\eqref{eq:rff_boch} using MC approximation:
   \vspace{-0.1in}
   \begin{align}
        & ~~~~~ \frac{2}{L} \sum_{l=1}^{L/2} \cos \left( \mathbf{w}_l^{\top}\left(\mathbf{x}-\mathbf{x}^{\prime} \right) \right), \nonumber \\ 
        & = \frac{2}{L} \sum_{l=1}^{L/2} \cos(\w^{\top}_l \x)  \cos(\w^{\top}_l \x^{\prime}) +  \sin(\w^{\top}_l \x)  \sin(\w^{\top}_l \x^{\prime}), \nonumber  \\ 
        & =  \vvarphi(\x)^{\top} \vvarphi(\x^{\prime}), 
   \end{align}
    where 
    \begin{align}
        \vvarphi(\x) \triangleq \sqrt{\frac{2}{L}} \begin{bmatrix}
        \sin(\w_1^{\top} \x) \\ \cos(\w_1^{\top} \x) \\ \vdots \\ \sin(\w_{L/2}^{\top} \x) \\ \cos(\w_{L/2}^{\top} \x) 
        \end{bmatrix}
        , ~ \w_l \stackrel{\text{iid}}{\sim} p(\w). 
    \end{align}
   The quality of the RFF-based estimator improves with increasing the number of spectral points $L$, approaching the true kernel as $L$ tends to infinity \cite{rahimi2008random}.
\end{proof}

%% file: appendix/mf.tex
\vspace{-0.1in}
\section{BCD-VI}
\label{app:bcd-vi}
\vspace{-0.05in}

\subsection{Proof of Theorem 2}
\label{app:mean_field}
\begin{proof}
Combining the model joint distribution $p(\bTheta, \Y)$ with its variational distributions $q(\bTheta)$, we can obtain the ELBO $\mathcal{L}(\bTheta)$: 
\begin{align}
       \mathcal{L}(\bTheta) &= \mathbb{E}_{q(\bTheta)}\left[\log \frac{p(\bTheta, \Y)}{q(\bTheta)}\right], \label{eq:elbo-bcd-vi} \\ 
       &=  \underbrace{ \mathbb{E}_{q(\bTheta)} \left[ \log p(\bTheta, \Y ) \right]}_{A} - \underbrace{\mathbb{E}_{q(\bTheta)} \left[ \log q(\bTheta) \right]}_{B}. \nonumber
\end{align}
Given the condition $q(\bTheta)=q(\bTheta_b)q(\bTheta_{\neq b})$, the term $A$ can be reformulated as
\begin{align}
    A  &= \int q(\bTheta) \log p(\bTheta, \Y) \mathrm{d} \bTheta,  \\
       & = \int q(\bTheta_b) \left\{\int \log p(\bTheta, \Y ) \prod_{k \neq b} q(\bTheta_k) \mathrm{~d} \bTheta_k \right\} \mathrm{d} \bTheta_b, \nonumber \\  
       &= \int q(\bTheta_b) \log \widetilde{p}\left(\bTheta, \Y \right) \mathrm{d} \bTheta_b, \nonumber
\end{align}
where 
\begin{align}
    \log \widetilde{p}\left(\bTheta, \Y \right) =\mathbb{E}_{q(\bTheta_{\neq b})}[\log p(\bTheta, \Y)]+\text { const. }
\end{align}
Now, let's expand term $B$, where 
\begin{align}
    \begin{aligned}
B & =\underset{q(\bTheta)}{\mathbb{E}}\left[\log q\left(\bTheta_b \right)+\sum_{k \neq b} \log q \left(\bTheta_k \right)\right], \\
& =\underset{q(\bTheta)}{\mathbb{E}}\left[\log q \left(\bTheta_b \right)\right]+\underset{q(\bTheta)}{\mathbb{E}}\left[\sum_{k \neq b} \log q\left(\bTheta_k \right)\right], \\
& =\underset{q\left(\bTheta_b \right)}{\mathbb{E}}\left[\log q\left(\bTheta_b \right)\right]+\underset{q\left(\bTheta_{\neq b}\right)}{\mathbb{E}}\left[\sum_{k \neq b} \log q\left(\bTheta_k\right)\right]. 
\end{aligned}
\end{align}
As we will be maximizing w.r.t. just $q(\bTheta_b)$, we can treat all terms that do not include this factor as the constant, then, the term $B$ can be reformulated as
\begin{align}
    B = \int q(\bTheta_b) \log q(\bTheta_b) \mathrm{d} \bTheta_b + \text{const}. 
\end{align}
Now, with the expanding of both terms, we can put it back together
\begin{align}
    \begin{aligned}
        \mathcal{L}(\bTheta) &= \int q(\bTheta_b) \log \widetilde{p}\left(\bTheta, \Y \right) \mathrm{d} \bTheta_b  \\ 
    &~~~ - \int q(\bTheta_b) \log q(\bTheta_b) \mathrm{d} \bTheta_b + \text{const}. 
    \end{aligned}
    \label{eq:mean-field-app}
\end{align}
Now, let us assume that we keep the $q(\bTheta_{\neq b})$ fixed and aim to maximize $\mathcal{L}(\bTheta)$ in equation \eqref{eq:mean-field-app} with respect to the distribution $q\left(\bTheta_b\right)$. This can be easily achieved by recognizing that \eqref{eq:mean-field-app} represents a negative KL divergence between $q\left(\bTheta_b \right)$ and $\widetilde{p}\left( \bTheta, \Y \right)$. Therefore, maximizing \eqref{eq:mean-field-app} with respect to $q\left(\bTheta_b\right)$ is equivalent to minimizing the KL divergence, which occurs when $q\left(\bTheta_b\right)$ is equal to $\widetilde{p}\left( \bTheta, \Y \right)$. As a result, we obtain a general expression for the optimal solution $q^{\star}\left(\bTheta_b\right)$, which is given by
\begin{align}
    \log q^{\star}\left(\bTheta_b\right)=\mathbb{E}_{q( \bTheta_{\neq b})}[\log p(\bTheta, \Y)]+\text { const. }
\end{align}
\end{proof}

\vspace{-.1in}
\subsection{Proof of Corollary 2.1}
\label{app:close-form}
\begin{proof}
    The ELBO defined in Eq.~\eqref{eq:elbo-bcd-vi} can be rewritten as 
    \begin{align}
        & \log p(\Y, \bTheta_{\neq b}) + \mathbb{E}_{q(\bTheta)} \left[ \log p(\bTheta_b \mid \bTheta_{\neq b}, \Y) \right]  \\ &~-  \mathbb{E}_{q(\bTheta_b)} \left[ \log q(\bTheta_b) \right] - \mathbb{E}_{q(\bTheta_{\neq b})}\left[ \log q(\bTheta_{\neq b}) \right]. \nonumber 
    \end{align}
    Given the condition that $q(\bTheta_b)$ is in the exponential family, we have
    \begin{align}
        q(\bTheta_b) = h(\bTheta_b) \exp( \boldsymbol{\nu}_b^{\top} \bTheta_b - a(\boldsymbol{\nu}_b) ), 
    \end{align}
    where $\boldsymbol{\nu}_b$ is the variational parameter for $q(\bTheta_b)$. Treating other terms ($q(\bTheta_{\neq b})$) as constant, the ELBO can be expressed as 
    \begin{align}
        L(\bTheta) \triangleq L &= \mathbb{E}_{q(\bTheta)} \left[ 
        \log p(\bTheta_b \mid \bTheta_{\neq b}, \Y) \right] - \mathbb{E}_{q(\bTheta_b)} \left[ \log q(\bTheta_b) \right] \nonumber \\ 
        &= \mathbb{E}_{q(\bTheta)} \left[ 
        \log p(\bTheta_b \mid \bTheta_{\neq b}, \Y) \right] - \mathbb{E}_{q(\bTheta_b)} \left[ 
\log h(\bTheta_b) \right] \nonumber \\ 
&~~~~- \boldsymbol{\nu}_b^{\top} a^{\prime}(\boldsymbol{\nu}_b) + a(\boldsymbol{\nu}_b). 
    \end{align}
    The derivative with respect to $\boldsymbol{\nu}_b$ is 
    \begin{align}
        \frac{\partial L}{ \partial \boldsymbol{\nu}_b} &= \frac{\partial } {\partial \boldsymbol{\nu}_b} \left( \mathbb{E}_{q(\bTheta)} \left[ 
        \log p(\bTheta_b \mid \bTheta_{\neq b}, \Y) \right] - \mathbb{E}_{q(\bTheta_b)} \left[ 
\log h(\bTheta_b) \right]  \right) \nonumber  \\ 
&~~~~- \boldsymbol{\nu}_b^{\top} a^{\prime \prime}(\boldsymbol{\nu}_b). 
    \end{align}
    The optimal $\boldsymbol{\nu}_b$ satisfies
    \begin{align}
        \boldsymbol{\nu}_b^{\star} &= [ a^{\prime \prime}(\boldsymbol{\nu}_b) ]^{-1} \frac{\partial } {\partial \boldsymbol{\nu}_b} \left( \mathbb{E}_{q(\bTheta)} \left[ 
        \log p(\bTheta_b \mid \bTheta_{\neq b}, \Y) \right] \right. \nonumber \\ 
        &~~~ \left. - \mathbb{E}_{q(\bTheta_b)} \left[ 
        \log h(\bTheta_b) \right]  \right) . 
        \label{eq:opt_nu}
    \end{align}
    In particular, if the conditional distribution $p(\bTheta_b \mid \bTheta_{\neq b}, \Y)$ is also an exponential family distribution then
\begin{align}
    p(\bTheta_b \mid \bTheta_{\neq b}, \Y) &= h(\bTheta_b) \exp \left\{ g_b(\bTheta_{\neq b}, \Y)^{\top} \bTheta_b \right. \nonumber \\ 
   & ~~~~ \left. - a\left( g_b(\bTheta_{\neq b}, \Y) \right) \right\}, \nonumber
\end{align}
where $g_i(\bTheta_{\neq b}, \Y)$ denotes the natural parameter for $\bTheta_b$ when conditioning on the observations $\Y$. Then, the simplified expression for the first derivative of the expected log probability of $\bTheta_{b}$ can be attained as 
\begin{align}
    \frac{\partial}{\partial \boldsymbol{\nu}_b} \mathbb{E}_{q(
        \bTheta_b)} \left[ h(\bTheta_b) \right]  + \mathbb{E}_{q(\bTheta)} \left[ g_b(\bTheta_{\neq b}, \Y) \right]^{\top} a^{\prime \prime}(\boldsymbol{\nu}_b) .  \nonumber
\end{align}
Replacing this simplified expression into Eq.~\eqref{eq:opt_nu}, the maximum is attained as
\begin{align}
    \boldsymbol{\nu}_b^{\star} = \mathbb{E}_{q(\bTheta)} \left[ g_b(\bTheta_{\neq b}, \Y)\right],  
\end{align}
which can be analytically computed \cite{blei2003latent}.  
\end{proof}

%% file: appendix/srflvm.tex
\vspace{-.12in}
\section{Scalable RFLVMs (SRFLVMs)}
\label{app:scalable_RFLVMs}

\subsection{ELBO Derivation and Evaluation} 

\subsubsection{Gaussian Case}
\label{app:Gaussian_elbo}

\begin{small}
    \begin{equation}
    \begin{aligned}
            \mathcal{L} & = {\sum_{m=1}^M \mathbb{E}_{q(\X)q(\mathbf{W})} \left[ \log p(\y_{:, m} \vert \X, \mathbf{W}) \right]} {-\!\sum_{n=1}^N\operatorname{KL}(q(\x_n) \| p(\x_n))} \\
            & \approx \sum_{m=1}^M\!\frac{1}{I}\!\sum_{i=1}^{I}\!\log \mathcal{N}(\y_{:, m} \vert \boldsymbol{0}, \hat{\K}^{(i)}\!+\!\sigma^2 \mathbf{I}_N)\!-\!\sum_{n=1}^N\!\operatorname{KL}(q(\x_n) \| p(\x_n)) \\
            & \approx\!\sum_{m=1}^M \!\frac{1}{I}\!\sum_{i=1}^I\!\left\{\!-\!\frac{1}{2}\log\! \left|\hat{\K}^{(i)}\!+\! \sigma^2 \mathbf{I}_N\!\right|\!-\!\frac{1}{2} \y_{:, m}^{\top}\!(\hat{\K}^{(i)}\!+\!\sigma^2 \mathbf{I}_N )\!^{-1}\!\y_{:, m}\!\right\}\! \\ 
            & ~~~ - \!  \frac{1}{2} \sum_{n=1}^{N} \Big[ \operatorname{tr}(\mathbf{S}_{n}) + \boldsymbol{\mu}_n^{\top} \boldsymbol{\mu}_n -\log |\mathbf{S}_{n}| - Q  \Big], 
    \end{aligned}
    \nonumber
\end{equation}
\end{small}
where $\hat{\K}^{(i)} $ denotes $  (\bPhi \bPhi^{\top})^{(i)}$. 

\subsubsection{Logistic Case}
\label{app:log_elbo}

The likelihood for logistic likelihood can be represent as 
\begin{align}
    \mathcal{L} = \text{(a)} + \text{(b)} + \text{(c)} + \text{const}, 
\end{align}
where 
\begin{small}
    \begin{align}
    \text{(a)} &= 
    \prod_{m=1}^{M} \prod_{n=1}^N \log(c_{nm}) \nonumber \\ 
    &~~~+\!\mathbb{E}_{q(\bA, \X, \W})\!\left[ a_{nm}\!\bPhi_{n, :}~\mathbf{h}_m\!-\!b_{nm}\!\log \left( 1\!+\!\exp \left( \bPhi_{n, :}~\mathbf{h}_m \right) \right) \right], \nonumber \\
    \text{(b)} &= \sum_{i=1}^{N} \operatorname{KL} \left( q(\mathbf{x}_i) \| p(\mathbf{x}_i) \right), \nonumber  \\ 
    &= \frac{1}{2} \sum_{n=1}^{N} \Big\{ - \log |\mathbf{S}_{n}| - Q + \operatorname{Tr}(\mathbf{S}_{n}) + \boldsymbol{\mu}_n^{\top} \boldsymbol{\mu}_n  \Big\}, \nonumber \\ 
    \text{(c)} &= \sum_{m=1}^M \mathbb{E}_{q(\mathbf{W}, \mathbf{X}, \boldsymbol{\Omega})}\left[ \operatorname{KL}(p(\mathbf{h}_m \vert \y_m, \bm{\Phi}, \boldsymbol{\Omega}) \| p(\mathbf{h}_m)) \right], \nonumber \\ 
    &= \frac{1}{2} \sum_{m=1}^M \mathbb{E}_{q(\mathbf{W}, \mathbf{X}, \boldsymbol{\Omega})} \left[ -\log |\boldsymbol{\Sigma}_{\mathbf{a}_m}|\!-\!mL\!+\!\operatorname{Tr}(\boldsymbol{\Sigma}_{\mathbf{a}_m})\!+\!\boldsymbol{\mu}_{\mathbf{a}_m}^{\top} \boldsymbol{\mu}_{\mathbf{a}_m} \right].  \nonumber
\end{align}
\end{small}

\subsection{Motivation for Pólya-gamma augmentation}
\label{app:motivation_for_PG}

\begin{figure}[h!]
    \centering
\includegraphics[width=0.7\linewidth]{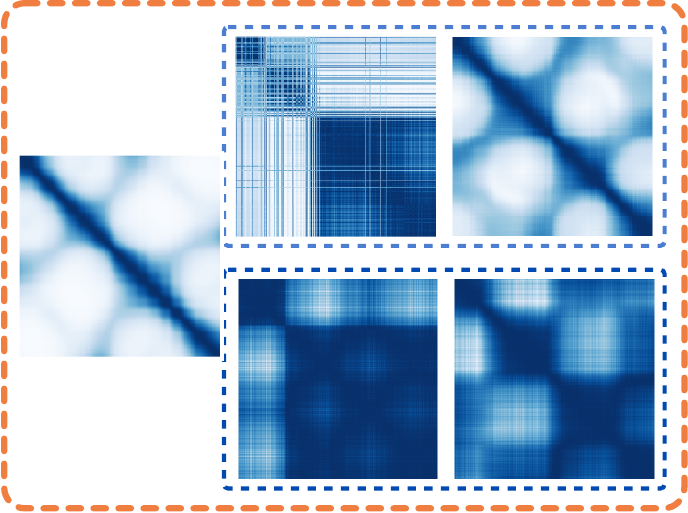}
    \caption{ (\textbf{Left}) Ground-truth kernel matrix; (\textbf{Right Top}) Gaussian case, $\hat{\K}$ obtained without marginalized $\bA$ versus with marginalized $\bA$; \textbf{(Right Bottom)} Bernoulli case, $\hat{\K}$ obtained without versus with the closed-form solution for the posterior of $\bA$.}
    \label{fig:why_PG}
\end{figure}

\noindent As demonstrated in Fig.~\ref{fig:why_PG}, marginalizing out $\bA$\footnote{Note that when $\bA$ is not marginalized out, we infer its posterior with the RGVI-based method within the likelihood block. } or employing its ground-truth posterior significantly enhances the performance of kernel learning compared to using approximate inference on $\bA$, irrespective of the likelihood type. These results underscore the necessity of utilizing PG augmentation to achieve an analytical closed-form solution for the posterior of $\bA$. 

\subsection{Proof of Theorem \ref{theo:opt_a}}
\label{app:theo_3}

Defining $\psi_{nm}$ in Eq.~\ref{eq:pg} as $\bPhi_{n, :}~\mathbf{h}_m$, we have
\begin{small}
    \begin{align}
     & \frac{\left(\exp \left( \bPhi_{n, :}~\mathbf{h}_m \right)\right)^{a_{n m}}}{\left(1+\exp \left( \bPhi_{n, :}~ \mathbf{h}_m \right)\right)^{b_{n m}}} \nonumber \\ 
    &=\!\!-2^{-b_{nm}}\!\exp(\kappa_{nm}\!\bPhi_{n, :}~ \mathbf{h}_m ) \mathbb{E}_{p(\omega \mid b_{nm}, 0)}\!\left[ \exp \left( - \omega_{nm} \frac{(\bPhi_{n, :}~ \mathbf{h}_m)^2 }{2}\!\right)\!\right] \nonumber
    \end{align} 
\end{small} 
where
\begin{align}
    \kappa_{nm} = a_{nm} - b_{nm}/2. 
\end{align}
Note that, if we condition on $\boldsymbol{\Omega}$, the likelihood is proportional to: 
\begin{align}
    & \propto -2^{-b_{nm}} \exp( \kappa_{nm} \bPhi_{n, :}~ \mathbf{h}_m )  \exp \left( - \omega_{nm} \frac{ (\bPhi_{n, :}~ \mathbf{h}_m)^2 }{2}  \right), \nonumber \\ 
    &\propto -2^{-b_{nm}} \exp \left\{ - \frac{\omega_{nm}}{2} \left( \bPhi_{n, :}~ \mathbf{h}_m - z_{nm} \right)^2  \right\},  
\end{align}
where $z_{nm} = \frac{\kappa_{nm}}{\omega_{nm}}. $
Then, we can immediately get the posterior of $\boldsymbol \omega_{nm}$ (Refer to Eq.~\eqref{eq:pos_omega}):
\begin{align}
    p( \boldsymbol \omega_{nm} \mid \ba_m) = \operatorname{PG}\left(b_{nm},  \psi_{nm} \right),  
\end{align}
and the posterior distribution of $\mathbf{h}_m$: 
\begin{align}
    \mathcal{N} \left( \mathbf{m}_{m}, \mathbf{V}_{m}  \right), 
\end{align}
where 
\begin{align*}
            & \boldsymbol{\Omega}_m\! =\!\operatorname{diag}\left(\left[\boldsymbol{\omega}_{1 m}, \ldots, \boldsymbol{\omega}_{N m}\right]\right)   & & \kappa_{nm} = a_{nm} - b_{nm} / 2    \\
            &\mathbf{V}_{m}=\left(\bPhi^{\top} \boldsymbol{\Omega}_m \bPhi+\mathbf{B}_0^{-1}\right)^{-1} & &\boldsymbol{\kappa}_m  = [ \kappa_{1m}, ..., \kappa_{nm} ]^{\top}  \\ 
            & \mathbf{m}_{m}=\mathbf{V}_{m}\left(\bPhi^{\top} \boldsymbol{\kappa}_m+\mathbf{S}_0^{-1} \mathbf{m}_0\right).
\end{align*}

%% file: appendix/related_work.tex
\section{Related Work}
\label{sec:related_work}


We briefly introduce other related work on GPLVMs, with a detailed summary provided in Table~\ref{table:comparison}. The first GPLVM was introduced by Lawrence \cite{lawrence2005probabilistic}. Later, Titsias \textit{et al.} presented a fully Bayesian GPLVM (BGPLVM) that  variationally integrate out the latent variables \cite{titsias2010bayesian}.  This model, however, demonstrated computational efficiency primarily with Gaussian data likelihood and specific kernel functions, such as the RBF and linear kernels \cite{williams2006gaussian}, limiting the model capacity of GPLVM. 

Recently, several efforts have been directed towards improving the scalability and flexibility of the GPLVM \cite{lalchand2022generalised,de2021learning} and ensuring its compatibility with a variety of likelihood functions \cite{ramchandran2021latent}. One of the most representative works is \cite{ramchandran2021latent}, namely GPLVM-SVI, which seamlessly incorporates inducing points-based sparse GP and stochastic variational inference \cite{hensman2013gaussian}, enabling scalable inference and the potential to handle a variety of likelihood functions.  Unfortunately, despite the progress made, these models rely on inducing points-based sparse GP as proposed by \cite{titsias2009variational}, necessitating the optimization of additional inducing points. This increases the computational load and raises the risk of convergence to suboptimal solutions.
As a result, even with advancements in model capability, these GPLVM variants often struggle to fully realize their theoretical potential \cite{li2024preventing}. 

%% file: appendix/experiments.tex
\section{Experiment Details} 
\label{app:experiment}

\subsection{Data descriptions and preprocessing}
\label{app:ds_preprocessing}

To begin, we delineate the specific parameter configurations employed for generating synthetic $S$-shaped datasets, as utilized in Section  \ref{subsec:toy_example}. These datasets generated from a GPLVM characterized by the following kernel configuration:
\begin{align}
        & k_{\mathrm{hybrid}}(\x, \x^\prime) = k_{\mathrm{rbf}}(\x, \x^\prime) + k_{\mathrm{periodic}}(\x, \x^\prime), \\
        &k_{\mathrm{rbf}}(\x, \x^\prime) = \ell_{o} \exp(-\frac{(\x- \x^\prime)^2}{2 \ell_l^2}),  \text{ with }  \ell_o = 0.5, \ell_l = 1; \nonumber \\  
        & k_{\mathrm{periodic}}(\x, \x^\prime) = \ell_{o} \exp \left(-\frac{2 \sin ^2\left(\frac{\x-\x^{\prime}}{p}\right)}{ \ell_l^2}\right), \nonumber \\ 
        &\text{ with } \ell_o = 0.5, \ell_l = 1, p=4.5. \nonumber
\end{align}
Then, we provide a thorough introduction to real-world datasets, subsequently downsizing large-scale datasets to a reduced size to manage the considerable computational complexity inherent in RFLVMs \cite{gundersen2021latent}.
\begin{itemize}
    \item{
        \textbf{Bridges:} 
        We utilized account how many bicycles crossed each of the four East River bridges in New York City each day\footnote{\url{https://data.cityofnewyork.us/Transportation/Bicycle-Counts-for-East-River-Bridges/gua4-p9wg}}. We used weekday vs. weekend as binary labels since these data are unlabeled and as such information is connected with bicycle numbers.
    }
    
    
    \item{
        \textbf{CIFAR-10:} 
        For a final dataset of size $2000$, we subsampled $400$ photos from each class and restricted the classes to $[1-5]$. The images were reduced in size from $32 \times 32$ to $20 \times 20$ pixels and turned to grayscale. \textit{Test performance of different models on the full dataset can be found in Section \ref{subsec:large_scale}. }
    }
    
    \item{
        \textbf{MNIST:} By randomly selecting 1000 images, we reduced the size of the dataset. \textit{Test performance of different models on the full dataset can be found in Section \ref{subsec:large_scale}.}
    }
    
    \item{
        \textbf{Montreal:} We use the number of cyclists per day on eight bicycle lanes in Montreal.\footnote{\url{http://donnees.ville.montreal.qc.ca/dataset/f170fecc-18db-44bc-b4fe-5b0b6d2c7297/resource/64c26fd3-0bdf-45f8-92c6-715a9c852a7b}}. Since these data are unlabeled, we utilized the four seasons as labels, due to seasonality is correlated with bicycle counts.
    }
    
    \item{
        \textbf{Newsgroups:} The 20 Newsgroups Dataset\footnote{\url{http://qwone.com/~jason/20Newsgroups/}}. The classes were restricted to \textit{comp.sys.mac.hardware}, \textit{sci.med}, and \textit{alt.atheism}, and the vocabulary is limited to words with document frequencies in the range $10-90\%$.
    }
    

    \item{
        \textbf{Yale:} The Yale Faces Dataset\footnote{\url{http://vision.ucsd.edu/content/yale-face-database}}. We used subject IDs as labels.
    }

    \item{
        \textsc{Brendan:}
        The dataset consists of 2000 images, each with dimensions of \(20 \times 28\) pixels, portraying Brendan's face\footnote{\url{https://cs.nyu.edu/~roweis/data/frey_rawface.mat}}. 
    }
    
    
\end{itemize}


\subsection{Benchmark Methods Descriptions}
\label{app:benchmark}

\begin{itemize}
    \item{
        \textbf{PCA, LDA, NMF, Isomap:} We employed implementations from the \texttt{sklearn.decomposition} module within the \texttt{scikit-learn} library for PCA \cite{roweis1997algorithms}, LDA \cite{blei2003latent}, NMF \cite{lee1999learning}, and Isomap \cite{balasubramanian2002isomap}.
    }
    
    \item{
        \textbf{BGPLVM:} We utilized the \texttt{BayesianGPLVMMiniBatch} implementation from \texttt{GPy} library\footnote{\url{http://github.com/SheffieldML/GPy}}, which is an inducing points-based method \cite{titsias2009variational}.
    }

    \item{
        \textbf{VAE:} The VAE \cite{kingma2013auto} implementation followed the example code provided by the \texttt{pytorch} library\footnote{\url{https://github.com/pytorch/examples/blob/main/vae/main.py}}.
    }

     \item{
        \textbf{NBVAE, DCA, CVQ-VAE, RFLVM:} Implementations for these algorithms align with their respective official code repositories available online
        \footnote{\url{https://github.com/ethanhezhao/NBVAE}}
        \footnote{\url{https://github.com/theislab/dca}}
        \footnote{\url{https://github.com/lyndonzheng/CVQ-VAE}}
        \footnote{\url{https://github.com/gwgundersen/rflvm}}. 
    }    
\end{itemize}

\subsection{Missing Data Imputation}
\label{app:missing_data}

The reconstruction data $\Y_{rec}$ are  generated by their posterior mean $\hat{\mathbf{Y}}_{rec} \!=\! \mathbb{E}[ \mathbf{Y}_{rec} \mid  \mathbf{X}, \mathbf{Y}_{obs}, -]$. Visualizations of the reconstructed observed data are presented in Fig.~\ref{appx_fig:mnist_missing_illustration} and Fig.~\ref{appx_fig:bface_missing_illustration}, highlighting the superior ability of SRFLVMs to restore missing pixels.
